\documentclass[sigconf,nonacm]{acmart}
\setcopyright{none}

\usepackage{microtype}
\usepackage{subcaption}
\usepackage{enumitem}
\usepackage{tikz}
\usetikzlibrary{arrows.meta, positioning}
\usepackage{graphicx}
\usepackage{float}
\usepackage{flafter}
\usepackage[capitalize]{cleveref} 

\usepackage{tikz}
\usepackage{pgfplots}
\pgfplotsset{compat=1.18}
\usepgfplotslibrary{fillbetween}
\definecolor{paretoSRG}{RGB}{24,128,121}
\definecolor{baseSRG}{RGB}{75,75,75}
\definecolor{xtColor}{RGB}{205,103,49}
\usepgfplotslibrary{groupplots}
\newtheorem{assumption}{Assumption}
\newtheorem{criterion}{Criterion}

\AtBeginDocument{%
  }

\newcommand\blfootnote[1]{%
  \begingroup
  \renewcommand\thefootnote{}\footnotetext{#1}%
  \addtocounter{footnote}{-1}%
  \endgroup
}

\begin{document}

\title[Training-Guided Occlusion]{XtrAIn: Training-Guided Occlusion for Feature Attribution}

\author{%
Thodoris Lymperopoulos \qquad
Ioannis Kakogeorgiou \qquad
Denia Kanellopoulou\\
NCSR Demokritos, Athens, Greece
}

\renewcommand{\shortauthors}{Lymperopoulos et al.}

\begin{abstract}
  Occlusion-based attribution methods provide an intuitive way to estimate feature importance by perturbing input features and measuring the resulting change in model output. However, their reliability is strongly affected by how feature removal is implemented: externally selected baselines can introduce bias, out-of-distribution samples, and unstable explanations, while in nonlinear models the occlusion of a set of features can also alter the contribution of non-occluded features. We refer to this effect as attribution shift, as the attribution scores of the non-occluded features drift from their initial values. To challenge these major issues that render explanations unstable, we introduce XtrAIn, a training-guided attribution method that transfers the occlusion operation from the input space to the parameter space. Instead of replacing input values with hand-crafted baselines, XtrAIn follows the model’s training trajectory and measures how feature-associated parameter updates affect the output logits. We further introduce Xstep, a lightweight approximation for reducing computational cost, and XtrAIn+, a target-focused variant that emphasizes updates aligned with the target class. Experiments on controlled image datasets and PAM50 breast-cancer subtype classification show that the proposed methods produce cleaner and more interpretable attribution patterns than standard attribution baselines. Overall, XtrAIn provides a training-aware perspective on feature attribution and offers a useful diagnostic tool for studying how feature-level evidence is formed during training.

\end{abstract}


\maketitle

\blfootnote{Corresponding author: \texttt{thodoris.lymperopoulos@gmail.com}}

\section{Introduction}
\label{sec:intro}

As AI systems are increasingly used in high-impact decision-making settings, understanding the factors that drive their predictions has become a central concern in modern machine learning. Explainable AI (XAI) \cite{Mumuni2025ExplainableAI, 10.1613/jair.1.12228, vilone2020explainableartificialintelligencesystematic} has emerged as a primary vehicle for this effort, with attribution methods \cite{11039632, 9369420} forming a widely adopted paradigm for quantifying the contribution of individual input features to a model's output. Among them, occlusion-based approaches \cite{10.1007/978-3-319-10590-1_53, 10.5555/3295222.3295230, petsiuk2018riserandomizedinputsampling, 8237633} occupy a foundational position — operating by systematically removing subsets of input features and measuring the resulting change in model output. The magnitude of this change is taken as a proxy for the contribution of the occluded features to the model's decision. This perturbation logic is closely related to causal reasoning~\cite{NEURIPS2021_4f5c422f, nnsattr_causal}, making occlusion a persistent component of explainability techniques for increasingly complex architectures~\cite{covert2023learningestimateshapleyvalues}.

Nevertheless, occlusion-based methods remain sensitive to how feature removal is implemented, especially through the choice of baseline values. In their attempt to simulate such removal, different techniques for value selection may introduce biases \cite{pmlr-v162-rong22a}, as demonstrated by several studies \cite{sturmfels2020visualizing, Haug2021OnBF, Geoscience}. Additionally, the shape and size of the occluded regions are not standardized, and different parametrizations can lead to different attribution maps \cite{bluecher2024decoupling}. Another important issue is out-of-distribution (OoD) behavior, where an incautious selection of baseline values can move samples to regions outside the original distribution~\cite{10.5555/3454287.3455160, 10.1007/s11222-021-10057-z, pmlr-v119-kumar20e, 10.1007/978-3-031-09037-0_8, jain2022missingnessbiasmodeldebugging, Geoscience}. This introduces safety concerns for sensitive applications \cite{salehi2022a}. Researchers have addressed this either by introducing corrupted data during training \cite{10.5555/3454287.3455160, 10.5555/3540261.3540540, BROCKI2023131} or by using generative models for more natural inpainting \cite{chang2019explainingimageclassifierscounterfactual, 10.1007/978-3-030-69544-6_7, augustin2024digindiffusionguidanceinvestigating}. While these approaches may alleviate OoD issues, they do not fully remove the dependence on the chosen feature-removal rule. In this regard, some techniques apply criteria for a theoretically grounded information removal \cite{10.1016/j.imavis.2022.104516, Ren2023CanWF, Izzo_2021}, yet defining a generally reliable and objective removal strategy remains challenging. Similar sensitivities also affect feature-removal-based evaluation metrics, whose outcomes may depend on the chosen perturbation strategy \cite{Gevaert2024, pmlr-v162-rong22a, Nauta_2023}.

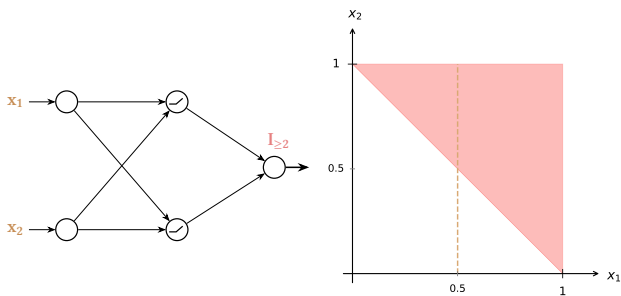
\begin{figure}[t]
    \centering

    \begin{minipage}[c]{0.48\columnwidth}
        \centering
        \resizebox{\linewidth}{!}{%
        \begin{tikzpicture}[
            >=Stealth,
            node distance=2cm,
            neuron/.style={circle, draw=black, thick, minimum size=0.5cm, inner sep=0pt},
            input/.style={circle, fill=black, minimum size=0.2cm, inner sep=0pt},
            every label/.style={font=\large}
        ]

        \node[circle, draw=white, thick] (start) {};
        \node[neuron, below=0.4cm of start] (x1dot) {};
        \node[neuron, below=2.4cm of x1dot] (x2dot) {};

        \node[left=0.6cm of x1dot, text=brown!100!white!90, font=\LARGE] (x1) {$\mathbf{x_1}$};
        \node[left=0.6cm of x2dot, text=brown!100!white!90, font=\LARGE] (x2) {$\mathbf{x_2}$};

        \node[neuron, right=2cm of x1dot] (h1) {};
        \node[neuron, right=2cm of x2dot] (h2) {};

        \draw[thick] (h1.center) +(-0.18,-0.05) -- +(-0.02,-0.05) -- +(0.13, 0.08);
        \draw[thick] (h2.center) +(-0.18,-0.05) -- +(-0.02,-0.05) -- +(0.13, 0.08);

        \node[neuron, right=1.7cm of h1, yshift=-1.5cm] (out) {};
        \node[above=0.06cm of out, xshift=3pt, font=\LARGE, text=red!80!black!50] {$\mathbf{I_{\geq2}}$};

        \draw[->] (x1) -- (x1dot);
        \draw[->] (x2) -- (x2dot);

        \draw[->] (x1dot) -- (h1);
        \draw[->] (x1dot) -- (h2);
        \draw[->] (x2dot) -- (h1);
        \draw[->] (x2dot) -- (h2);

        \draw[->] (h1) -- (out);
        \draw[->] (h2) -- (out);

        \draw[->, black, line width=1pt] (out) -- ++(0.8, 0);

        \end{tikzpicture}%
        }
    \end{minipage}
    \hfill
    \begin{minipage}[c]{0.48\columnwidth}
        \centering
        \includegraphics[width=\linewidth]{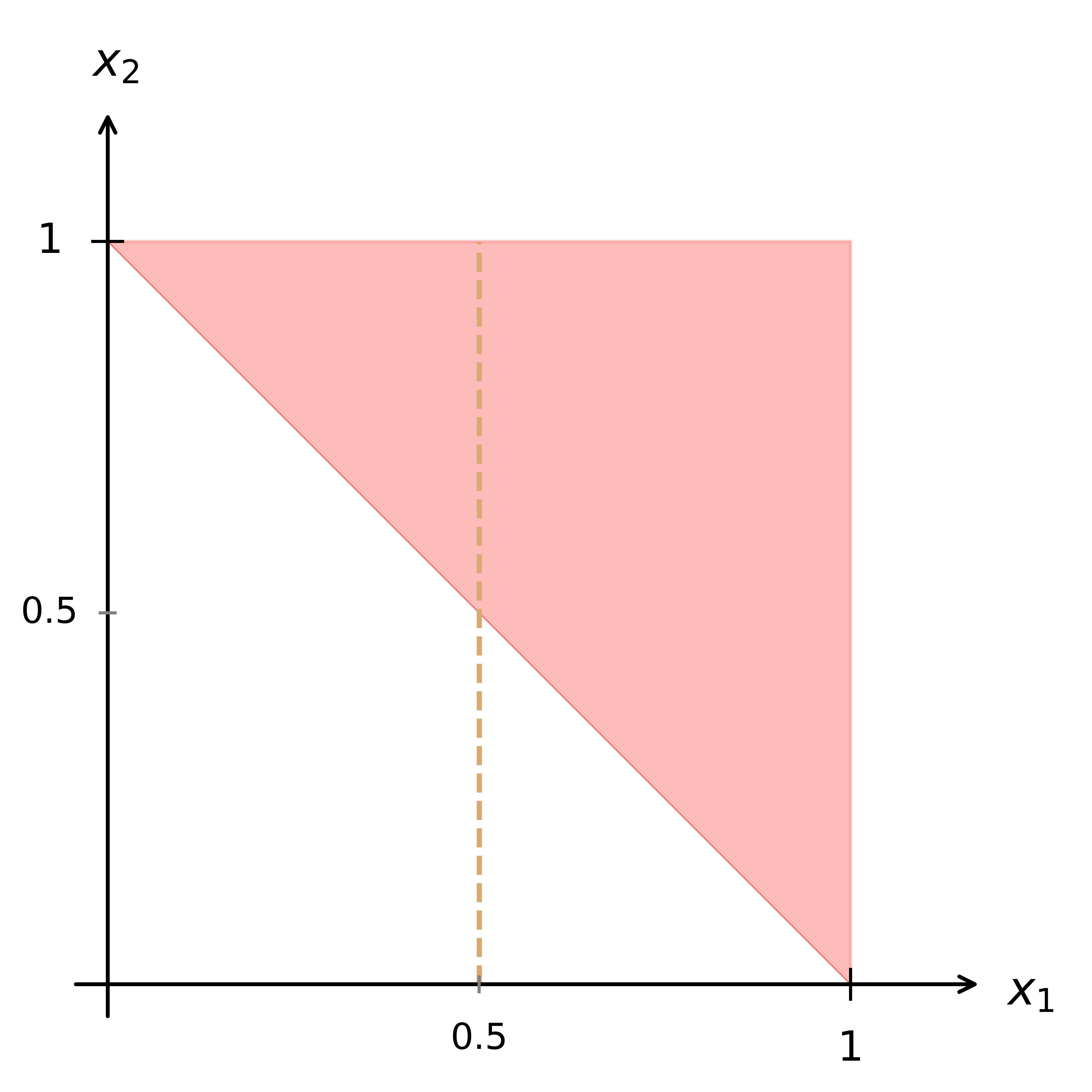}
    \end{minipage}
    \hfill\null
    \caption{\emph{Illustration of attribution shift in input-space occlusion.}
    In the simple FCNN on the left, all weights are set to one and all biases are zero.    The two hidden ReLU neurons therefore receive the same input, $x_1+x_2$. For $x_1=x_2=0.5$, both hidden neurons output $1$, so the input to the final threshold neuron is $1+1=2$ and the output is activated. However, a small decrease $\epsilon >0$ in $x_2$ changes both hidden activations to values below $1$, making the final input smaller than $2$ and deactivating the output. This abrupt change in model's response should spark large shifts in feature attributions — since these scores are related to the model's output. This illustrates how input-space occlusion can produce interaction-dependent attribution shifts.}
    \label{fig:attr_shift}
\end{figure}

Beyond baseline selection, we highlight an additional limitation of input-space occlusion: because feature removal is implemented through an \emph{externally imposed} input intervention, the attribution assigned to an occluded feature subset may also reflect changes in the effect of the remaining, non-occluded features.
Let $I_1 \subset I$ denote an occluded subset of input features, and let $I_2 \subseteq I \setminus I_1$ denote the subset of the remaining non-occluded features. Occlusion estimates the contribution of $I_1$ by comparing the original model response $f(I)$ with the response obtained after removing or replacing $I_1$, denoted $f(I \setminus I_1)$. This implicitly assumes that the effect of the non-occluded features remains constant across the two evaluations and is therefore neglected. However, in nonlinear models, the intervention on $I_1$ can also change how the remaining features $I_2$ interact with the model, as illustrated in \cref{fig:attr_shift}. Therefore, the differences in model's response does not necessarily isolate the effect of $I_1$ alone — as popular occlusion-based methods assume — but may include 
may include interaction-dependent effects from the remaining features. We refer to this phenomenon as \textit{attribution shift}: the attribution assigned to $I_1$ may include effects from feature dependencies, induced by the disturbed behavior of $I_2$. In this regard, attribution scores resulting from occlusion-based techniques are indeed interaction dependent (despite the hypothesis for interaction-blind feature attributions \cite{NEURIPS2020_443dec30}), but in a perilous way that perturbs feature estimation scores and complicates the disentanglement of the intended feature-removal effect.

Attribution shift stems from the input-space occlusion mechanism itself and is not fully resolved by changing only the baseline-filling strategy. This motivates us moving the perturbation operation beyond the input space, towards transferring the occlusion operation to the parameter space: instead of relying on externally imposed input values, \textit{XtrAIn} uses perturbation values naturally induced by the model's own training updates. In this work, we instantiate this principle through the feature-associated weights of fully connected neural networks (FCNNs). By perturbing these weights between consecutive model states, we measure the induced change in the model logits and aggregate these parameter-space occlusions over training to obtain a final attribution score that reflects how each feature-associated update contributes to the model's output.

Our work makes the following contributions:
\begin{itemize}
    \item We analyse how \textit{attribution shift} arises in input-space occlusion for nonlinear models, showing that occluding one feature can also alter the contribution of other, non-occluded features.

    \item We introduce $XtrAIn$, a new occlusion-based attribution method that transfers the occlusion operation from the input space to feature-associated parameter updates. We further introduce $Xstep$, a lightweight approximation for compute-intensive settings, and $XtrAIn^+$, a target-focused variant designed to emphasize updates that strengthen the target class.

    \item We propose \textit{CleanScore}, a diagnostic metric for evaluating attribution cleanness in settings where signal and background regions are known.

    \item We evaluate the proposed method on controlled image datasets and a PAM50 breast-cancer subtype classification task, showing that parameter-space occlusion produces cleaner, safer and more interpretable explanations than standard attribution baselines.
\end{itemize} 
\begin{figure*}[t]
    \centering
    \includegraphics[width=0.8\textwidth]{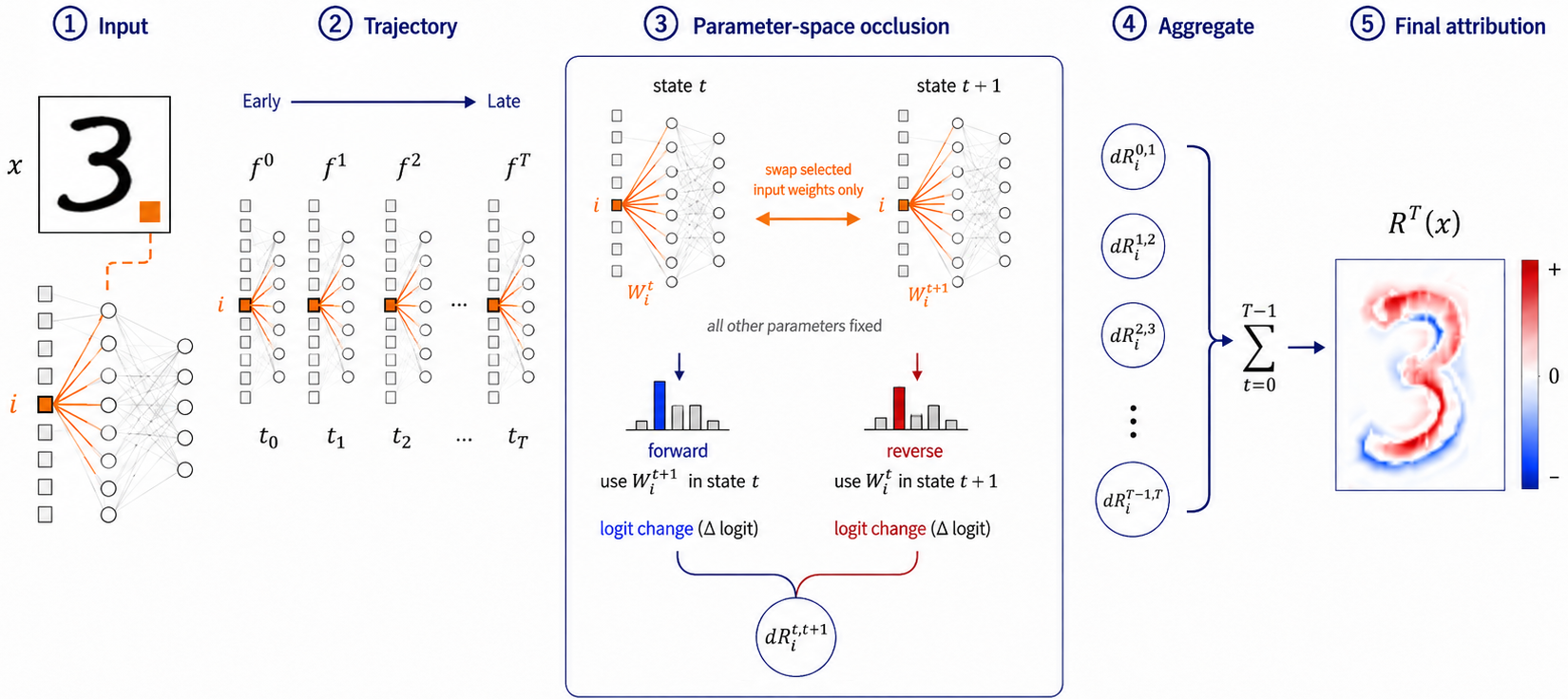}
    \caption{
\emph{Overview of \textit{XtrAIn}}. Given an input sample $x$, the method follows the training trajectory of the model and tracks the weights $W_i^t$ associated with each input feature $i$. For two consecutive model states, $t$ and $t+1$, \textit{XtrAIn} performs parameter-space occlusion by replacing only the feature-associated weights while keeping all other parameters fixed. The resulting forward and reverse logit changes are combined into a step-wise attribution update $dR_i^{t,t+1}$, and these updates are accumulated over training to obtain the final attribution map $R^T(x)$.
    }
    \label{fig:xtrain_overview}
\end{figure*}

\section{Related Work}
\label{sec:related}


\textbf{Training Dynamics}. 
Many ML disciplines study training dynamics, with the most notable being the framework of \textit{Developmental Interpretability}, which examines the 
learning phenomena of deep models as they form over time. This has led to in-depth investigations of a model's complex behaviors, such as grokking \cite{power2022grokkinggeneralizationoverfittingsmall} 
and emergent capabilities \cite{schaeffer2023emergentabilitieslargelanguage, Steinhardt2023EmergentAbilities}, while offering insights into its 
interpretability, as with the detection of personality traits \cite{chen2025personavectorsmonitoringcontrolling}. 
However, this line of research focuses primarily on interpreting model behavior, rather than estimating feature importance. Conversely, another approach 
that addresses this theoretical question is that of \textit{Feature Selection} \cite{10.5555/944919.944968} which focuses on training statistics, 
feature importance and data influence. A related subfield, namely \textit{Embedded Methods} \cite{10.1145/3205651.3208227} operates by introducing an 
artificial layer between the input and the first hidden layer to collect different statistics throughout the training process—most commonly weights and 
gradients \cite{Importance2024Veas}—which are then used to infer feature importance using different techniques applied to the collected gradient or weight profiles. 
Nonetheless, the relationship between these statistics and the model's overall decision-making process remains ambiguous.

\textbf{Occlusion-based Techniques}. These highly popular attribution methods  estimate feature importance by systematically altering parts of the input representation and measuring the corresponding drop in the model’s output score. The work of \cite{JMLR:v22:20-1316} comprises the most detailed study of occlusion-based methods, grouping them all under the same framework and providing a foundation for theoretical comparison. Despite their strengths and successes, critical limitations such as added bias \cite{pmlr-v162-rong22a} and out-of-distribution data \cite{10.5555/3454287.3455160, 10.1007/s11222-021-10057-z} persist.


\textbf{Beyond Classical Occlusion.}
A series of techniques have attempted to alleviate the theoretical limitations of Occlusion methods for feature removal. Statistical measures being applied to 
the whole data distribution, such as PDP \cite{10.1214/aos/1013203451}, FMI \cite{JML-v20-18-760} and ALE \cite{apley2019visualizingeffectspredictorvariables} 
offer insights into feature importance, while marginalizing the effect of complementary features. However, they are designed to provide global explanations. Authors 
in \cite{10.1007/978-3-032-08333-3_8} apply perturbations in the transformed Fourier space, yet with unknown effects for the problems at hand. A different line of 
work \cite{inproceedingsOcclusionSensitiv} computes a similarity score among the inner representations of augmented images and their occluded counterparts, yet the limitations of feature removal persist. Recently, authors in \cite{lymperopoulos2026weightperturbationfeatureattribution} developed a novel occlusion-based approach using weight perturbations, estimating feature importance as a distance score between the untrained and trained model versions. Inspired by this direction, we build upon it while significantly strengthening its foundations.

Our proposed method advances beyond classical occlusion rules using weight perturbations. We exploit mechanisms within training for feature independence and transfer the field of occlusion rules to the parameter space, where the problems of Added Bias and OoD are not expressed, while attribution shift is carefully eliminated. 

\section{Methodology}
\label{sec:method}

In general, attribution methods assign importance scores using rules derived from specific model properties. However, the choice of which properties should define \textit{importance} is not uniquely specified, since importance itself lacks a precise and universally accepted definition \cite{lipton2017mythosmodelinterpretability, haufe_explainable_2026}. As a result, every attribution method implicitly relies on assumptions that connect the model's mathematical behavior to a human-interpretable notion of feature contribution. In this work, we make these assumptions explicit and use them to motivate the design of our training-guided attribution rule.

\begin{assumption}
    \label{assum:1}
    Attribution scores are relative values whose interpretation depends primarily on their sign, ranking, and relative magnitude across features.
\end{assumption}

\textbf{Mathematical Notation.}
We assume an arbitrary FCNN $f = f^L \circ f^{L-1}\circ \dots \circ f^0$ receiving input $x \in \mathcal{X}$, with $\mathcal{X}$ being 
the input space. Given an arbitrary intermediate layer $l$ with $n_l$ neurons, weights $W^l \in \mathbb{R}^{n_{l-1}\times n_{l}}$ and bias $b^l\in \mathbb{R}^{n_l}$,
the input $x^l_j$ and output $z^l_j$ of a neuron $j$ in this layer are calculated as:
\begin{align}
    \label{eq:fcnn}
    x^{l}_j &= \sum_{i \in[l-1]} z_i^{l-1} \cdot w_{ij}^l + b_j^l, \\
    z_j^l &= \sigma(x_j^l),
\end{align}
where $[l]$ represents the neurons of layer $l$ and $\sigma$ is a nonlinear activation function (usually ReLU).

\subsection{Foundations}
\label{sub:foundations}

Conceptually, model parameters govern the interactions among neurons within the network. In this regard, attribution methods are expected to depend on these values; otherwise, the resulting scores would be model-independent, thereby weakening their role as model explanations. This implies that the changes in parameter values result in the alteration of the attribution scores assigned to a given input \cite{adebayo2020sanitycheckssaliencymaps}. Therefore, during training, where parameters change across optimization steps, the evolution of the model can be associated with a sequence of attribution scores $\mathbf{R}^t \in \mathbb{R}^{n_0}$, with $t \in \{1,\dots,T\}$ representing the training step and $T$ the total number of steps.

\begin{assumption} 
    \label{assum:2}
    In the general case, an update of the model parameters alters the attribution scores for a given input.
\end{assumption}

Building on \cref{assum:2}, we develop a methodology for feature-importance estimation that departs from the standard \textit{static} view, where attribution scores are extracted only from the final trained model. Instead, we propose a \textit{training-guided} attribution score, obtained by aggregating attribution changes across consecutive model states. This shifts the core question from \textquoteleft \textit{how to explain a trained model}\textquoteright{} to \textquoteleft \textit{how to explain a model update}\textquoteright. We posit that this presents a more manageable problem to solve, as it allows for the exploitation of meaningful, tractable data drawn from the training process. 

To tackle this challenge, we introduce a step-wise attribution rule that estimates the added attribution value $d\mathcal{R}^t$ induced by each model update. It will form the updating mechanism of a recursive function
for the estimation of feature attribution. This perspective requires that the model update process remains consistent across training, so that attribution changes can be compared and accumulated over consecutive optimization steps.

\begin{assumption} 
    \label{assum:3}
    The model's operational principles and updating mechanism remain consistent throughout training. Accordingly, the attribution-updating mechanism must remain invariant across training steps.
\end{assumption} 

\subsection{Update Rule \& XtrAIn}
\label{sub:rec_func}

To develop an update rule for the estimation of $d\mathcal{R}^t$, we reflect upon the model's overarching goal, that is
the \textit{meaningful update of its parameters} towards loss minimization. This is ultimately expressed at the model logits.

\begin{assumption} 
    \label{assum:4}
    The change in model logits represents the effect of parameters' update.
\end{assumption}

\Cref{assum:4} fosters the estimation of the effect for each feature's parameters update directly at the output logits (in special case, a zero score is assigned to parameter updates that cause no effect in model output). Therefore, for each input feature $i$ and target class $\odot$, we define the update-level attribution score as:
\begin{equation}
    \label{eq:upd_rule}
    d\mathcal{R}^{t,t+1}_{i}(x;\odot)
    =
    \mathcal{I}^{\odot \top}
    \cdot\left(
    df^{t,t+1}_{i}(x) + d\tilde{f}^{t,t+1}_{i}(x)
    \right),
\end{equation}
where $df^{t,t+1}_{i}(x), d\tilde{f}^{t,t+1}_{i}(x) \in \mathbb{R}^{n_L}$ are logit-change vectors defined as:
\begin{align}
    df^{t,t+1}_{i}(x)
    &=
    f^{t}_{W_i^t \rightarrow W_i^{t+1}}(x) - f^{t}(x), \label{eq:df} \\
    d\tilde{f}^{t,t+1}_{i}(x)
    &=
    f^{t+1}(x) - f^{t+1}_{W_i^{t+1} \rightarrow W_i^{t}}(x). \label{eq:df'}
\end{align}
Here, $f^t$ and $f^{t+1}$ denote the model before and after a training update at time $t$, respectively, while $W_i^t$ denotes the set of weights directly connected to input feature $i$ at this time, i.e.,
\[
W_i^t = \{w^t_{ij} \mid j \in \{1,\ldots,n_1\}\}.
\]
The notation $f^{t}_{W_i^t \rightarrow W_i^{t+1}}$ denotes the model at state $t$ where only the feature-associated weights $W_i^t$ are replaced by their updated values $W_i^{t+1}$, while all other parameters remain fixed. Similarly, $f^{t+1}_{W_i^{t+1} \rightarrow W_i^{t}}$ denotes the reverse replacement in the updated model.

The sign vector $\mathcal{I}^{\odot} \in \{-1,1\}^{n_L}$ maps logit changes to rewards:
\begin{equation}
    \mathcal{I}^{\odot}_c =
    \begin{cases}
    \text{ }1, & \text{if } c = \odot, \\
    -1, & \text{if } c \neq \odot.
    \end{cases}
\end{equation}
Thus, $d\mathcal{R}^{t,t+1}_{i}(x;\odot)$ is a scalar attribution update for feature $i$, computed with respect to a fixed target class $\odot$ by projecting the logit-change vector onto the corresponding target-class sign vector. Collecting these scalar scores over all input features yields the step-wise attribution vector $d\mathcal{R}^{t,t+1}(x;\odot) \in \mathbb{R}^{n_0}$. For simplicity, we omit the explicit target argument in the following and use $\mathcal{R}^t(x)$ and $d\mathcal{R}^{t,t+1}(x)$ to denote their target-conditioned counterparts.

The two logit-change vectors in \cref{eq:upd_rule}, defined in \cref{eq:df,eq:df'}, correspond to forward and reverse parameter-space occlusions between consecutive model states. The first term measures the effect of replacing the feature-associated weights at state $t$ with their updated values from state $t+1$, while the second term measures the complementary effect of reverting the same weights in the updated model. Considering both directions makes the update rule symmetric with respect to consecutive model states and supports the inverse property discussed in \cref{sub:crit}.

Regarding the sign vector $\mathcal{I}^{\odot}$, its role is to map each logit update to a reward signal. Its design stems from the 
model's training principle, namely, the promotion of the target neuron and suppression of non-target neurons.

\begin{assumption}
    \label{assum:5}
    A logit update is positively rewarded if it represents:
    \begin{itemize}[topsep=0pt, parsep=0pt, partopsep=0pt]
        \item an increase, in case of the target neuron, 
        \item a decrease, in case of the non-target neurons.
    \end{itemize}
    The opposite cases are considered negative.
\end{assumption}

We examine whether the proposed update rule produces structured attribution patterns during training. As shown in \cref{fig:upd_evol_patterns}, the intermediate update scores reveal feature-associated structures that evolve across optimization steps. Early updates tend to highlight stable input-aligned patterns, while later updates include more class-competitive structures, reflecting how the model progressively adjusts its decision boundaries.

\begin{figure*}[h]
    \centering
    \begin{minipage}{0.48\textwidth}
        \centering
        \includegraphics[width=\linewidth]{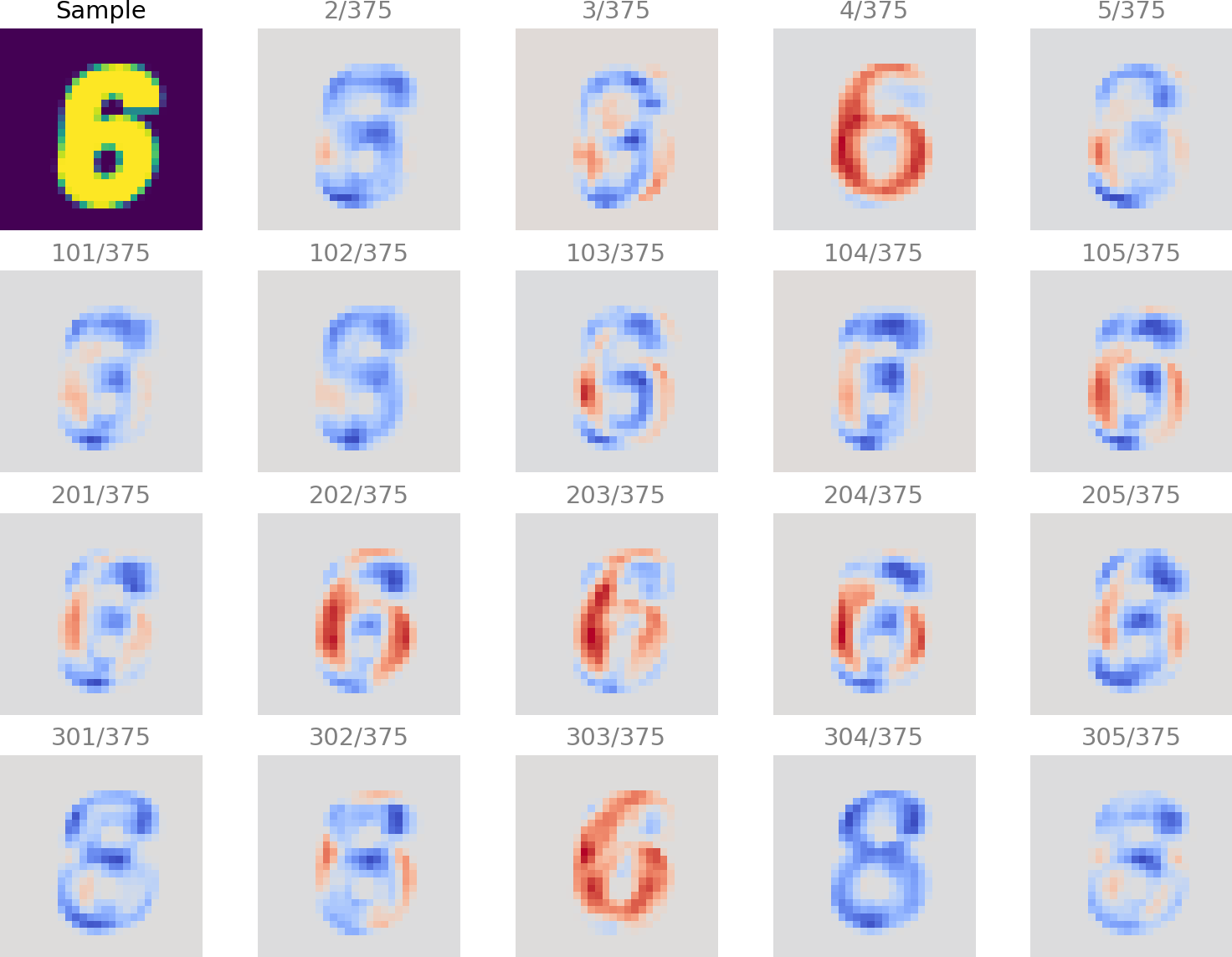}
        \label{fig:sub_1}
    \end{minipage}
    \hfill
    \begin{minipage}{0.48\textwidth}
        \centering
        \includegraphics[width=\linewidth]{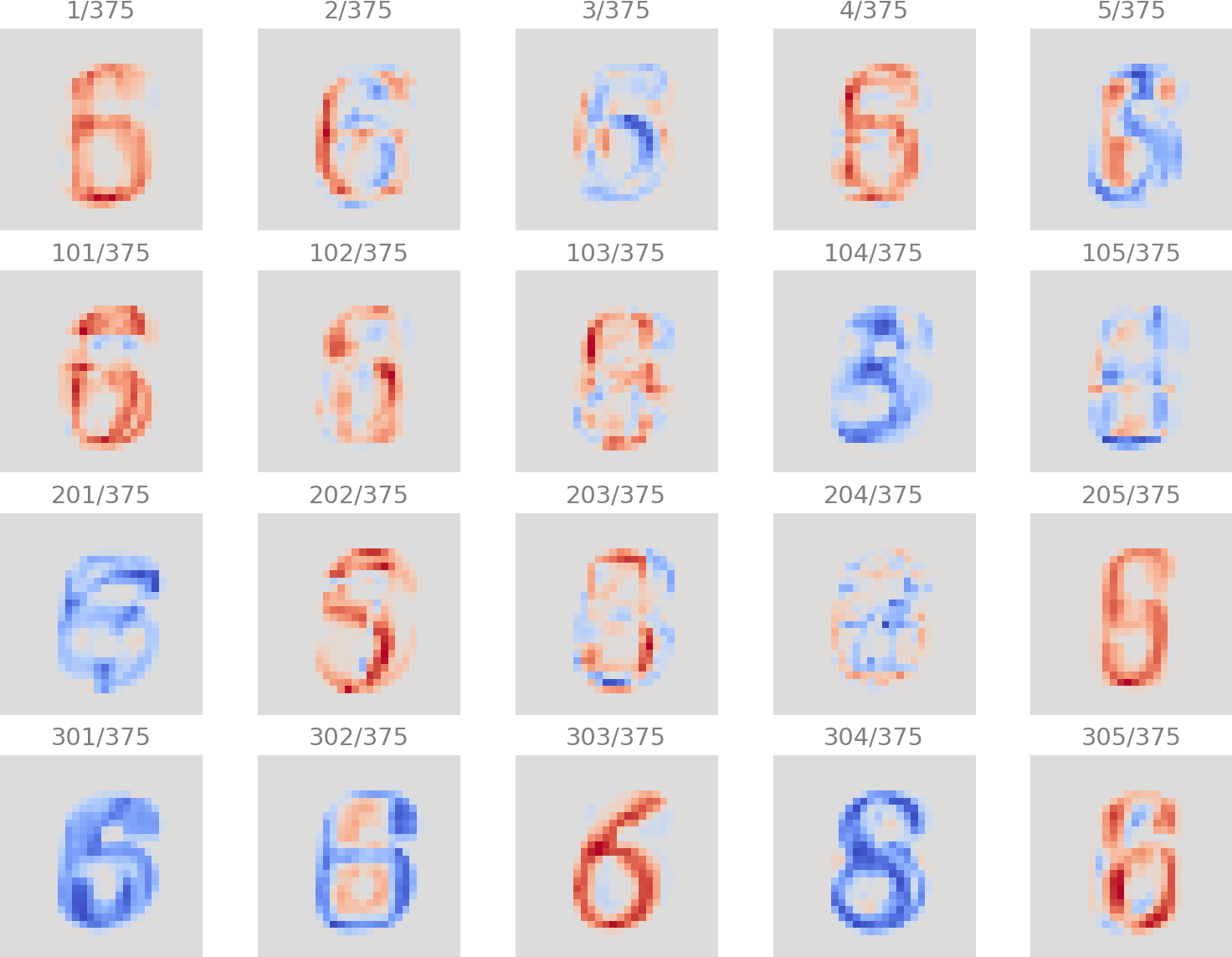}
        \label{fig:sub_2}
    \end{minipage}
    \begin{minipage}{0.48\textwidth}
        \centering
        \includegraphics[width=\linewidth]{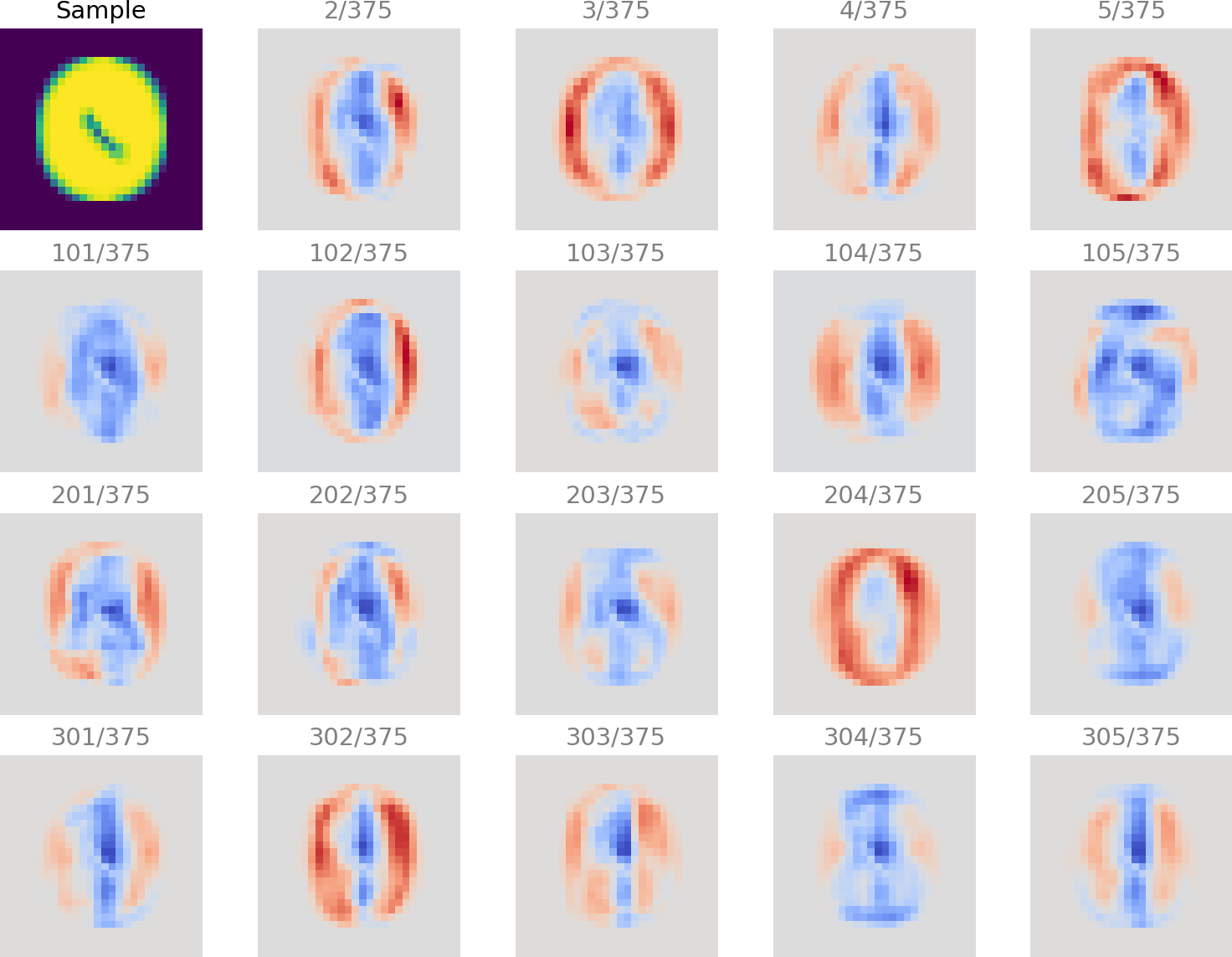}
        \subcaption{Patterns at the first training epoch}
        \label{fig:sub_3}
    \end{minipage}
    \hfill
    \begin{minipage}{0.48\textwidth}
        \centering
        \includegraphics[width=\linewidth]{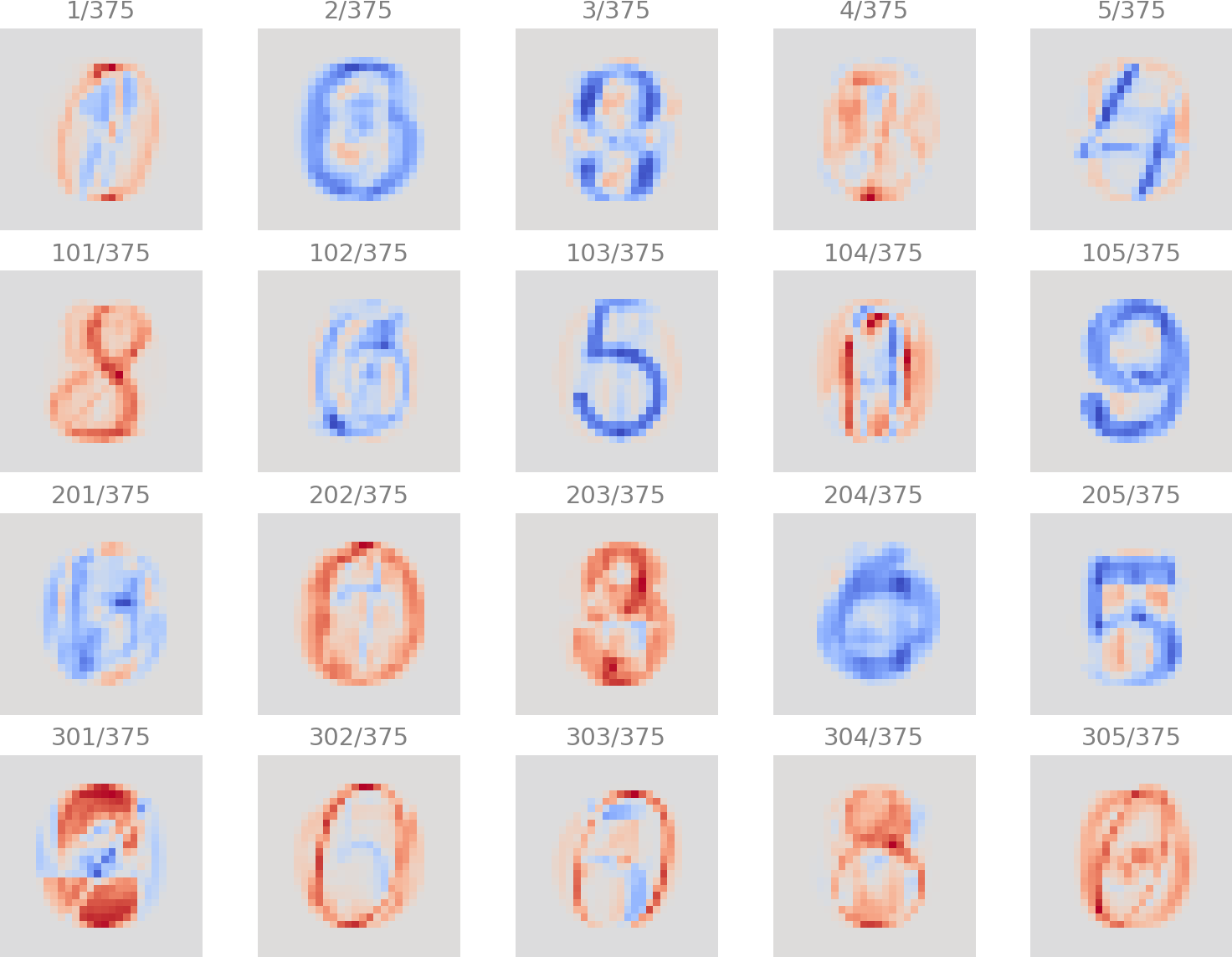}
        \subcaption{Patterns at the fifth training epoch}
        \label{fig:sub_4}
    \end{minipage}
    \caption{\emph{Intermediate update-level attribution patterns for two samples across training epochs.} The resulting scores of the update rule are displayed at different steps of a given training epoch (as indicated by their titles), uncovering a diverse set of attribution patterns within the FCNN's training process. In the early training state of the model, heatmaps indicate higher stability in pattern formation, while in later stages, more intricate patterns emerge.}
    \label{fig:upd_evol_patterns}
\end{figure*}

Considering all intermediate updates, we introduce a training-guided attribution method, named $XtrAIn$, which accumulates the attribution contributions induced by consecutive model updates. Given the step-wise attribution vector $d\mathcal{R}^{t, t+1}(x)$, $XtrAIn$ updates the attribution score recursively as:

\begin{align}
    \label{eq:XtrAIn_rule}
    \mathcal{R}^{t+1}(x) = \mathcal{R}^t(x) + d\mathcal{R}^{t, t+1}(x).
\end{align}

\noindent The recursion is initialized with $\mathcal{R}^0(x)=0$, treating the randomly initialized model as a neutral attribution baseline, since it has not yet learned input-output associations from the data.

Since the exact version of $XtrAIn$ requires evaluating attribution contributions across all consecutive training states, its computational cost can become prohibitive for large models or long training trajectories. To address this limitation, we introduce $Xstep$, a lightweight approximation that applies the same update rule only between selected model states, thereby skipping intermediate updates. As shown in \cref{sec:appl}, this approximation preserves the main attribution patterns while substantially reducing the computational burden, making it suitable for time- and resource-constrained settings. The total disregard of all intermediate model states results in an enhanced version of methods introduced in \cite{lymperopoulos2026weightperturbationfeatureattribution}.

Together, $XtrAIn$ and $Xstep$ provide a training-guided approach to feature-importance estimation, where attribution is derived from the model's own parameter updates rather than from externally chosen input baselines. This avoids hand-crafted baseline choices and reduces the risk of introducing artificial or out-of-distribution perturbations.

\subsection{Theoretical Considerations} 
\label{sub:theory}

\subsubsection{Rationale} 
\label{subsub:rationale}
To identify a mechanism for feature independence we reflect upon a parameter's update: In every step of the training process, it receives a gradient signal that points toward loss minimization. Its updated value will be defined by the term:
\begin{equation}
    \label{eq:train_upd}
    \Delta w_{ij} = -\eta * \frac{\partial \mathcal{L}}{\partial w_{ij}}.
\end{equation} 
Crucially, during the calculation of the gradient signal, all contributing parameters, which in FCNNs include the weights and biases of subsequent layers, are frozen, with their effect being captured within this signal through the chain rule. In this regard, this update is:
\begin{itemize}[topsep=0pt, parsep=0pt, partopsep=0pt]
    \item \textit{meaningful}: it is applied for loss minimization,
    \item \textit{independent}: all other model parameters are considered stationary.
\end{itemize}
The gradient signal carries the parameter's particular \textit{responsibility} towards loss optimization. Hence, we assume that its contribution can be theoretically traced according to \cref{assum:4} by an occlusion rule among the model and a replica, in which the value $w_{ij}$ is altered to $w'_{ij}$. 

Shifting our analysis from parameter to neuronal contributions, a higher level of abstraction is required, since the loss signal is only meaningful for the model's parameters, and neurons are the intermediate steps. In this process, the gradient of the loss with respect to the input features is calculated as:
\begin{equation}
    \frac{\partial \mathcal{L}}{\partial i} =  \sum_{j \in [n_1]} w_{ij} * \frac{\partial \mathcal{L}}{\partial j}
    \label{eq:neuron_loss_sig}
\end{equation}
where $i, j$ represent neurons of the input and first hidden layer respectively. However, the values $\frac{\partial \mathcal{L}}{\partial j}$ are shared among all input features $i$; a feature's individual traits are only introduced by weights $w_{ij}$. Hence, we can conceptually reduce a neuron to its attached parameters (for the input neurons, these being their weights $W_i$). These appear equivalent in \cref{eq:neuron_loss_sig}, enabling their grouping and reduction under one occlusion rule, estimating feature attribution by generalizing an individual parameter's occlusion rule to a group of carefully selected parameters.

\begin{assumption} 
    \label{assum:6}
    A neuron's effect should be estimated according to the occlusion rule, when grouping and considering the changes in its attached parameters. 
\end{assumption}

Having established the rationale behind the design of the terms $df$, $df'$, we now proceed to the explanation of their combined effect.

\subsubsection{Criteria and Theoretical Aspects} 
\label{sub:crit}

Researchers in XAI \cite{lrpOriginal, sundararajan2017axiomaticattributiondeepnetworks, 10.5555/3295222.3295230, srinivas2019fullgradientrepresentationneuralnetwork, fu2020axiombasedgradcamaccuratevisualization} 
have widely adopted foundational criteria to assess their methodologies, in response to the scarcity of theoretically sound evaluation metrics. These criteria are considered as necessary conditions any attribution method should satisfy. The methodology of $XtrAIn$ introduces a new field for criteria development, with an objective to guide the method's design and validate its effectiveness. 

\begin{criterion}
    \label{crit:one}
    (\textbf{Inverse Property}) Let $t$, $t+1$, $t+2$ be three consecutive time-steps and  $f^t, f^{t+1}, f^{t+2}$ be the corresponding models. If step $t+1$ is the \textit{reverse step} of $t$, meaning that $f^{t+2} = f^t$, then, it should hold that 
    $\mathcal{R}^{t+2}(x) = \mathcal{R}^t(x),\text{ } \forall x\in \mathbb{R}^{n_0}$.
\end{criterion}

Adherence to the Inverse Property yields robustness against attribution distortion in case of an immediate reverse in the model's state. 

\begin{theorem}
    $XtrAIn$ satisfies the Inverse Property.
    \label{theorem:one}
\end{theorem}

The theorem's proof can be found in \cref{app:inv_prop}. This criterion secures information erasure in the case of error correction. While many operations could combine terms $df$ and $df'$ while satisfying the Inverse Property, the addition operation was selected for its simplicity.

\subsection{Disentangling Target \& Non-Target Updates}
\label{sub:disent}

The update of model parameters is a synchronous process, simultaneously reinforcing target neurons and suppressing non-target ones. Different attribution methods have attempted to disentangle output neuron effects \cite{shrikumar2019learningimportantfeaturespropagating, gu2019understandingindividualdecisionscnns, iwana2019explainingconvolutionalneuralnetworks}, 
yet we argue that they are limited when attempting to retroactively separate class-specific signals within a model whose internal representations have already 
converged into a deeply entangled, non-linear synthesis of all classes. In this section, we demonstrate $XtrAIn$'s capacity to disentangle the effects of output 
neurons, focusing on target neurons to further improve interpretability and faithfulness to the model.

Indeed, the loss value and parameter update can be decomposed into \textit{target} and \textit{non-target} terms as:
\begin{align}
    \mathcal{L} &= \mathcal{L}_{target} + \mathcal{L}_{non\_target} \label{eq:loss_decomp}\\
    \Delta w_{ij} &= \Delta w_{target} + \Delta w_{non-target},
    \label{eq:weight_decomp}
\end{align}
which in the case of Cross Entropy Loss becomes: 
\begin{equation}
    \mathcal{L} = \underbrace{-o_{t}}_{target}+ \underbrace{log(\sum_{k=1}^{[L]} e^k)}_{non-target}
    \label{eq:param_dis_spec}
\end{equation}
A detailed mathematical proof can be found in \cref{app:loss_dis_sup}.

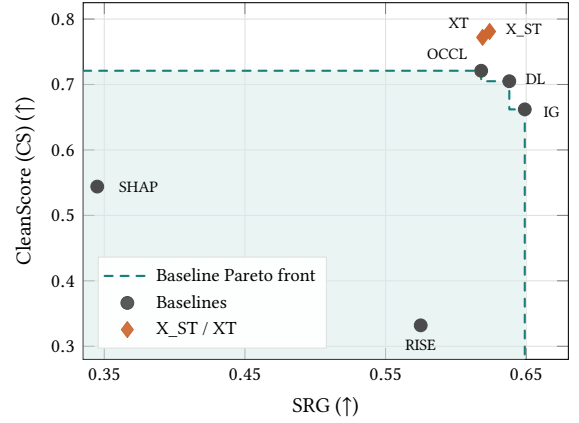
\begin{figure}[t]
\centering
\begin{tikzpicture}
\begin{axis}[
    width=8.0cm,
    height=6.3cm,
    xlabel={SRG ($\uparrow$)},
    ylabel={CleanScore (CS) ($\uparrow$)},
    xmin=0.335, xmax=0.680,
    ymin=0.280, ymax=0.825,
    xtick={0.35,0.45,0.55,0.65},
    ytick={0.30,0.40,0.50,0.60,0.70,0.80},
    grid=major,
    grid style={black!12},
    tick label style={font=\footnotesize},
    label style={font=\small},
    every axis plot/.append style={line join=round, line cap=round},
    legend pos=south west,
    legend cell align=left,
    legend style={
        font=\footnotesize,
        fill=white,
        fill opacity=0.92,
        text opacity=1,
        draw=black!12,
        inner sep=2pt
    },
]

\addplot[
    fill=paretoSRG!16,
    draw=none,
    opacity=0.55,
    forget plot
]
coordinates {
    (0.335,0.280)
    (0.335,0.721)
    (0.618,0.721) 
    (0.618,0.705)
    (0.638,0.705) 
    (0.638,0.662)
    (0.649,0.662) 
    (0.649,0.280)
    (0.335,0.280)
};

\addplot[
    thick,
    dashed,
    paretoSRG,
    mark=none
]
coordinates {
    (0.335,0.721)
    (0.618,0.721) 
    (0.618,0.705)
    (0.638,0.705) 
    (0.638,0.662)
    (0.649,0.662) 
    (0.649,0.280)
};
\addlegendentry{Baseline Pareto front}

\addplot[
    only marks,
    mark=*,
    draw=baseSRG,
    fill=baseSRG,
    mark size=2.4pt
]
coordinates {
    (0.618,0.721) 
    (0.638,0.705) 
    (0.345,0.544) 
    (0.575,0.332) 
    (0.649,0.662) 
};
\addlegendentry{Baselines}

\addplot[
    only marks,
    mark=diamond*,
    draw=xtColor,
    fill=xtColor,
    mark size=3.0pt
]
coordinates {
    (0.624,0.781) 
    (0.619,0.772) 
};
\addlegendentry{X\_ST / XT}

\node[anchor=south east, font=\scriptsize, xshift=-2pt, yshift=1pt]
    at (axis cs:0.618,0.721) {OCCL};

\node[anchor=west, font=\scriptsize, xshift=3pt, yshift=1pt]
    at (axis cs:0.638,0.705) {DL};

\node[anchor=west, font=\scriptsize, xshift=5pt]
    at (axis cs:0.345,0.544) {SHAP};

\node[anchor=north, font=\scriptsize, yshift=-2pt]
    at (axis cs:0.575,0.332) {RISE};

\node[anchor=west, font=\scriptsize, xshift=4pt, yshift=-1pt]
    at (axis cs:0.649,0.662) {IG};

\node[anchor=west, font=\scriptsize, xshift=3pt, yshift=0.5pt]
    at (axis cs:0.624,0.781) {X\_ST};

\node[anchor=south east, font=\scriptsize, xshift=-2pt, yshift=0.5pt]
    at (axis cs:0.619,0.772) {XT};

\end{axis}
\end{tikzpicture}
\caption{Baseline-frontier Pareto analysis of SRG and CleanScore on TMNIST. Higher is better for both metrics.}
\label{fig:pareto_tmnist}
\end{figure}

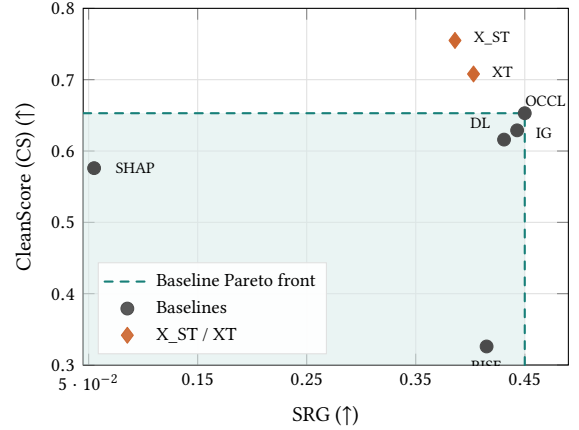
\begin{figure}[t]
\centering
\begin{tikzpicture}
\begin{axis}[
    width=8.0cm,
    height=6.3cm,
    xlabel={SRG ($\uparrow$)},
    ylabel={CleanScore (CS) ($\uparrow$)},
    xmin=0.045, xmax=0.490,
    ymin=0.300, ymax=0.800,
    xtick={0.05,0.15,0.25,0.35,0.45},
    ytick={0.30,0.40,0.50,0.60,0.70,0.80},
    grid=major,
    grid style={black!12},
    tick label style={font=\footnotesize},
    label style={font=\small},
    every axis plot/.append style={line join=round, line cap=round},
    legend pos=south west,
    legend cell align=left,
    legend style={
        font=\footnotesize,
        fill=white,
        fill opacity=0.92,
        text opacity=1,
        draw=black!12,
        inner sep=2pt
    },
]

\addplot[
    fill=paretoSRG!16,
    draw=none,
    opacity=0.55,
    forget plot
]
coordinates {
    (0.045,0.300)
    (0.045,0.653)
    (0.450,0.653) 
    (0.450,0.300)
    (0.045,0.300)
};

\addplot[
    thick,
    dashed,
    paretoSRG,
    mark=none
]
coordinates {
    (0.045,0.653)
    (0.450,0.653) 
    (0.450,0.300)
};
\addlegendentry{Baseline Pareto front}

\addplot[
    only marks,
    mark=*,
    draw=baseSRG,
    fill=baseSRG,
    mark size=2.4pt
]
coordinates {
    (0.450,0.653) 
    (0.431,0.616) 
    (0.055,0.576) 
    (0.415,0.326) 
    (0.443,0.629) 
};
\addlegendentry{Baselines}

\addplot[
    only marks,
    mark=diamond*,
    draw=xtColor,
    fill=xtColor,
    mark size=3.0pt
]
coordinates {
    (0.386,0.755) 
    (0.403,0.708) 
};
\addlegendentry{X\_ST / XT}

\node[anchor=west, font=\scriptsize, xshift=4pt, yshift=0pt]
    at (axis cs:0.433,0.67) {OCCL};

\node[anchor=south east, font=\scriptsize, xshift=-2pt, yshift=1pt]
    at (axis cs:0.431,0.616) {DL};

\node[anchor=west, font=\scriptsize, xshift=5pt]
    at (axis cs:0.055,0.576) {SHAP};

\node[anchor=north, font=\scriptsize, yshift=-2pt]
    at (axis cs:0.415,0.326) {RISE};

\node[anchor=west, font=\scriptsize, xshift=4pt, yshift=-1pt]
    at (axis cs:0.443,0.629) {IG};

\node[anchor=west, font=\scriptsize, xshift=4pt, yshift=1pt]
    at (axis cs:0.386,0.755) {X\_ST};

\node[anchor=west, font=\scriptsize, xshift=4pt, yshift=1pt]
    at (axis cs:0.403,0.708) {XT};

\end{axis}
\end{tikzpicture}
\caption{Baseline-frontier Pareto analysis of SRG and CleanScore on TMNIST\_L. Higher is better for both metrics.}
\label{fig:pareto_tmnist_l}
\end{figure}

\begin{table*}[!ht]
    \caption{Quantitative analysis on the proposed metric, SRG $\angle$ CS for the Typeface MNIST (TMNIST) and Typeface MNIST Lines (TMNIST\_L) datasets. $XtrAIn$ (XT) and $Xstep$ (X\_ST) outperform other occlusion-based methods in this score by a margin. Best results are highlighted in bold.}
    \vspace{-10pt}
    \label{tab:main_results}
    \vskip 0.15in
    \begin{center}
    \begin{footnotesize}
    \begin{sc}
    \setlength{\tabcolsep}{3pt}
    \begin{tabular}{lccccccc}
    \toprule
    & \multicolumn{1}{c}{OCCL} & \multicolumn{1}{c}{DL} & \multicolumn{1}{c}{SHAP} & \multicolumn{1}{c}{RISE} & \multicolumn{1}{c}{IG} & \multicolumn{1}{c}{X\_ST} & \multicolumn{1}{c}{XT}\\

    Dataset & SRG $\angle$ CS ($\uparrow$) & SRG $\angle$ CS ($\uparrow$) & SRG $\angle$ CS ($\uparrow$) & SRG $\angle$ CS ($\uparrow$) & SRG $\angle$ CS ($\uparrow$) & SRG $\angle$ CS ($\uparrow$) & SRG $\angle$ CS ($\uparrow$) \\
    \midrule
    TMNIST & 0.950 & 0.951 &  0.644 & 0.664 & 0.927 & \textbf{1.0} & 0.99 \\
    TMNIST\_L & 0.793 & 0.752 &  0.579 & 0.528 & 0.770 & \textbf{0.850} & 0.815  \\
    \bottomrule
    \end{tabular}
    \end{sc}
    \end{footnotesize}
    \end{center}
    \vskip -0.1in
\end{table*}

In this regard, $XtrAIn$ has the capacity to exploit the analysis of the loss into the two complementary terms according to 
\cref{eq:loss_decomp}, and apply the update rule to $\mathcal{L}_{target}$. A further step in this analysis potentially removes any negative updates 
that decreased the target-class logit for a sample $x$, thus prioritizing patterns used explicitly for the emergence of the target neuron. Thus, the updates 
within the calculation of $\mathcal{R}_{T}^+$ that lead to a decrease of the target neuron activation are filtered out. This leads to the calculation of 
positive target attribution, defined as $XtrAIn_T^+$. 

However, due to the simultaneous application of both target and non-target effects in parameter updates, any attempt to isolate and measure the individual
effects of each change is inherently constrained; the result is a theoretical construct, pointing to the aggregated \textit{tendency} of the model towards 
strengthening the target class. Despite this, these scores remain insightful for validating the model's expression of causes and enhancing trust in them.

\subsection{CleanScore}
\label{cleansc}
We introduce \textit{CleanScore}, a metric designed to detect noise in attribution maps, which in case of occlusion-based approaches, can partially capture biases resulting from attribution shift. On simple datasets with pixel-level attributions, this bias appears mainly as small stochastic perturbations rather than systematic mean shifts. These perturbations inflate attribution variance within both the signal region and background, providing a proxy for method-induced noise.

We exploit the structure of synthetic datasets commonly used to evaluate FCNNs, in which the image signal is confined to a central region $\mathcal{C}$, with $\mathcal{B} = \mathcal{C}^c$ denoting the background, where by construction the model has no informative signal to read. This partition in $\mathcal{C}$ and $\mathcal{B}$ is achieved by thresholding on the mean value of the aggregated training data. Thus, given a signed attribution map $h \in \mathbb{R}^{H \times W}$ normalized to [-1,1], we define:
\vspace{-1pt}
\begin{equation}
    \text{CleanScore}(h) = \underbrace{\frac{\sum_{i \in \mathcal{C}} |h_i|}{\sum_{i} |h_i|}}_{\text{concentration}} \cdot \underbrace{\frac{1}{1 + \hat{\sigma}_{\mathcal{C}}^+}}_{\text{positive uniformity}} \cdot \underbrace{\frac{1}{1 + \hat{\sigma}_{\mathcal{B}}}}_{\text{background silence}}
\end{equation}

\noindent where $\hat{\sigma}_{\mathcal{C}}^+$ and $\hat{\sigma}_{\mathcal{B}}$ are the standard deviations of positive attributions within $\mathcal{C}$ and all attributions within $\mathcal{B}$, respectively.

The \textit{concentration} term penalizes methods that diffuse attribution mass into the background, where, by construction, no informative signal exists. The \textit{positive uniformity} and \textit{background silence} terms use intra-region standard deviation as a comparative proxy for noise. The underlying intuition is that observed variance decomposes as $\sigma^2_{obs} \approx \sigma_{true}^2+\sigma_{noise}^2$ where $\sigma_{true}$ reflects genuine attribution structure and $\sigma_\text{noise}$ reflects occlusion-induced perturbations. Under the assumption that faithful methods broadly agree on which pixels carry signal — defensible in the simple-dataset regime, where informative pixels are well-defined — differences in $\sigma_\text{obs}$ across methods are attributable to differences in $\sigma_\text{noise}$. Therefore, CleanScore is intended as a comparative measure rather than an absolute one: it ranks methods by their relative noise levels rather than certifying any as bias-free.

Negative attributions within $\mathcal{C}$ are excluded from the uniformity term for simplicity, as the positive term is sufficient. All three factors lie in $(0,1]$, and their product gives a scalar in $(0, 1]$, with higher values indicating cleaner attributions.
\vspace{-1pt}
\section{Evaluation}
\label{sec:appl}

This section describes the experimental setup used to evaluate $XtrAIn$ and its variants. We conduct experiments on a standard FCNN architecture with ReLU activation for the intermediate layer and softmax for the output layer. The models have four layers with gradually reduced neurons and are trained until convergence (typically achieving test accuracy above 90\% within 10 epochs) on an NVIDIA GeForce RTX 4060 GPU.

\begin{figure*}[h]
    \centering
    
    \includegraphics[width=0.9\textwidth]{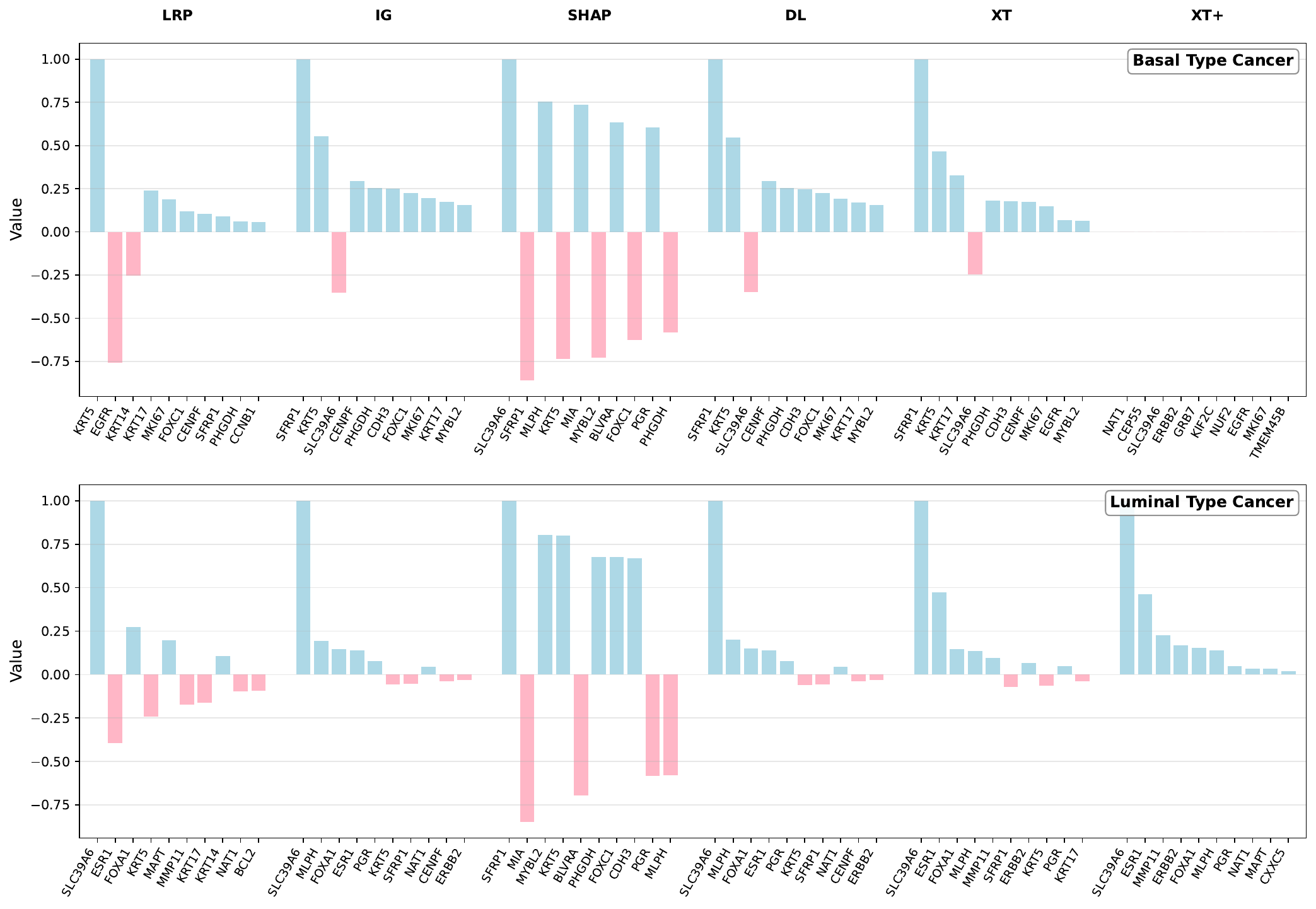} \\
    \vspace{-5pt}
    \caption{Bar plots for the top 10 normalized features of the aggregated attribution scores for each class, for the methods SHAP, LRP, IG, DL, $XtrAIn$ (XT) and $XtrAIn^+$ (XT+). In a controlled setting, we interrupt training at the second epoch, when only the Luminal Type class has been learned; $XtrAIn^+$ then successfully assigns zero importance to all features for this class.}
    \label{fig:bar_plot}
    \Description{A large plot with grouped bar plots for the following attribution methods: SHAP, LRP, IG, DL, $XtrAIn$ (XT) and $XtrAIn^+$ (XT+). 
    For each method the five most influential features it considers are plotted as bars. Most methods point to similar features.}
\end{figure*}

We select standard synthetic datasets for explainability of simple architectures while also satisfying the constraints of the CleanScore metric: Typeface MNIST (TMNIST) \cite{tmnist_vyawahare} with digits rendered from Google Fonts, Typeface MNIST Lines (TMNIST\_L) with random added lines on the TMNIST dataset, and AffineMNIST (AMNIST) \cite{affmnist_Mader} with affine transformations applied to MNIST. The latter is selected as a theoretical basis for interpretability of complex datasets where FCNNs cannot converge (achieving at best 25\% accuracy after 20 epochs) — thus excluded from quantitative evaluation. All images have size 28×28 with values normalized to [0.5, 1] to ensure positive input values across all pixels. 

We compare $XtrAIn$ ($XT$) and $Xstep$ ($X\_ST$) against common occlusion methods: Shapley values (SHAP) \cite{10.5555/3295222.3295230}, Integrated Gradients ($IG$) \cite{sundararajan2017axiomaticattributiondeepnetworks} (as it can be expressed as an occlusion-based method \cite{geiger2025causalabstractiontheoreticalfoundation}), DeepLift ($DL$) \cite{shrikumar2019learningimportantfeaturespropagating} and Randomized Input Sampling ($RISE$) \cite{petsiuk2018riserandomizedinputsampling}. The parameterization of these methods follows the standard setup of the Captum library. For SHAP, we use a uniform baseline fill, as it yielded the best results.

We use SRG \cite{bluecher2024decoupling} as the evaluation metric to compare attribution methods, selected for its robustness and stability in AUC score calculation. This score is computed as:
\vspace{-5pt}
\begin{equation}
    SRG(x) = LIF(x) - MIF(x),
\end{equation}

where LIF and MIF compute the AUC score of the deletion curve starting from the least and most important features respectively — according to the ranking produced by an attribution method. 

However, SRG is itself computed via occlusion, and so inherits the same OoD and artifact introduction problems that affect occlusion-based attribution methods. This creates a structural overlap between what these methods optimize for — the change in model output under feature masking — and what the metric measures. As a result, methods built on occlusion are expected to score well on SRG largely by construction, independent of whether they identify features the model actually relies on. 

We therefore do not consider the competitive, but not leading performance on SRG as a failure mode for our method. 
To break this circularity, we pair SRG with CleanScore, forming an orthogonal measure of interpretability accuracy and bias. We apply it to 100 test samples and report the norm of the two scores in \cref{tab:main_results}. The resulting Baseline-frontier Pareto plots are shown in \cref{fig:pareto_tmnist,fig:pareto_tmnist_l}. Additional attribution visualizations for popular occlusion-based methods are shown in \cref{fig:both_res}, together with $XtrAIn$ and its variants.

For the real-world dataset, we selected the Prediction Analysis of Microarray 50 (PAM50) dataset \cite{doi:10.1200/JCO.2008.18.1370, gdc_portal}, a 50-gene signature assay used to classify breast cancer based on its intrinsic molecular subtype. We followed the same process as in \cite{METSCH2025110124} to prepare the dataset for binary classification with basal-like or luminal-like classes. In this setting, the model converges within four epochs, reaching perfect accuracy. However, we observe a spurious strategy: for the first two epochs, it learns only the luminal-like class while ignoring the basal class, and only later starts learning the second class. We use this behavior to test whether the methods detect dynamics. The top 10 aggregated features per class are shown in \cref{fig:bar_plot}.

\begin{figure*}[b!]
    \centering
    \setlength{\tabcolsep}{2pt}  

    \begin{tabular}{c c c c c c c c}
        \textbf{Sample} & \textbf{OCCL} & \textbf{SHAP} & \textbf{DL} & \textbf{IG} &
        \textbf{X\_ST} & \textbf{XT} & \textbf{XT+} \\[3pt]

        \includegraphics[width=0.115\textwidth]{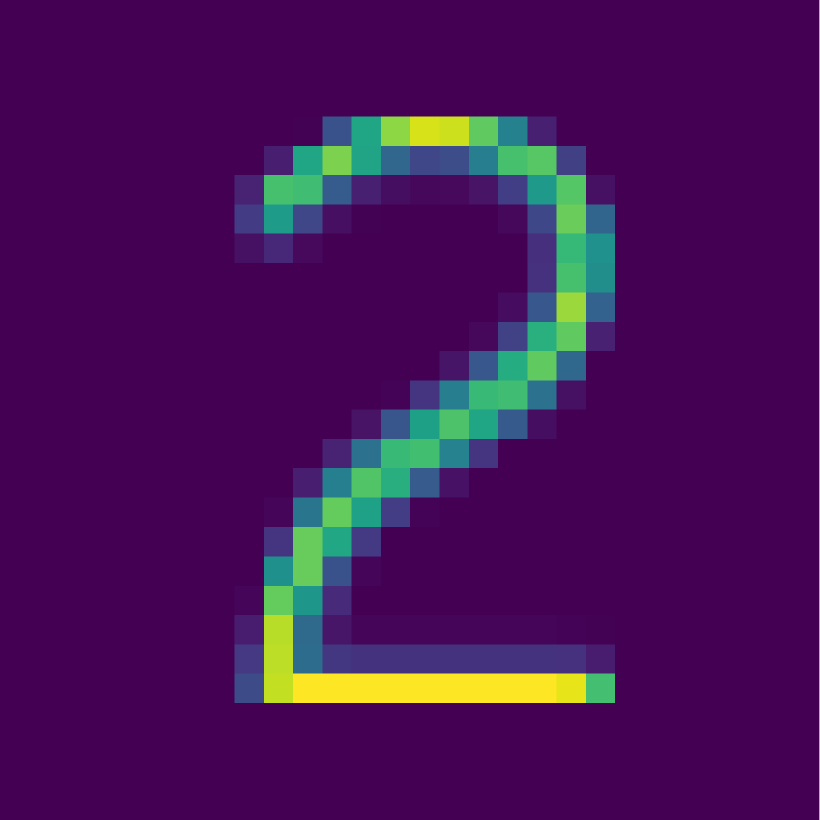} &
        \includegraphics[width=0.115\textwidth]{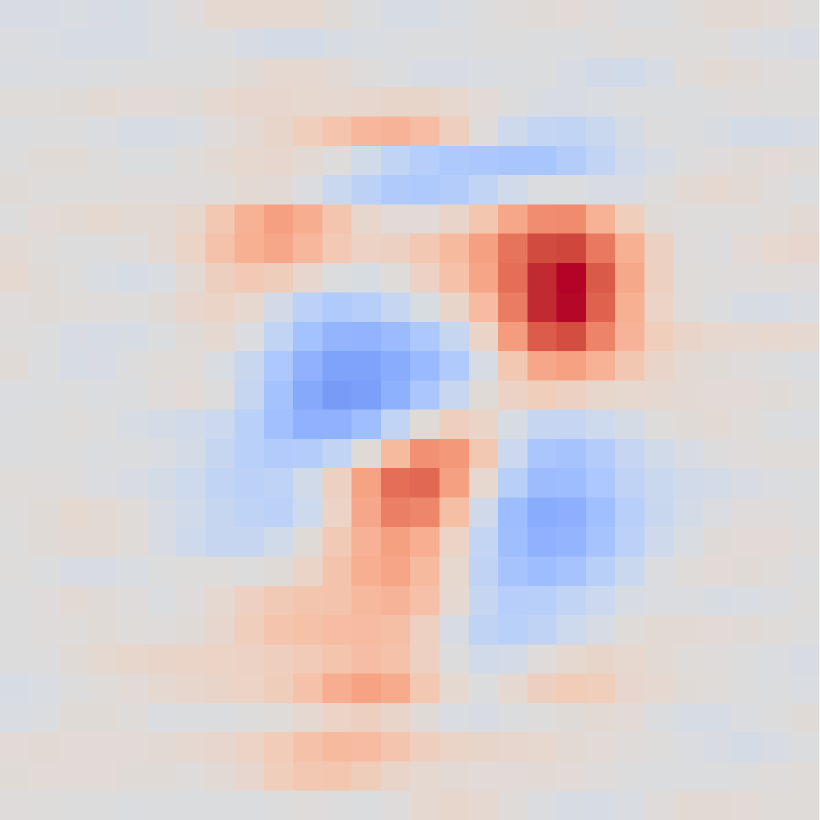} &
        \includegraphics[width=0.115\textwidth]{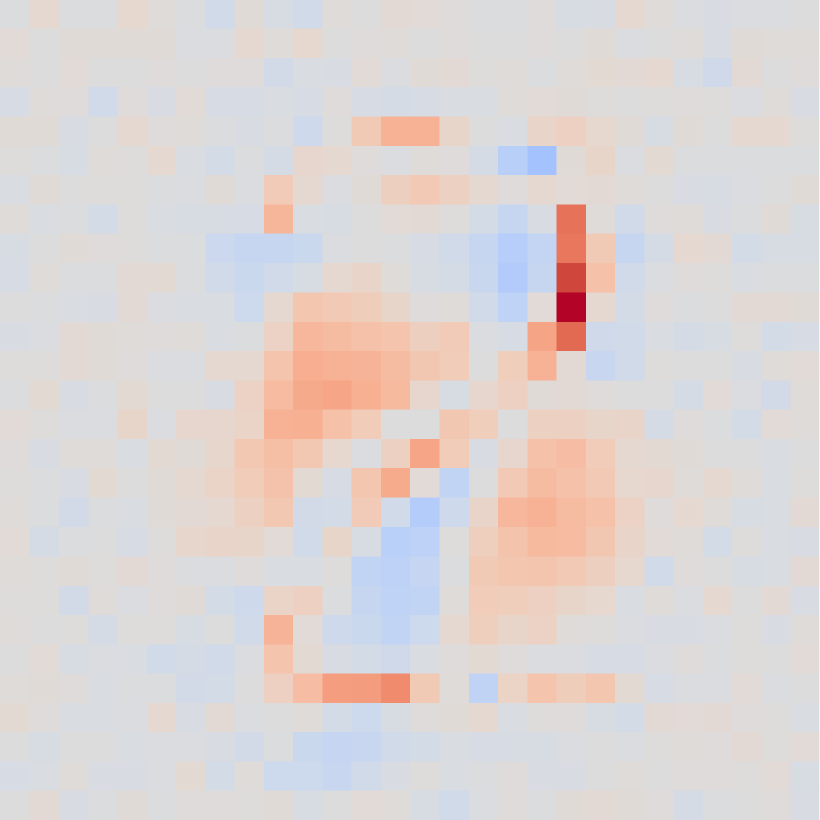} &
        \includegraphics[width=0.115\textwidth]{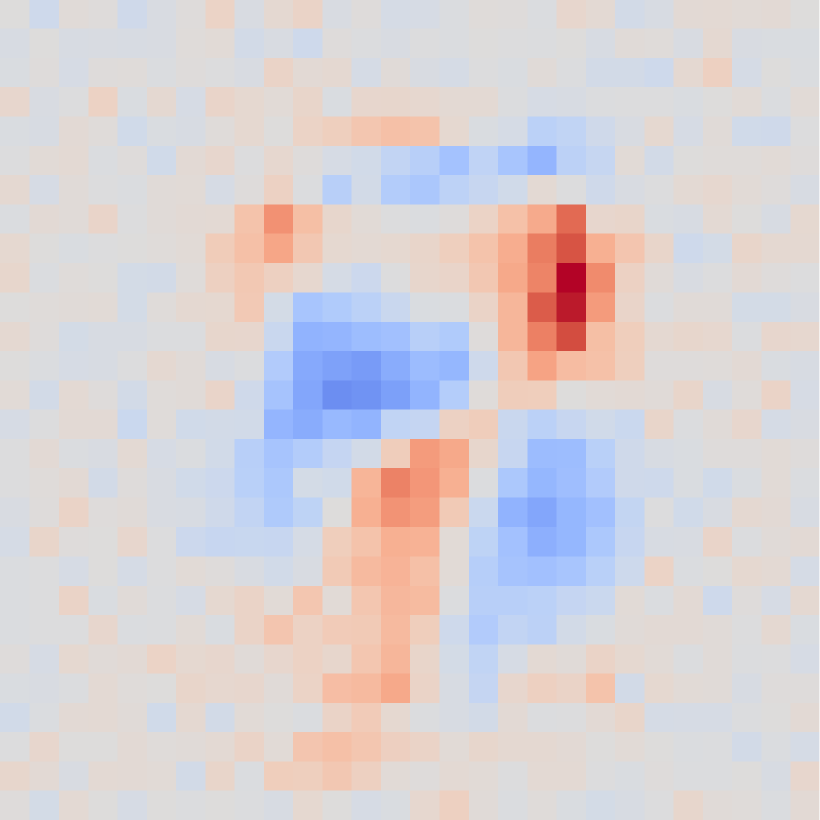} &
        \includegraphics[width=0.115\textwidth]{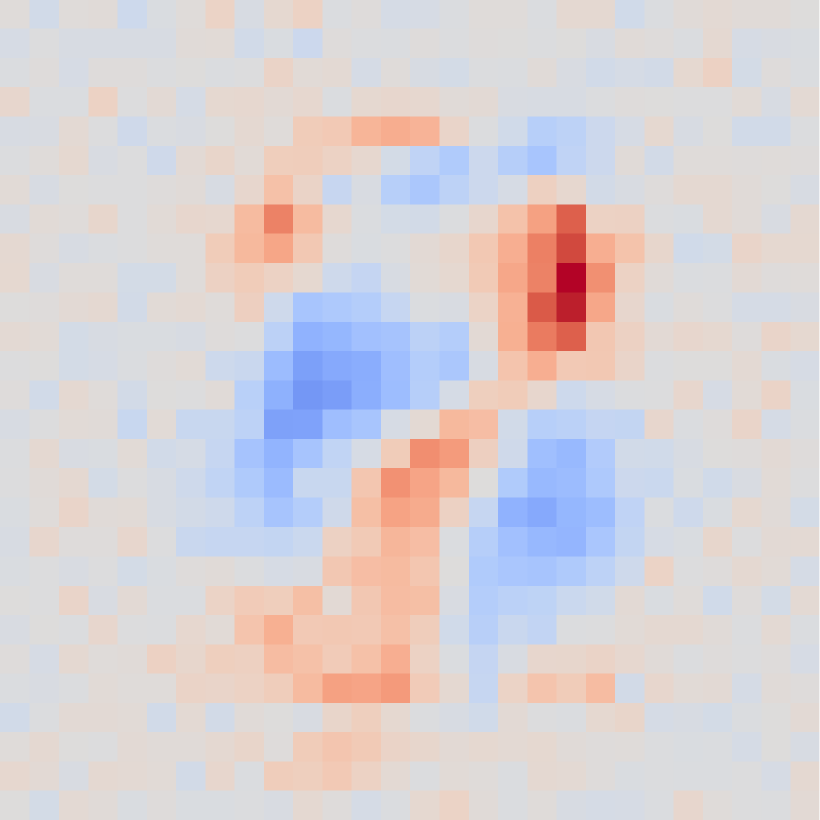} &
        \includegraphics[width=0.115\textwidth]{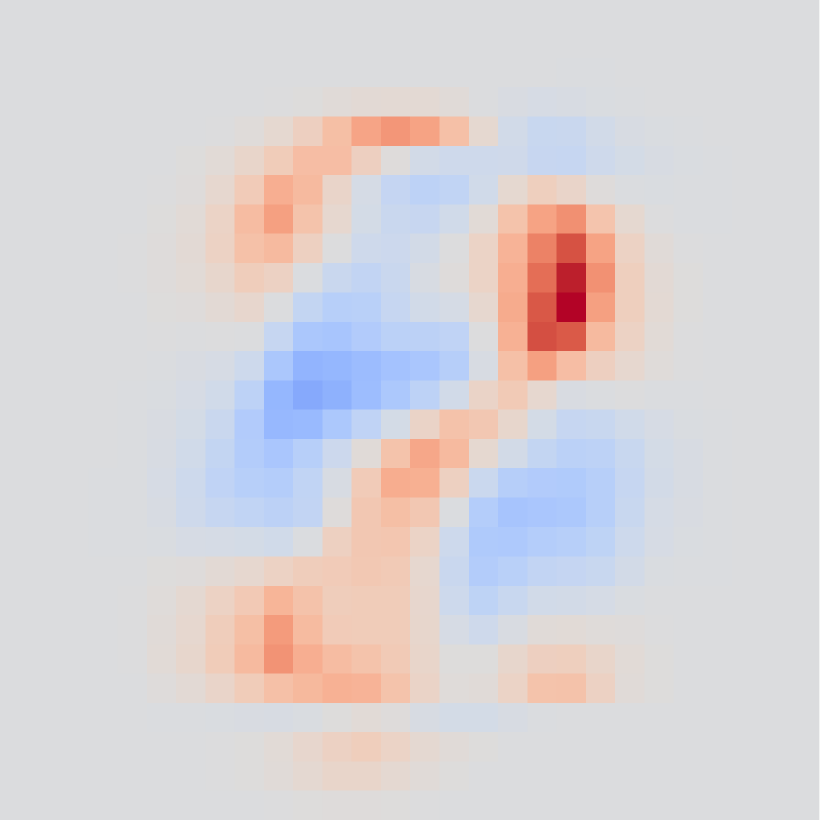} &
        \includegraphics[width=0.115\textwidth]{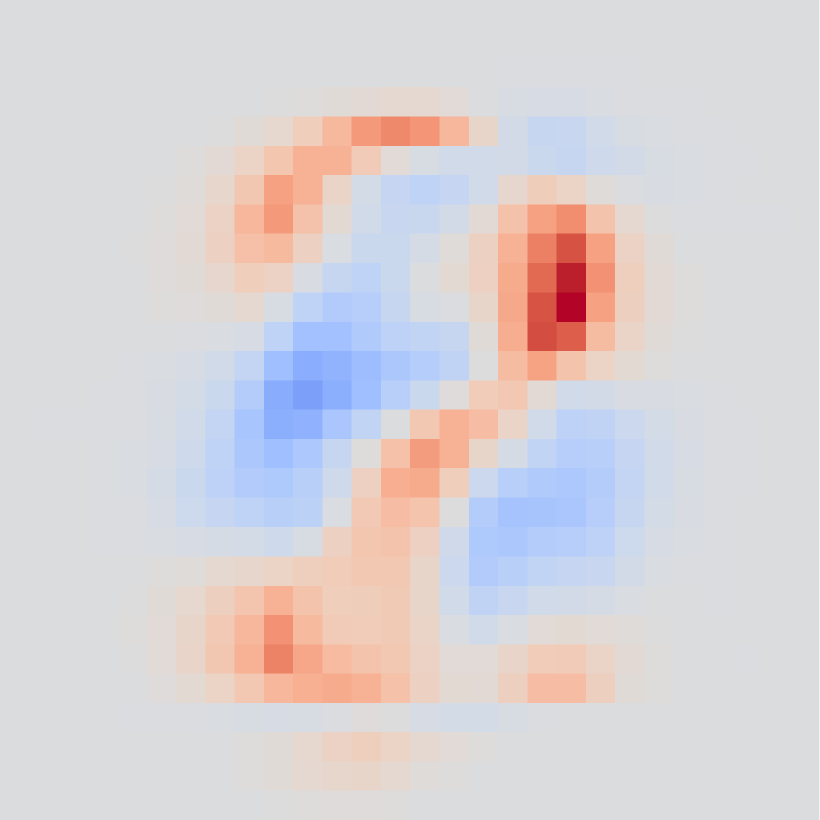} &
        \includegraphics[width=0.115\textwidth]{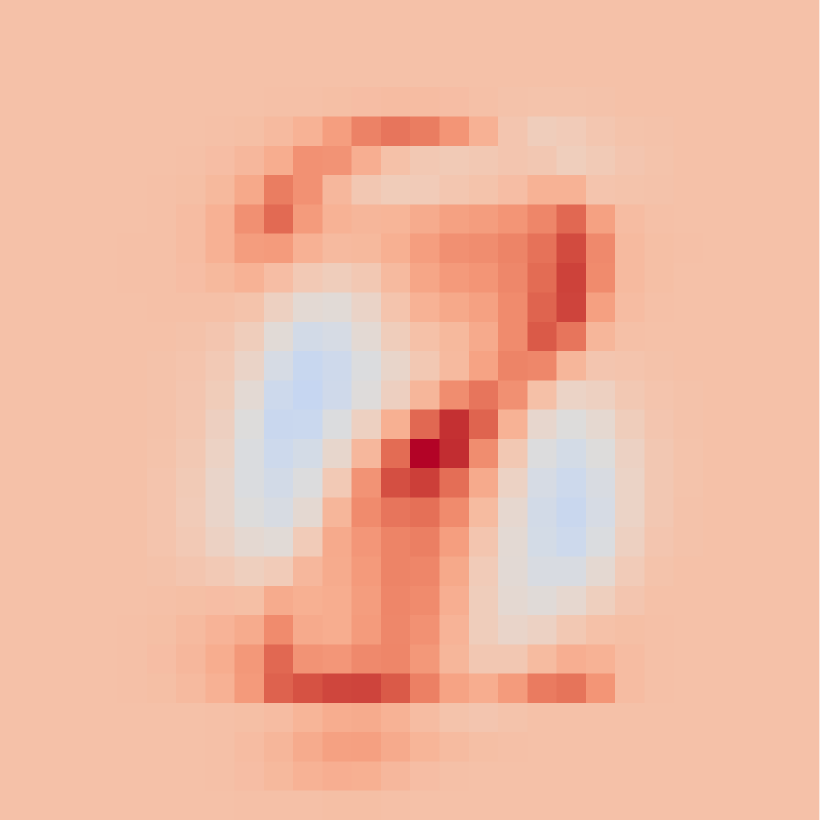} \\[3pt]

        \includegraphics[width=0.115\textwidth]{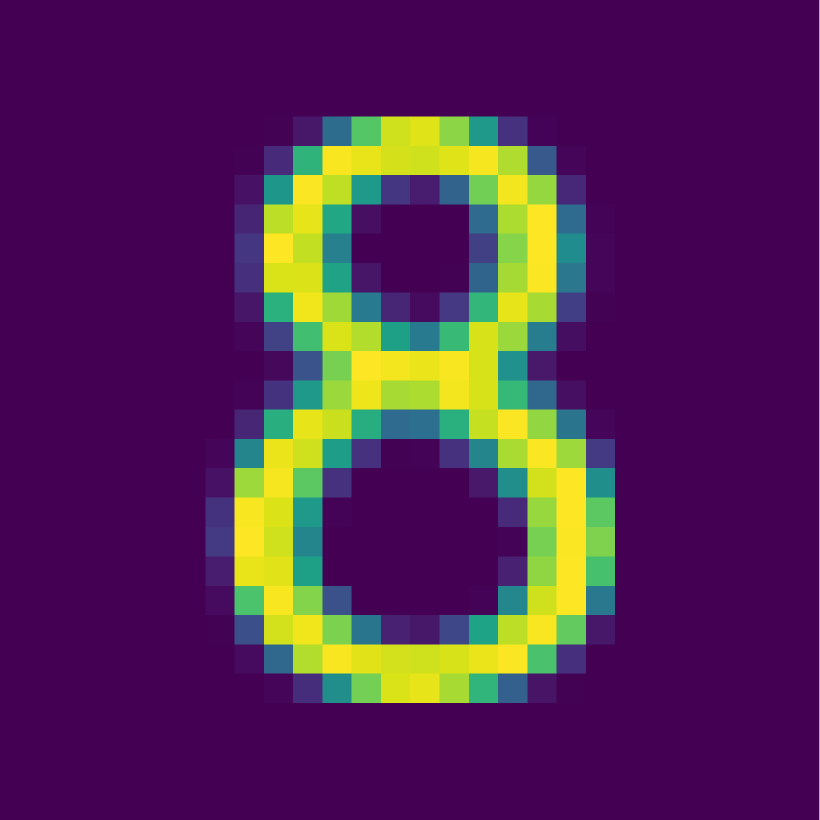} &
        \includegraphics[width=0.115\textwidth]{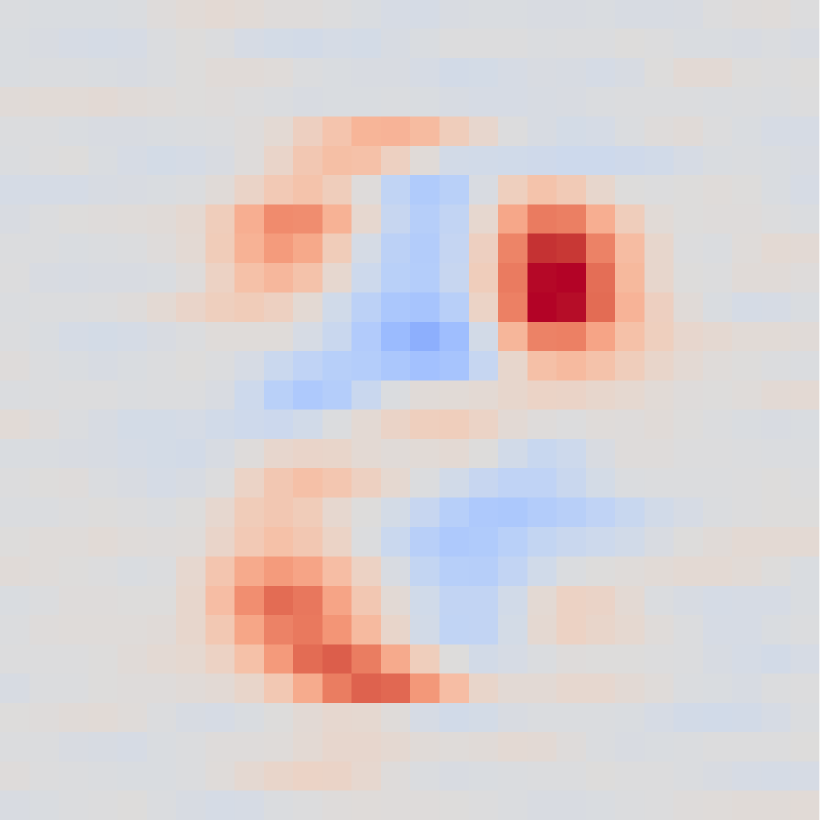} &
        \includegraphics[width=0.115\textwidth]{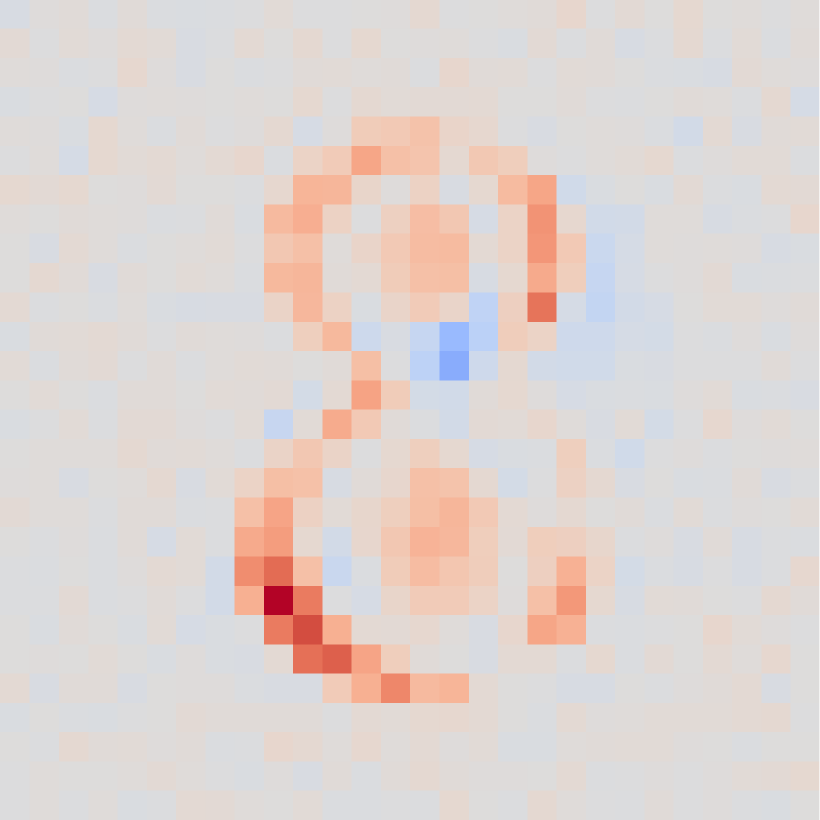} &
        \includegraphics[width=0.115\textwidth]{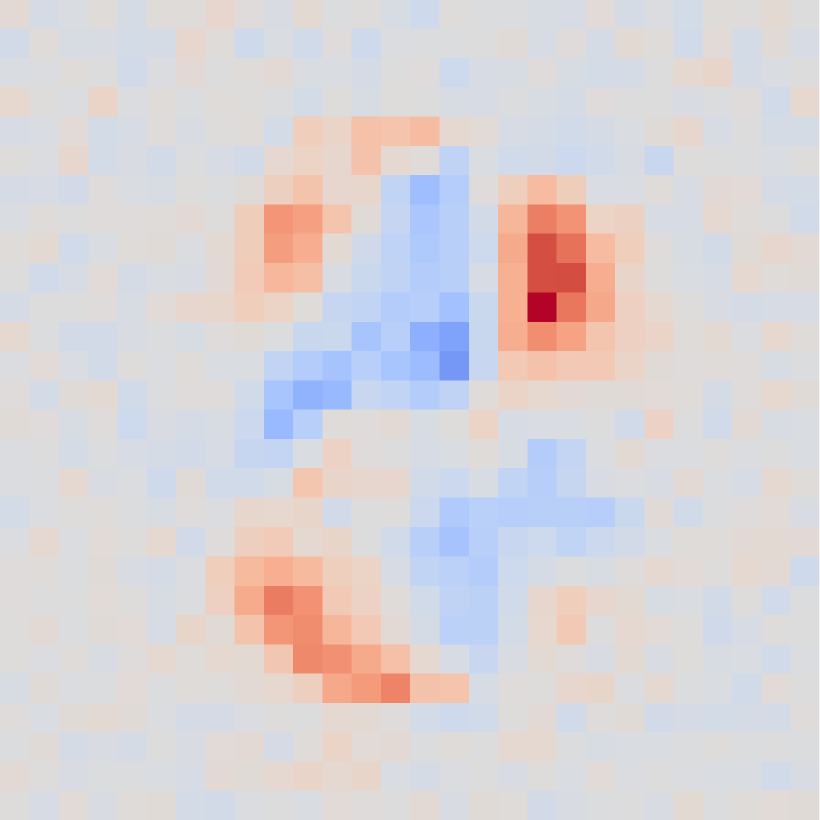} &
        \includegraphics[width=0.115\textwidth]{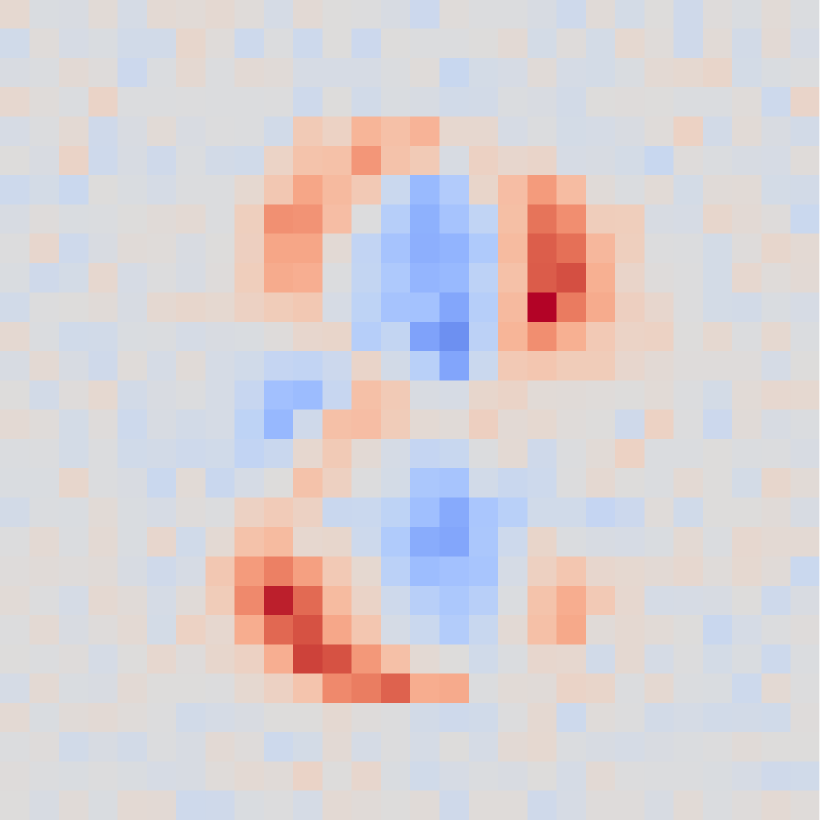} &
        \includegraphics[width=0.115\textwidth]{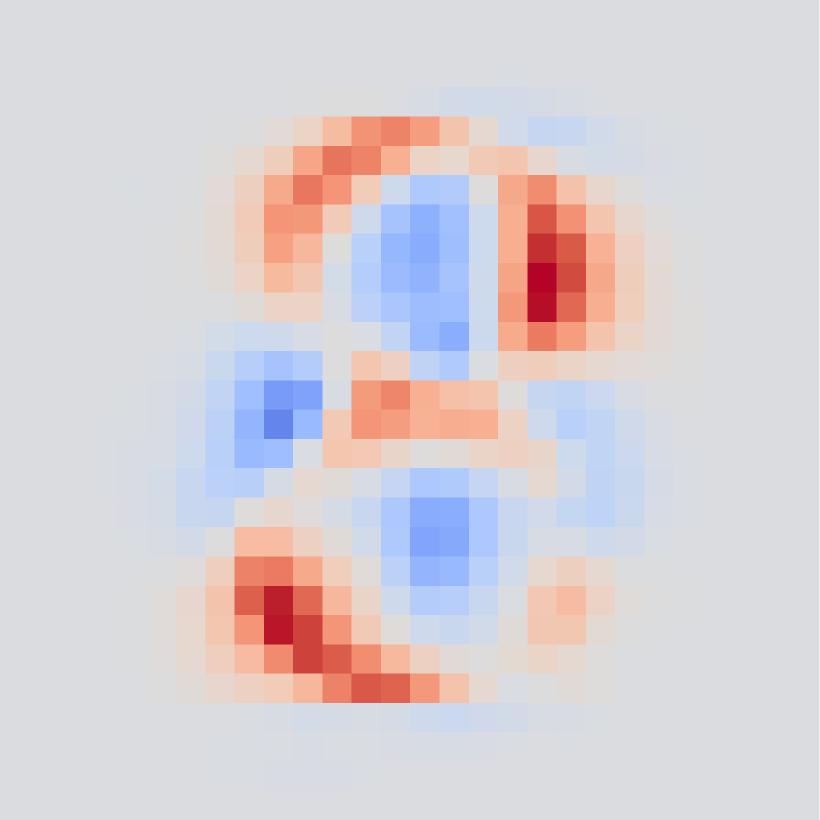} &
        \includegraphics[width=0.115\textwidth]{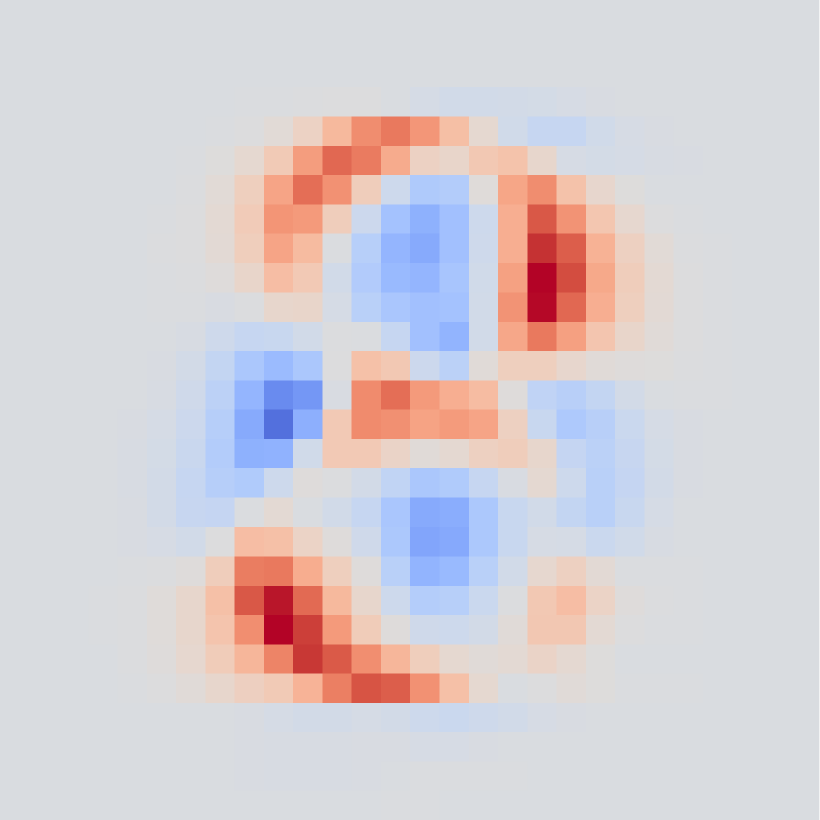} &
        \includegraphics[width=0.115\textwidth]{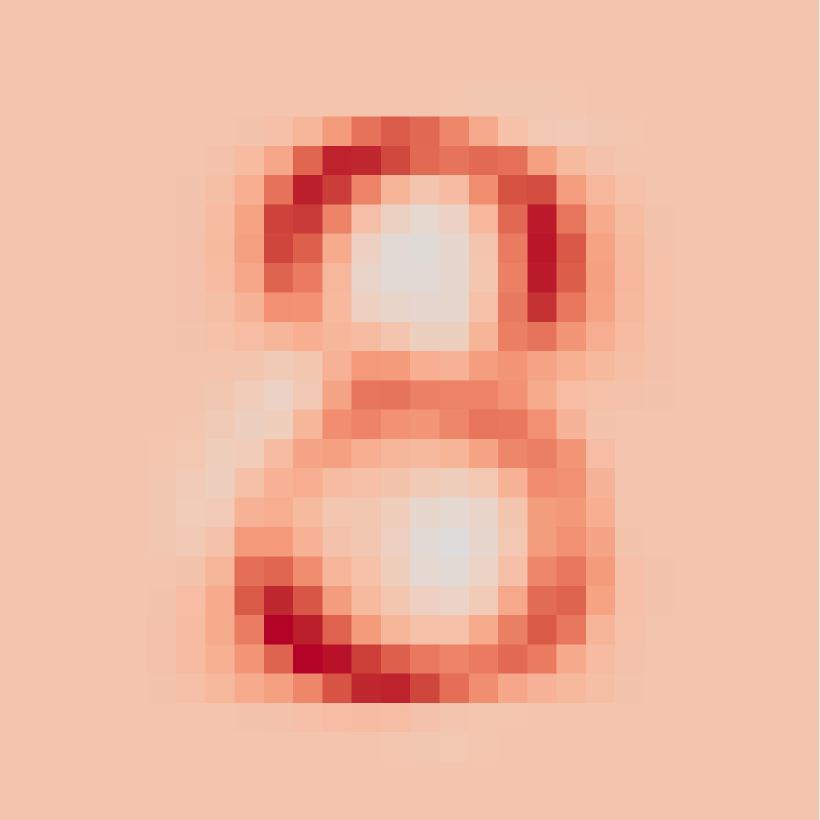} \\[3pt]

        \includegraphics[width=0.115\textwidth]{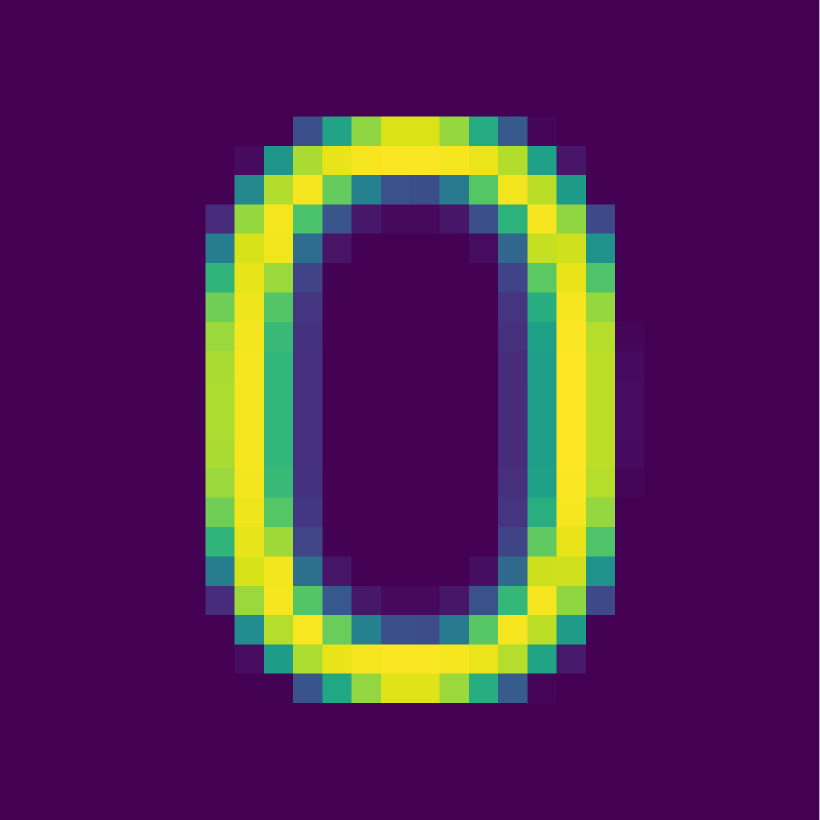} &
        \includegraphics[width=0.115\textwidth]{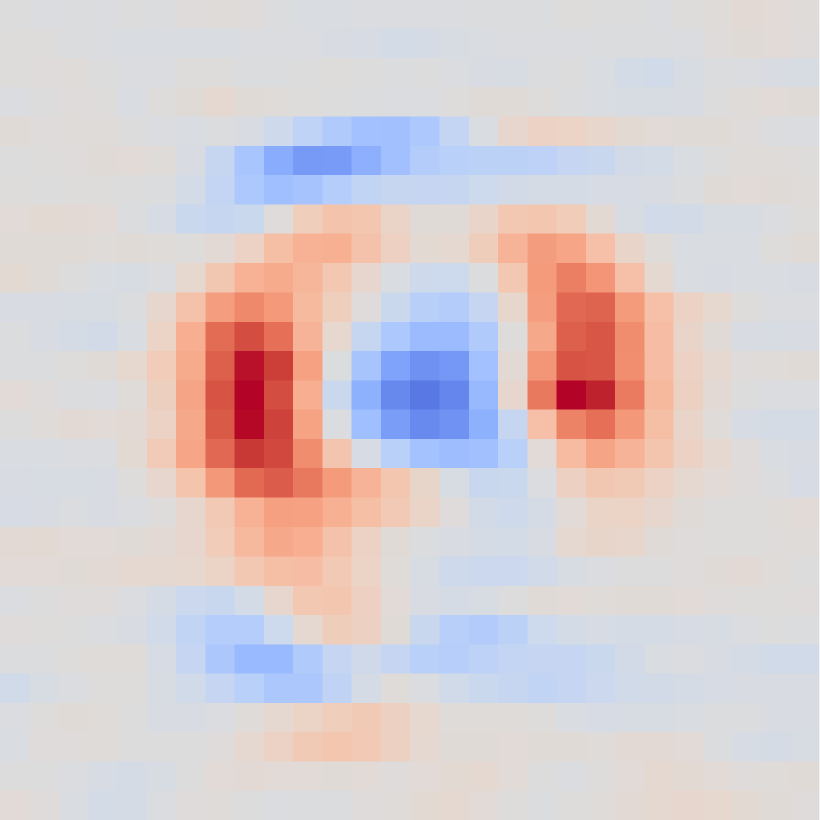} &
        \includegraphics[width=0.115\textwidth]{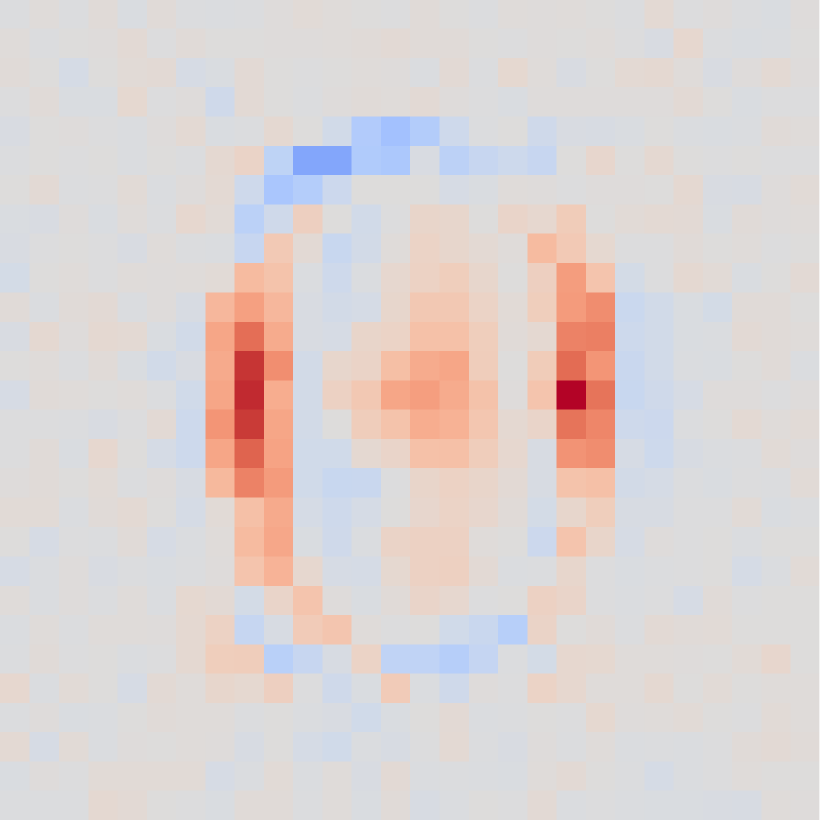} &
        \includegraphics[width=0.115\textwidth]{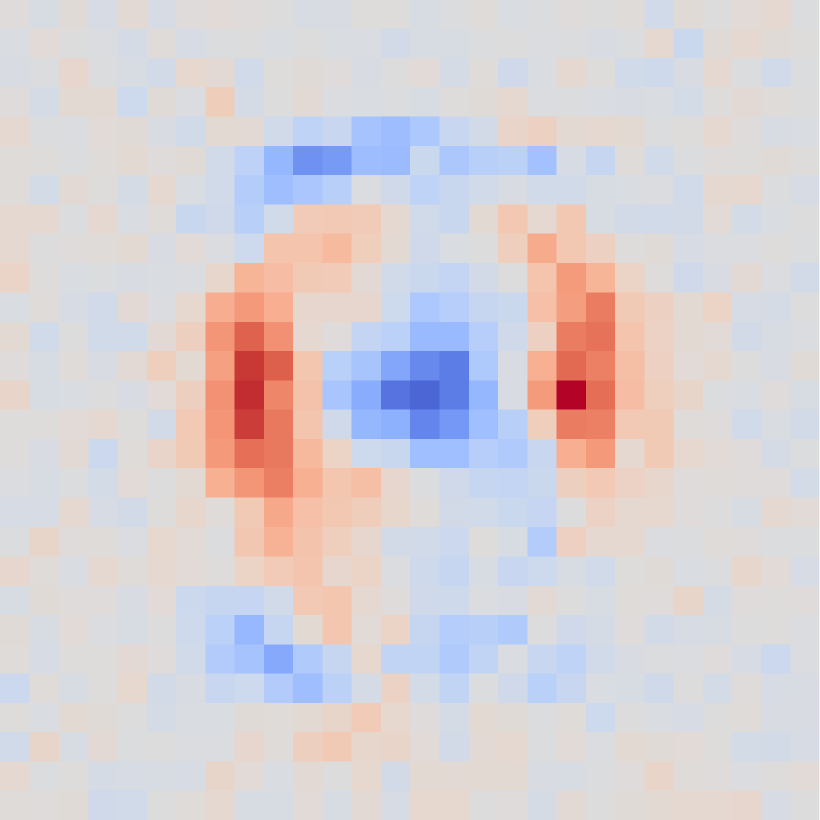} &
        \includegraphics[width=0.115\textwidth]{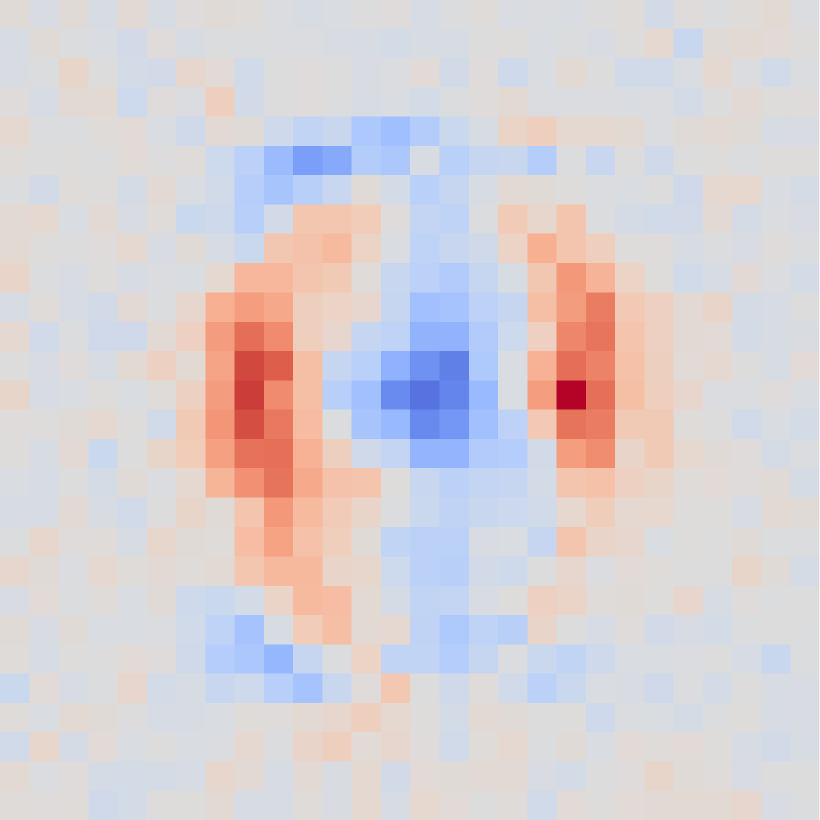} &
        \includegraphics[width=0.115\textwidth]{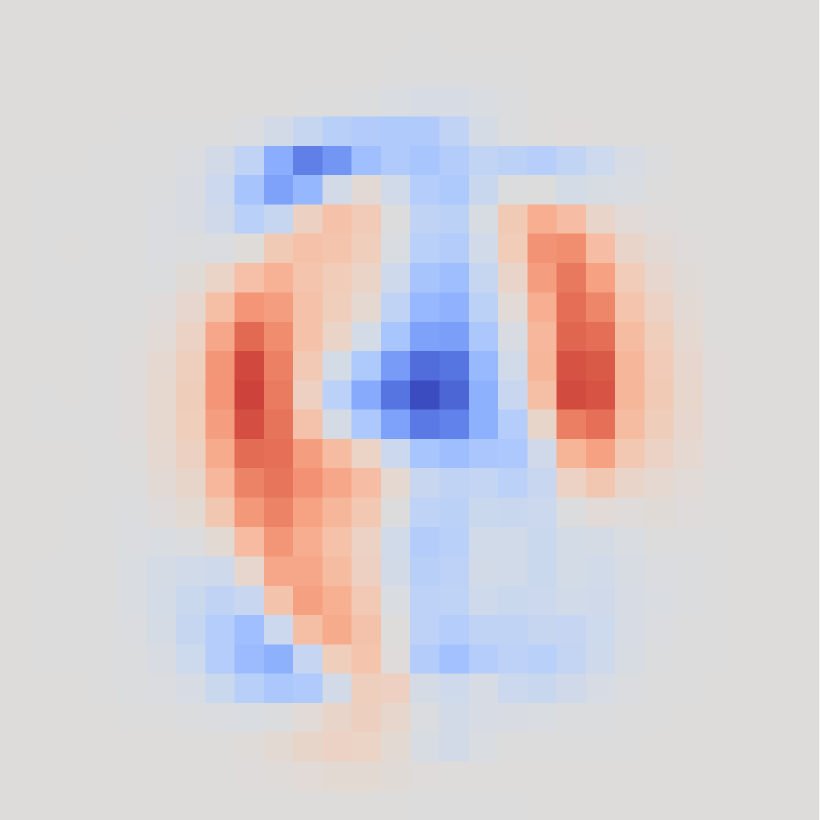} &
        \includegraphics[width=0.115\textwidth]{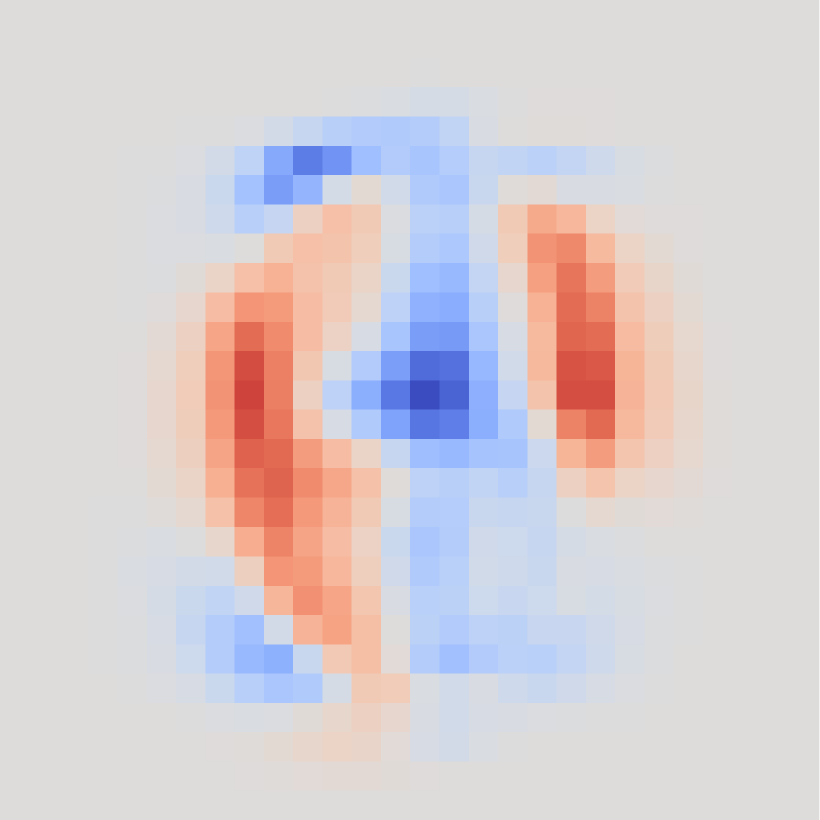} &
        \includegraphics[width=0.115\textwidth]{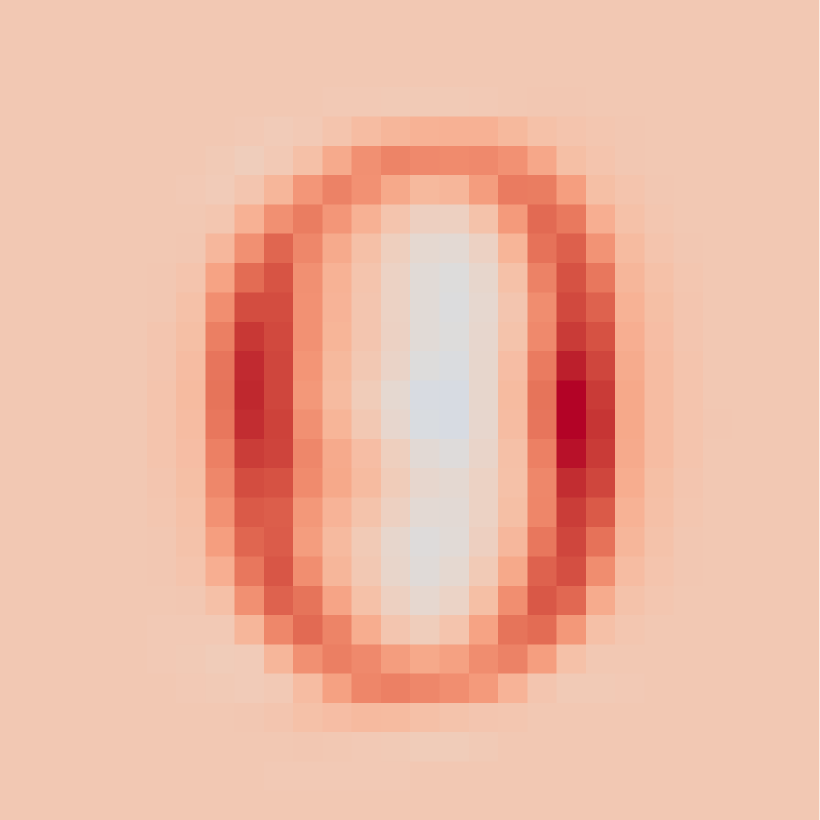} \\[3pt]

        \includegraphics[width=0.115\textwidth]{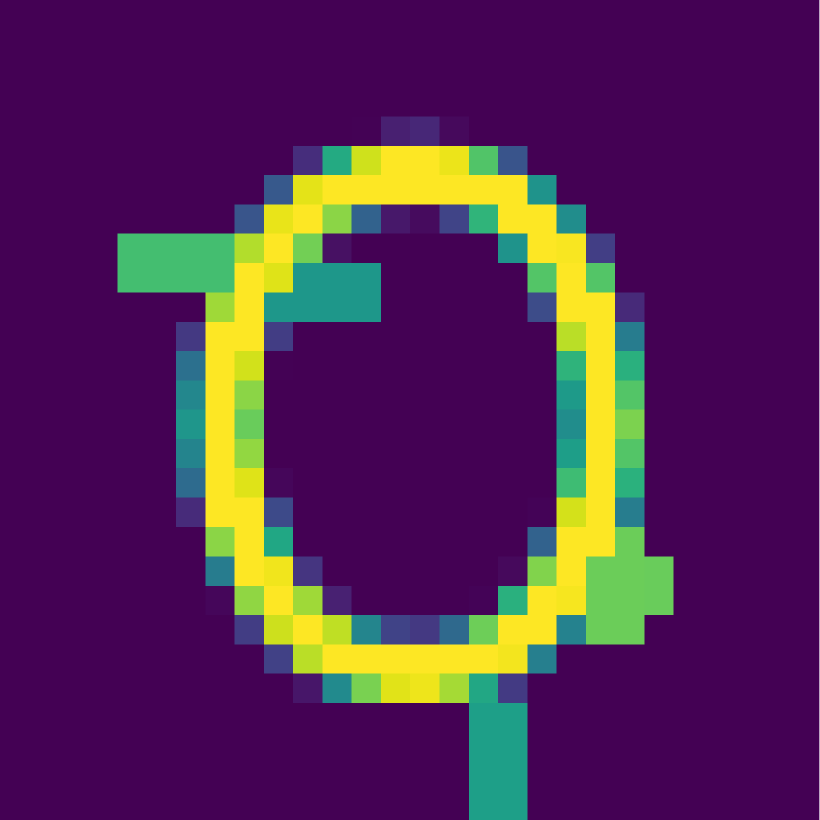} &
        \includegraphics[width=0.115\textwidth]{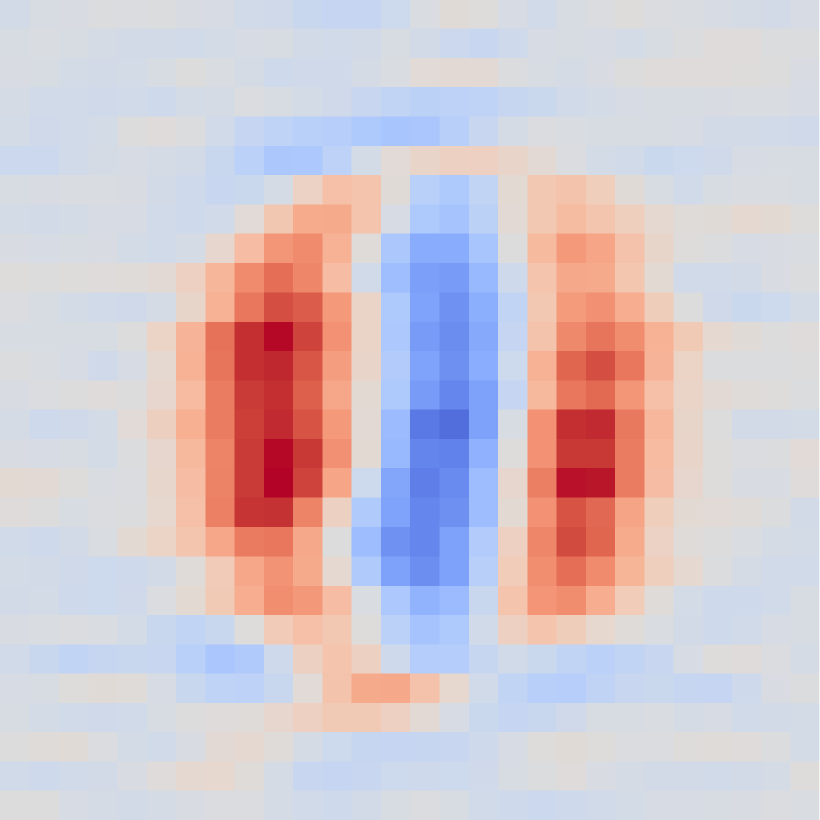} &
        \includegraphics[width=0.115\textwidth]{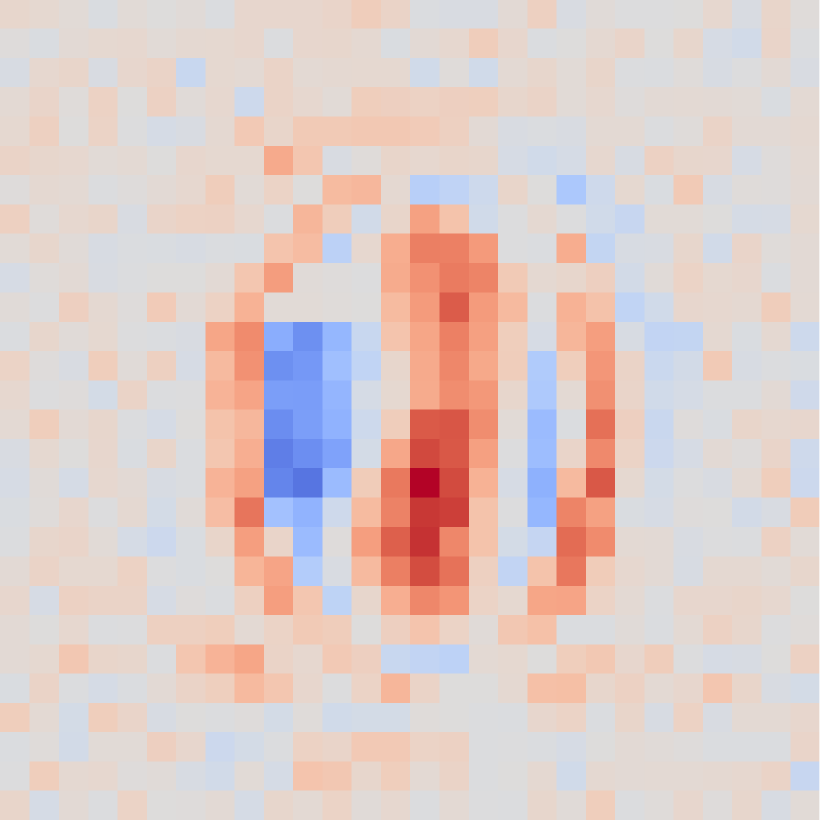} &
        \includegraphics[width=0.115\textwidth]{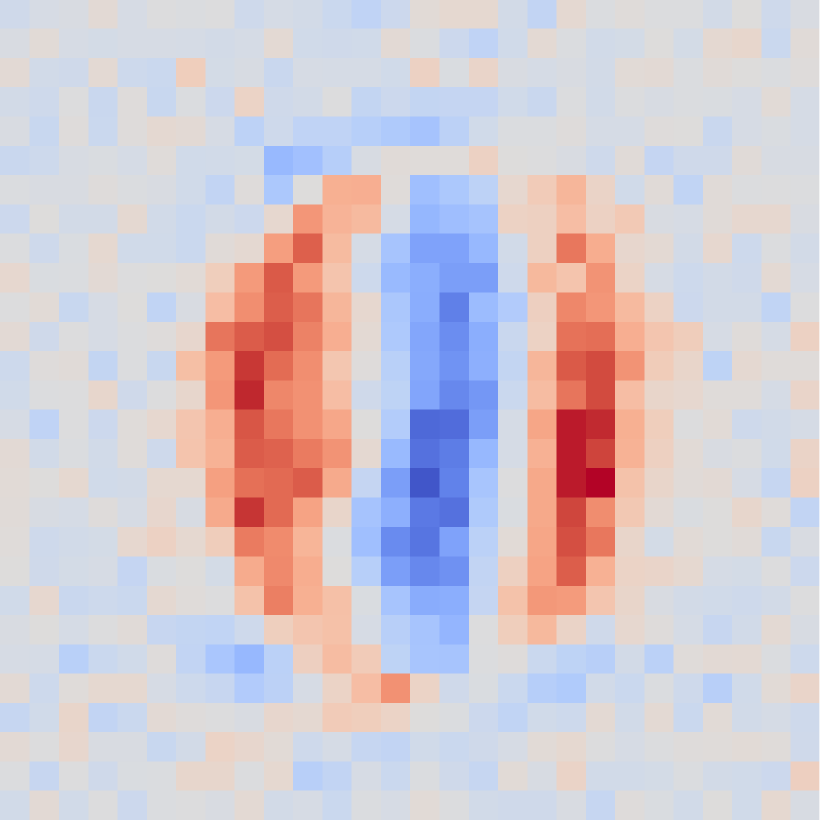} &
        \includegraphics[width=0.115\textwidth]{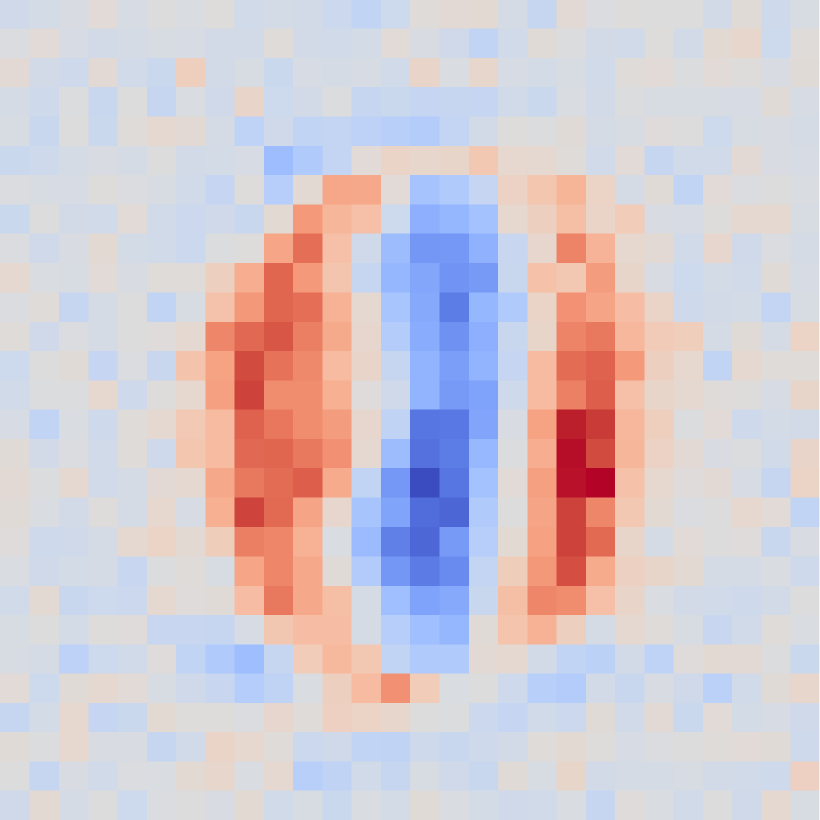} &
        \includegraphics[width=0.115\textwidth]{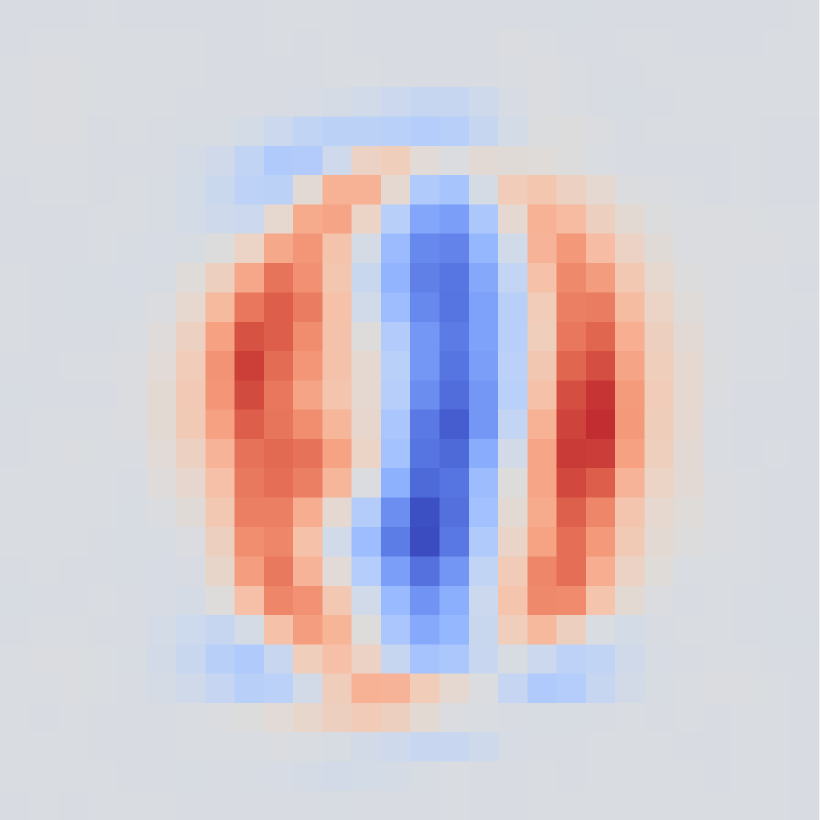} &
        \includegraphics[width=0.115\textwidth]{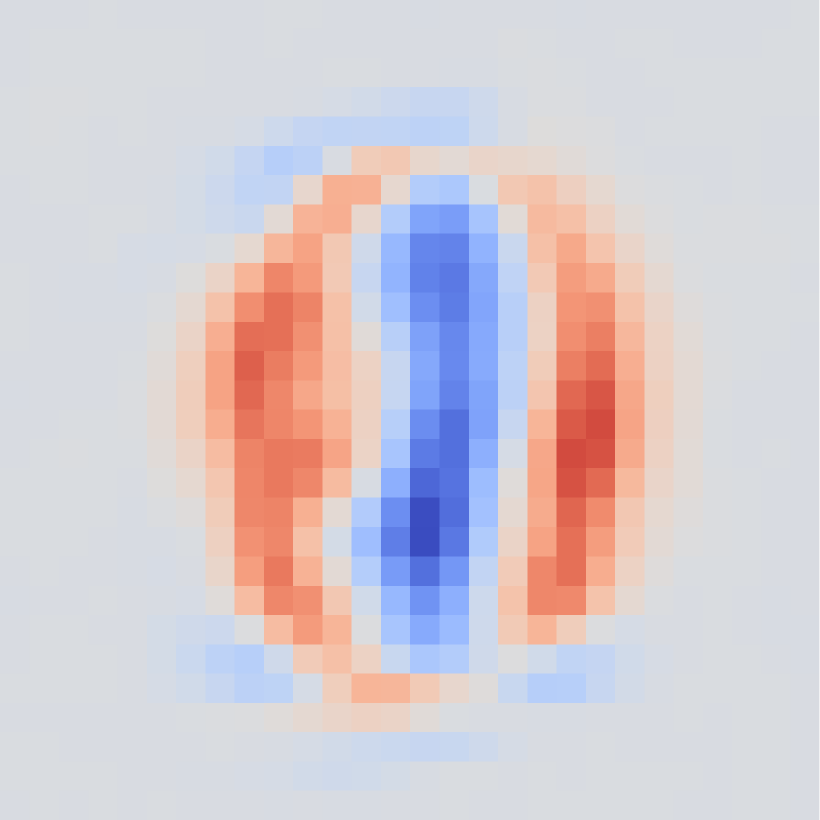} &
        \includegraphics[width=0.115\textwidth]{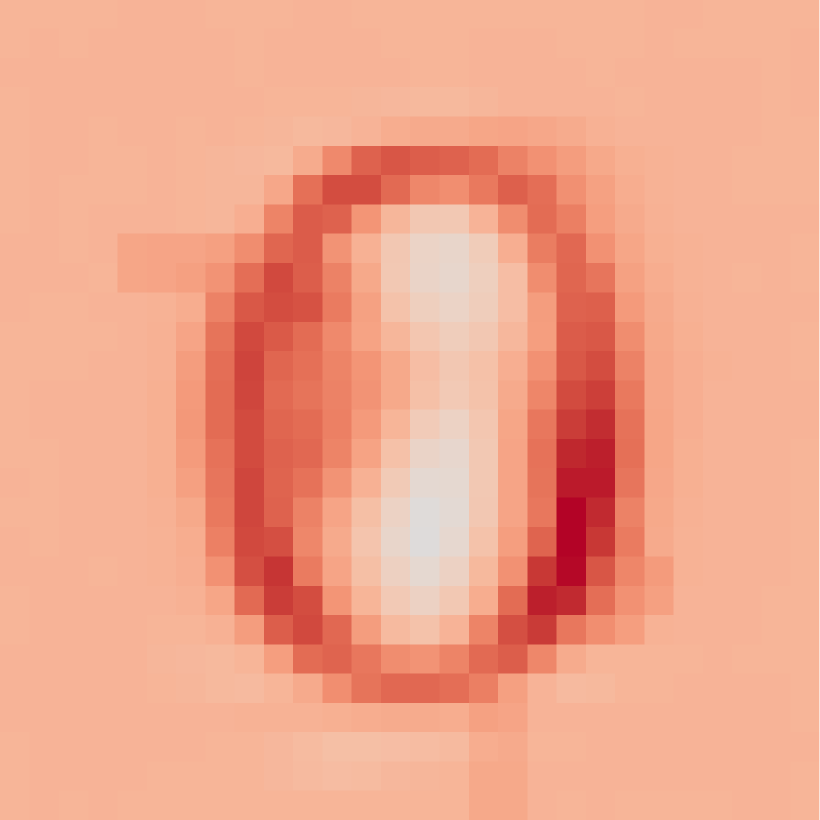} \\[3pt]

        \includegraphics[width=0.115\textwidth]{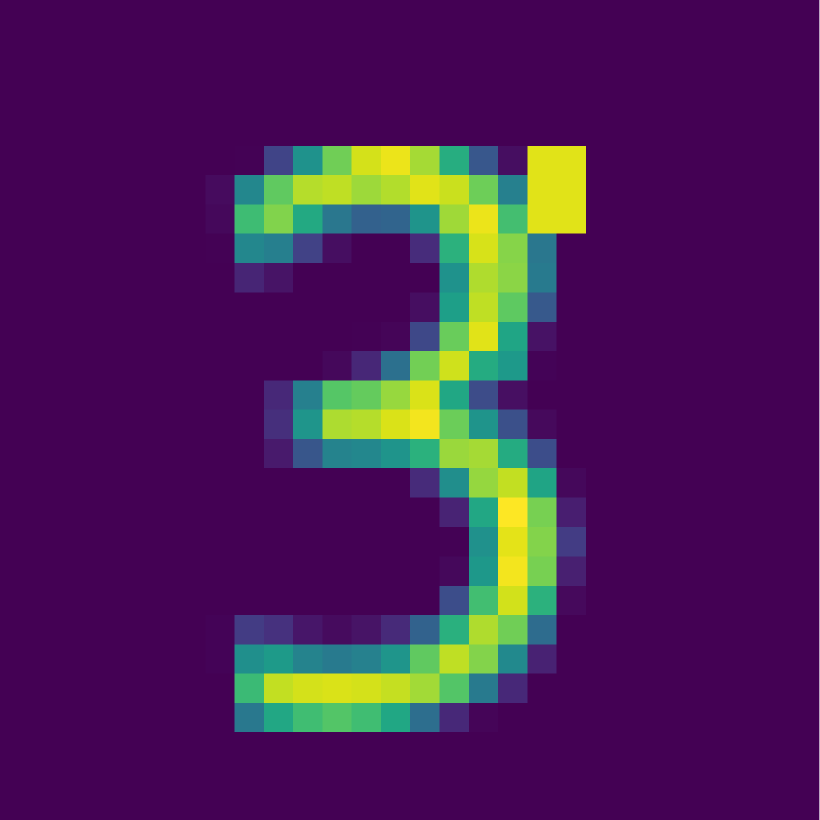} &
        \includegraphics[width=0.115\textwidth]{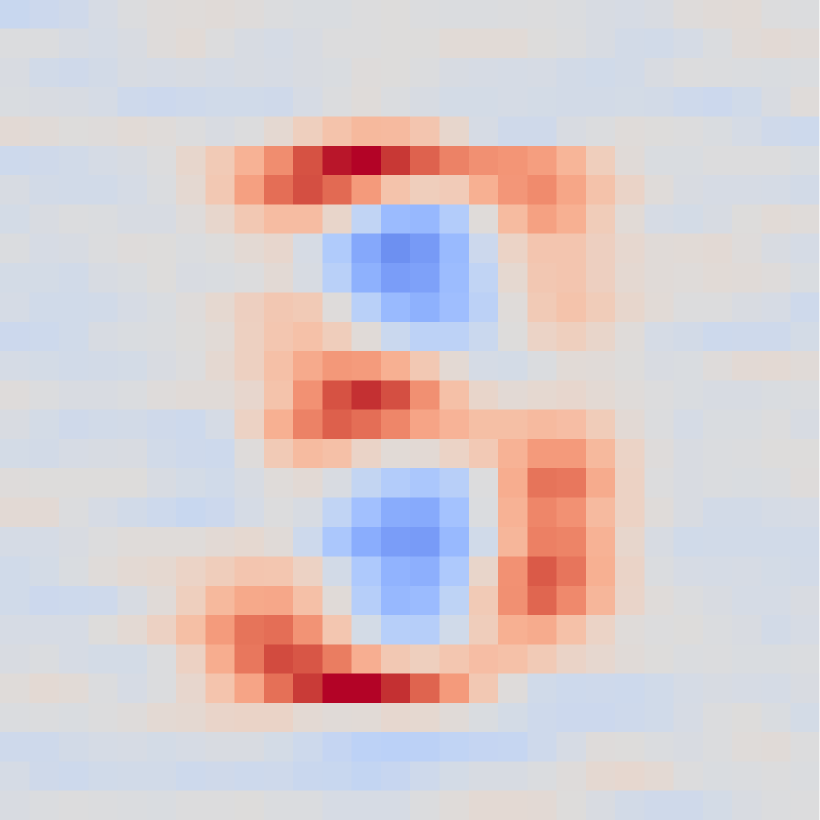} &
        \includegraphics[width=0.115\textwidth]{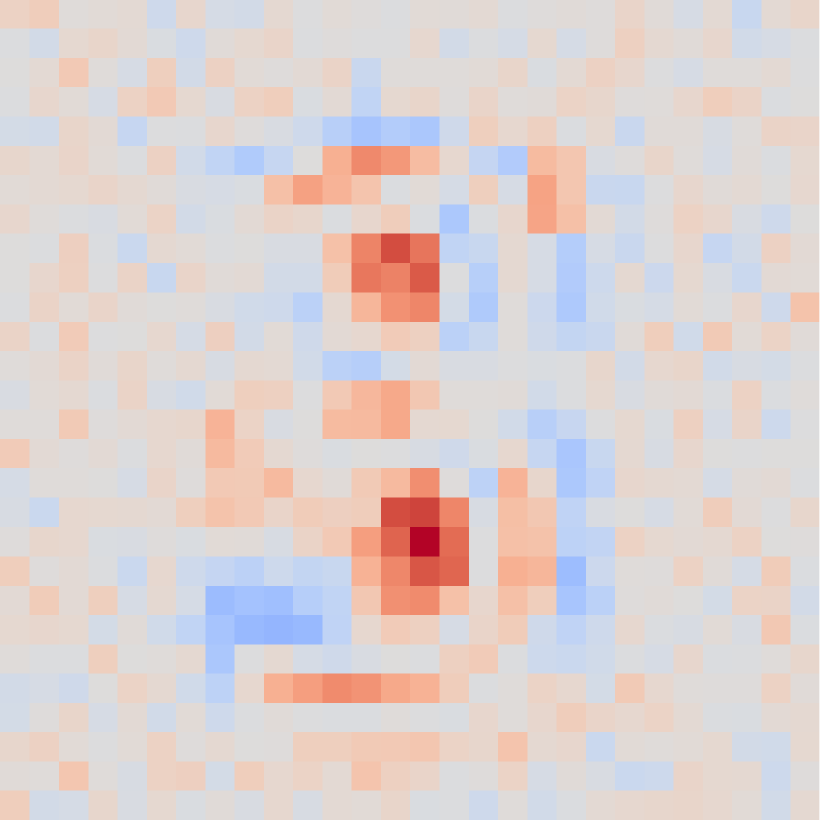} &
        \includegraphics[width=0.115\textwidth]{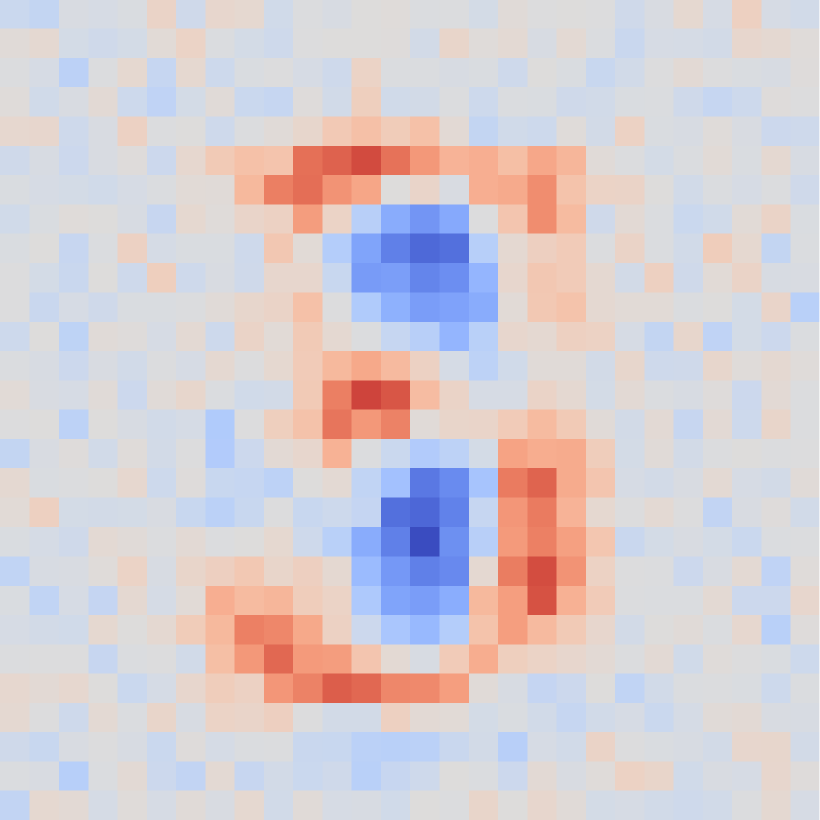} &
        \includegraphics[width=0.115\textwidth]{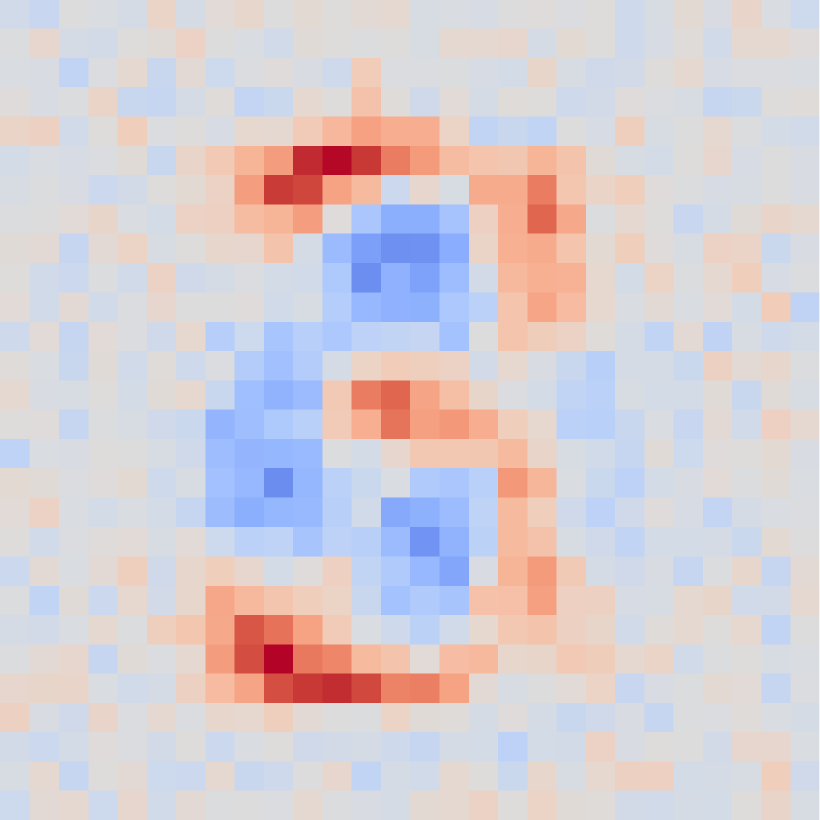} &
        \includegraphics[width=0.115\textwidth]{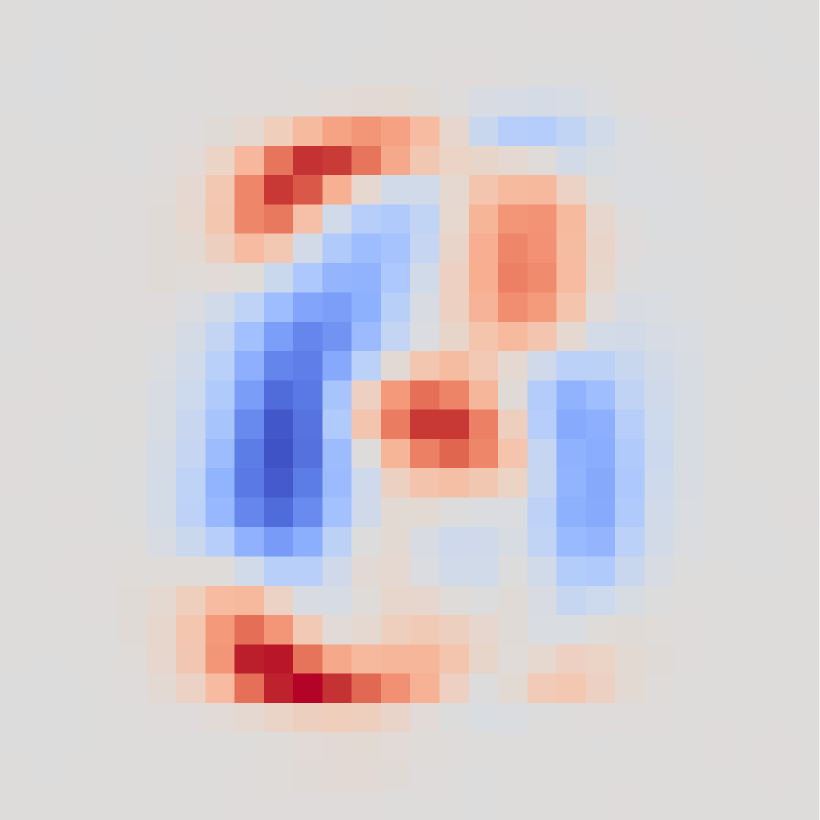} &
        \includegraphics[width=0.115\textwidth]{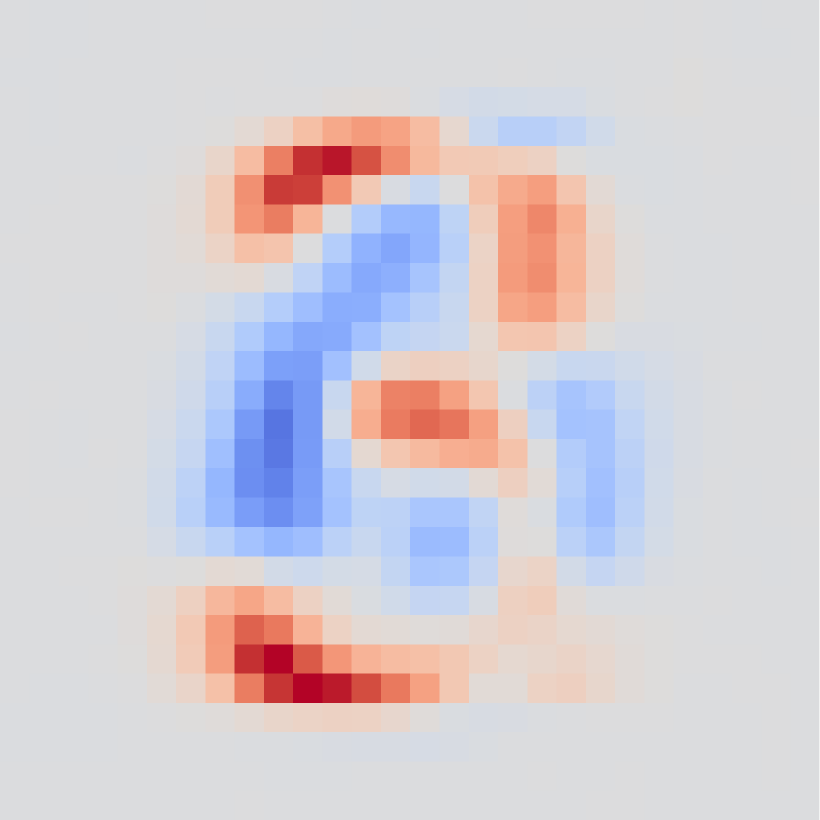} &
        \includegraphics[width=0.115\textwidth]{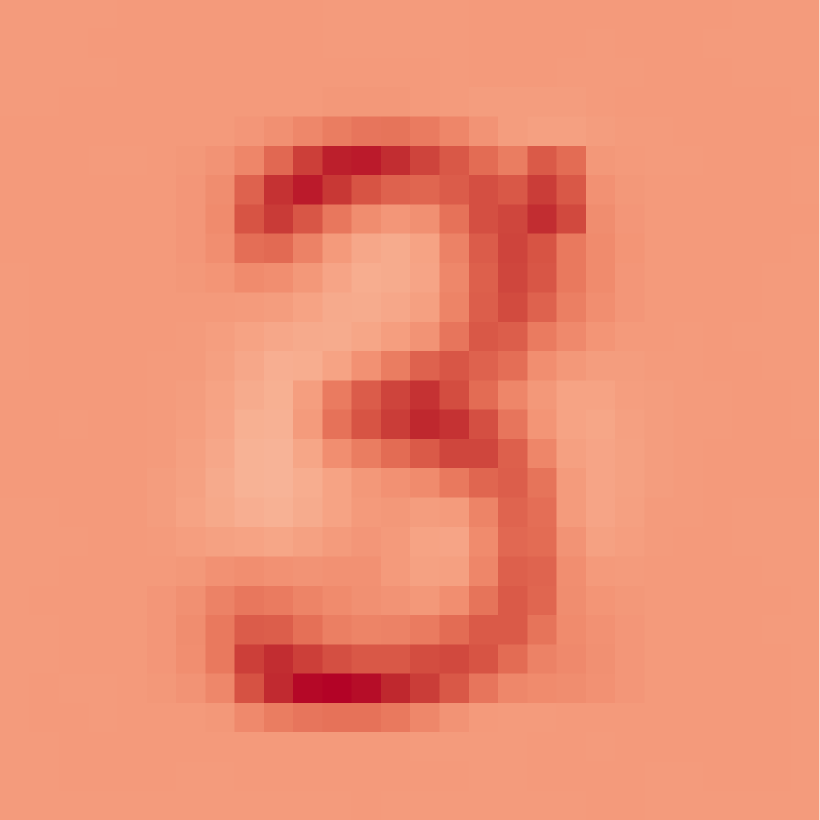} \\[3pt]

        \includegraphics[width=0.115\textwidth]{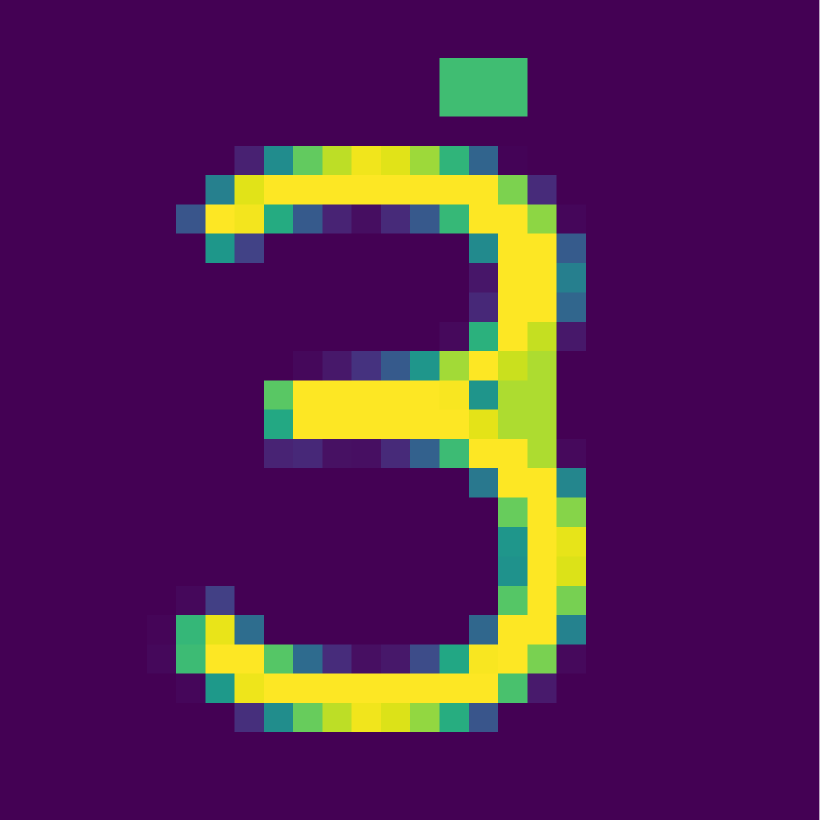} &
        \includegraphics[width=0.115\textwidth]{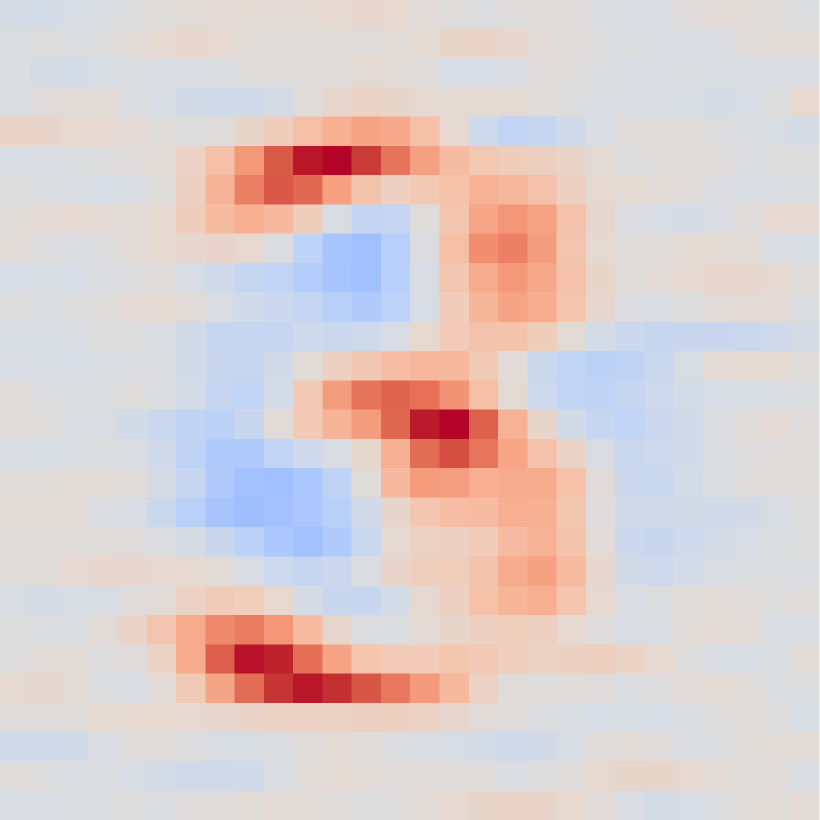} &
        \includegraphics[width=0.115\textwidth]{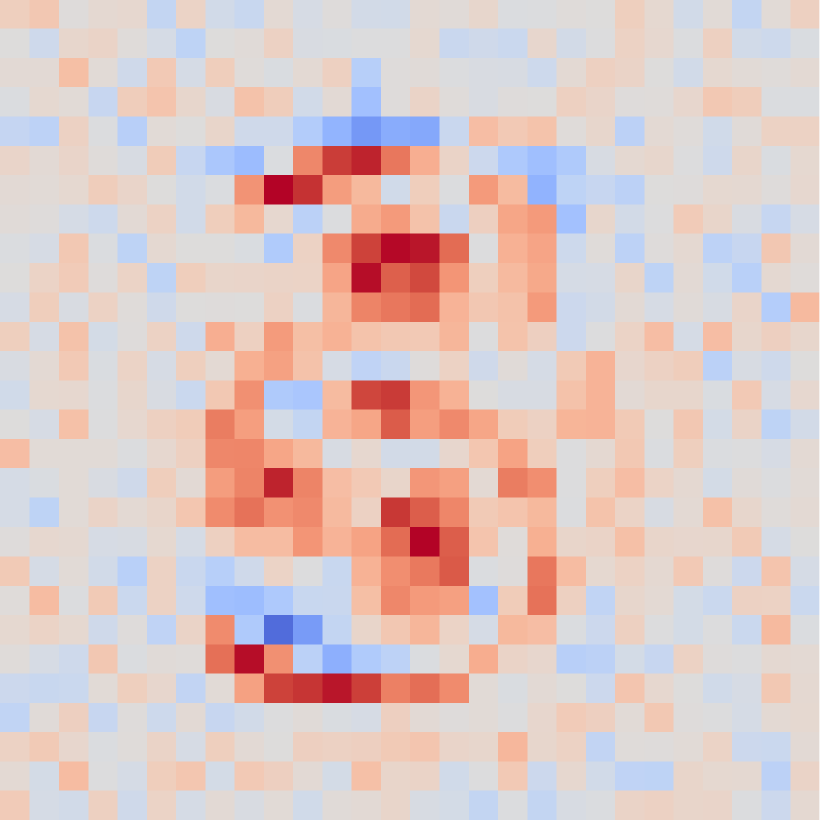} &
        \includegraphics[width=0.115\textwidth]{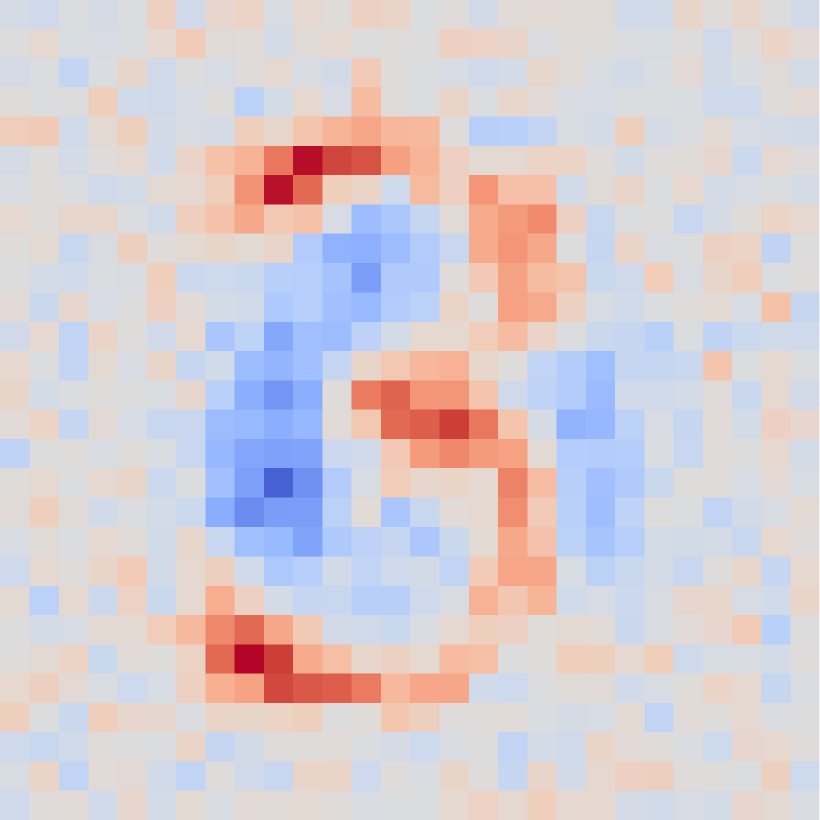} &
        \includegraphics[width=0.115\textwidth]{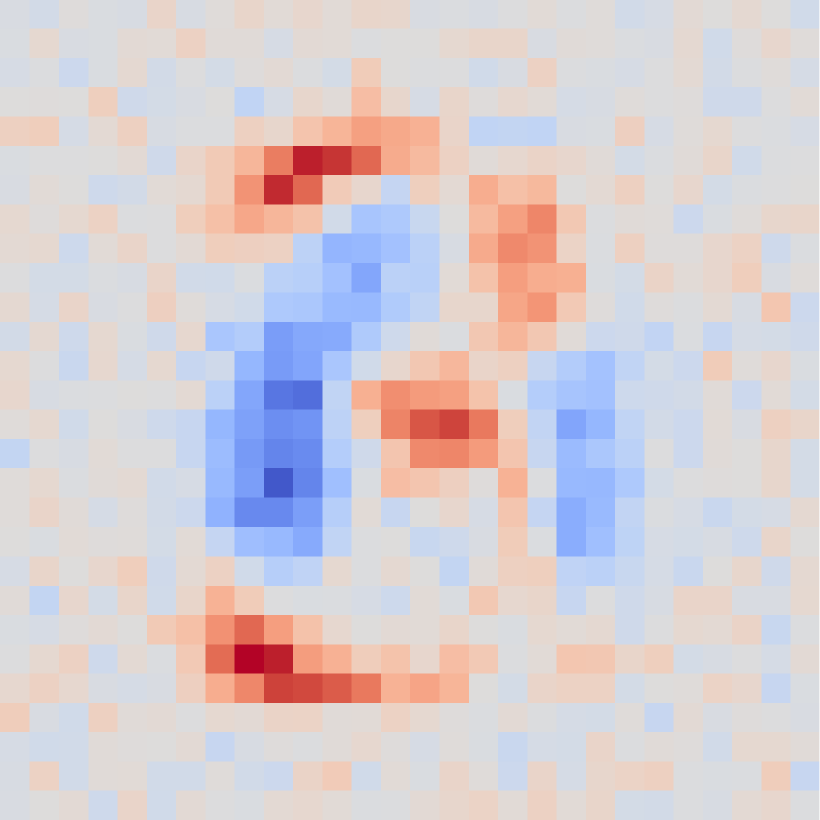} &
        \includegraphics[width=0.115\textwidth]{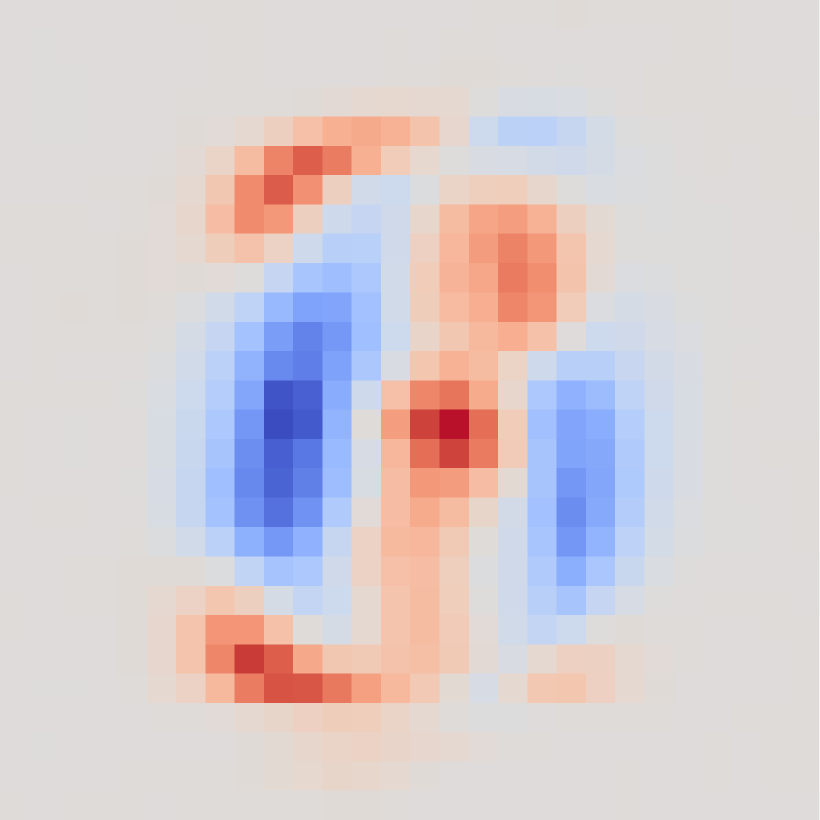} &
        \includegraphics[width=0.115\textwidth]{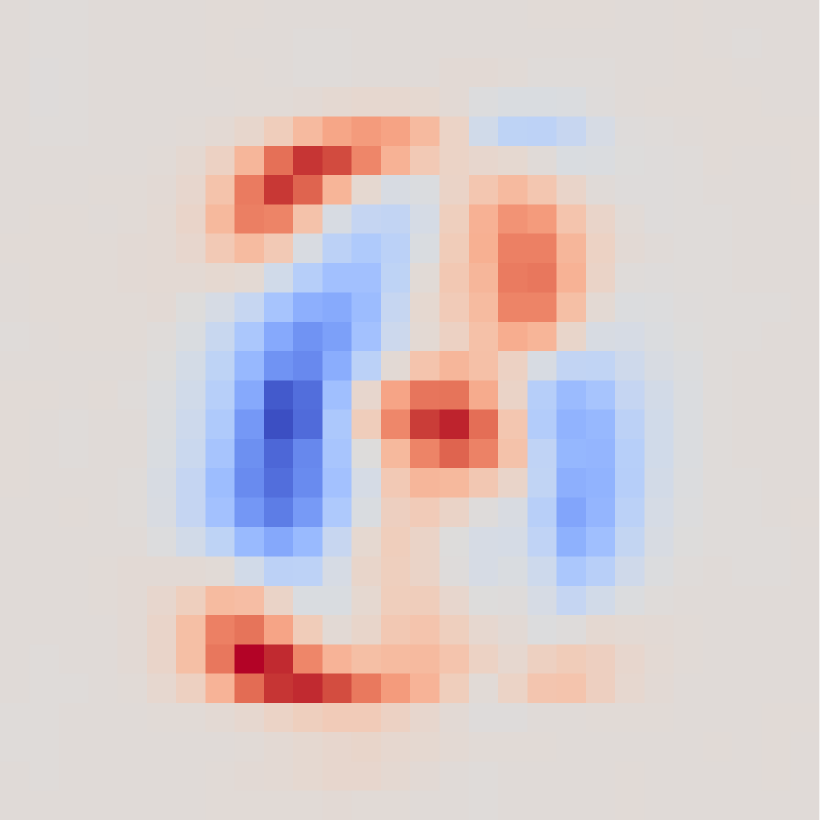} &
        \includegraphics[width=0.115\textwidth]{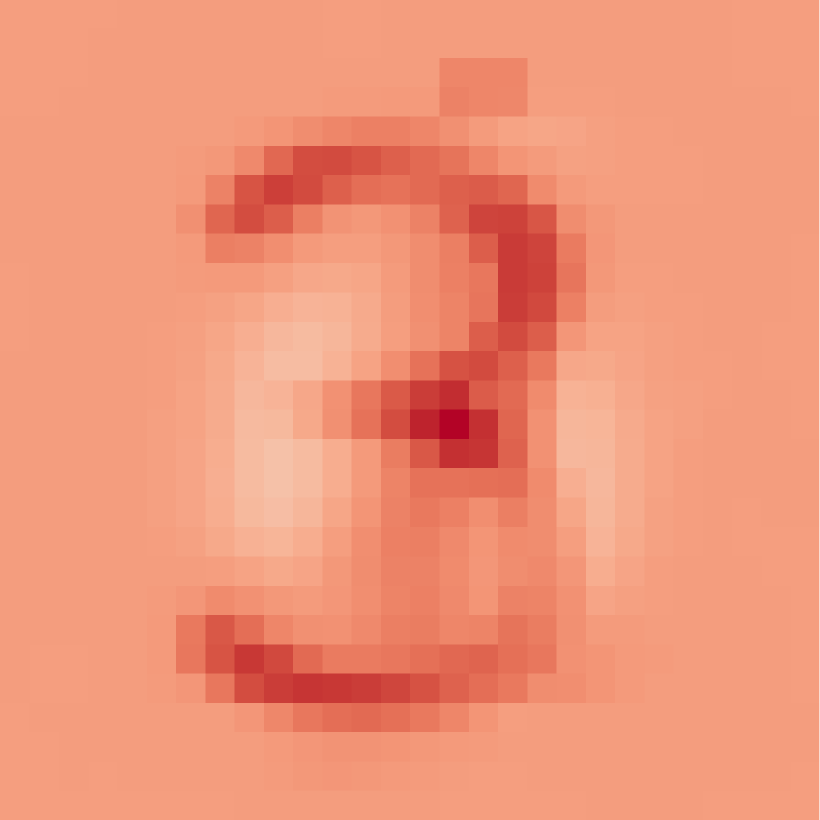} \\[3pt]

        \includegraphics[width=0.115\textwidth]{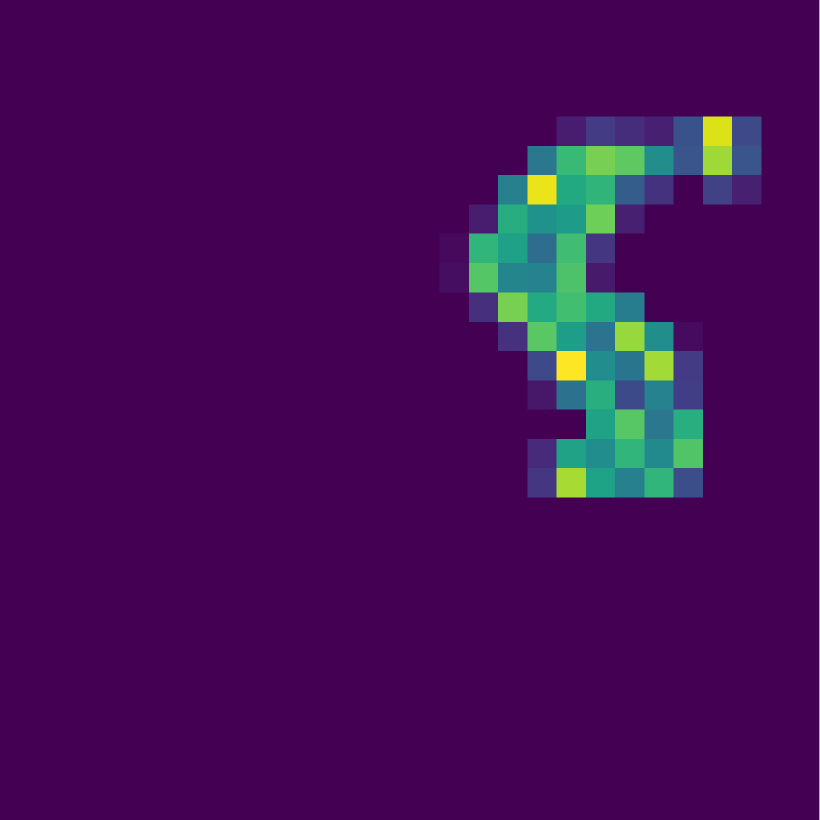} &
        \includegraphics[width=0.115\textwidth]{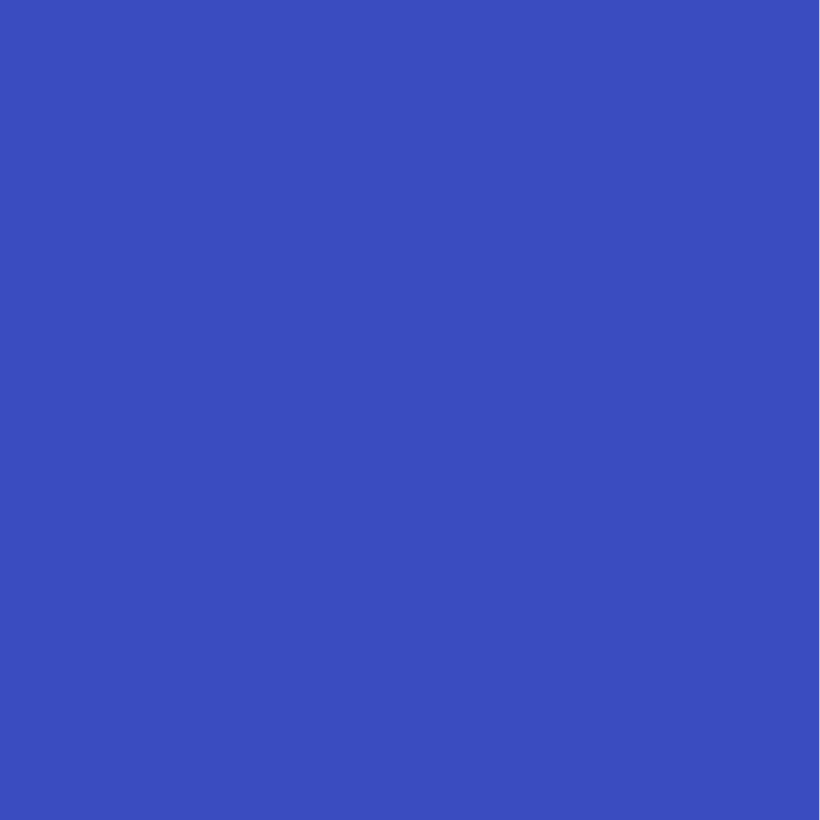} &
        \includegraphics[width=0.115\textwidth]{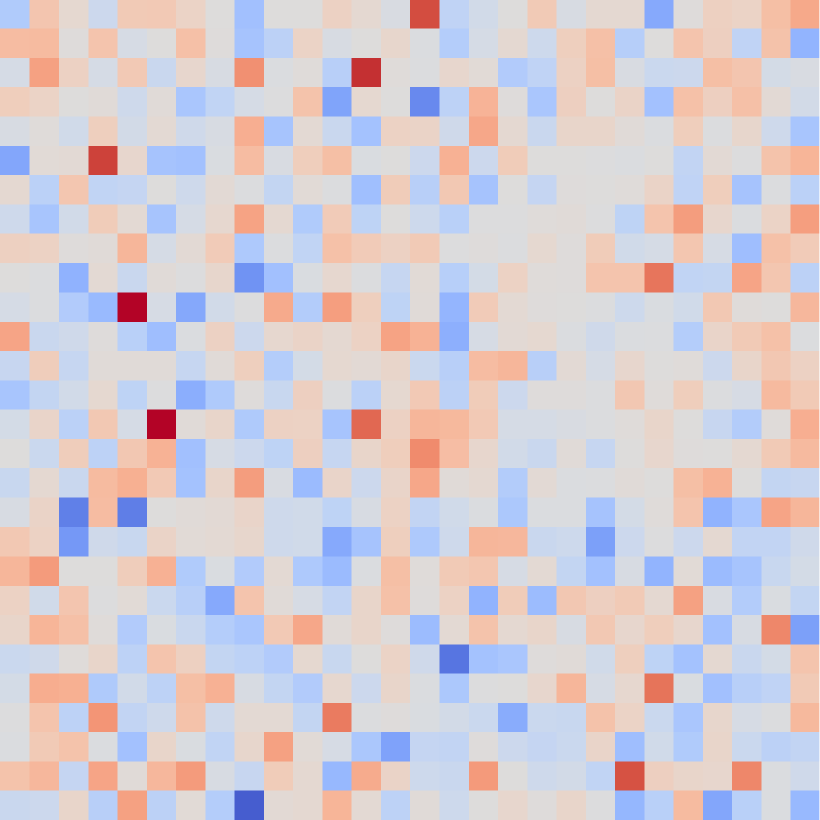} &
        \includegraphics[width=0.115\textwidth]{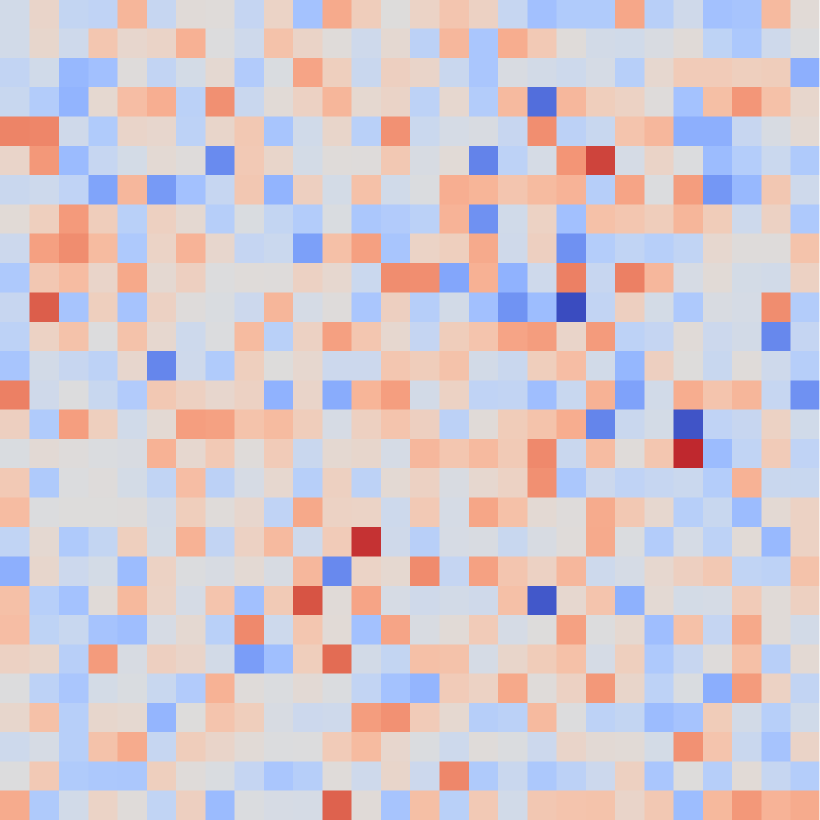} &
        \includegraphics[width=0.115\textwidth]{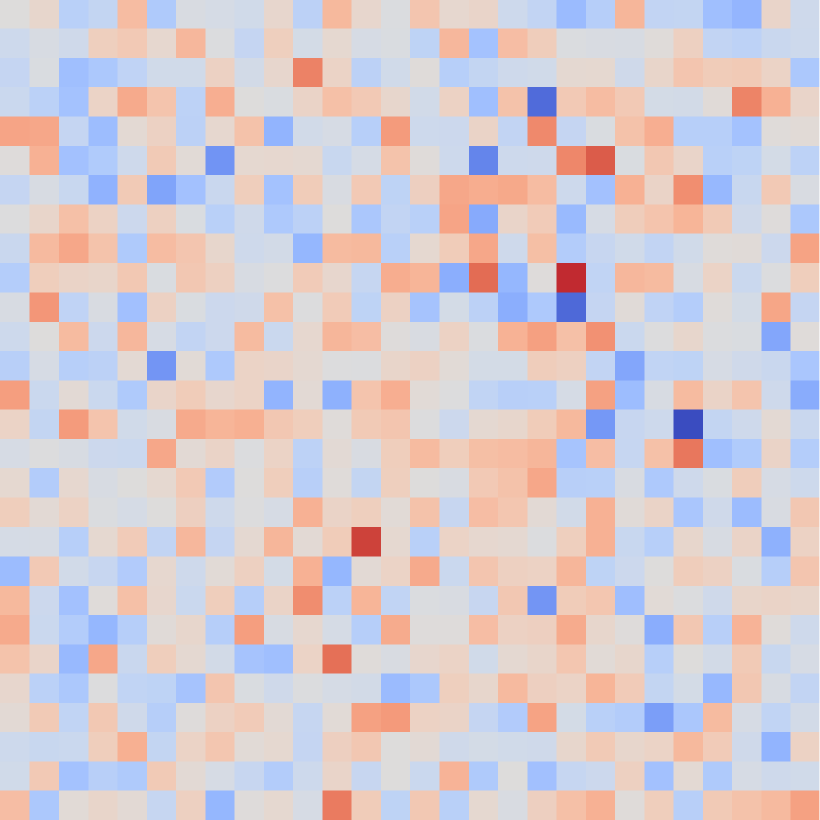} &
        \includegraphics[width=0.115\textwidth]{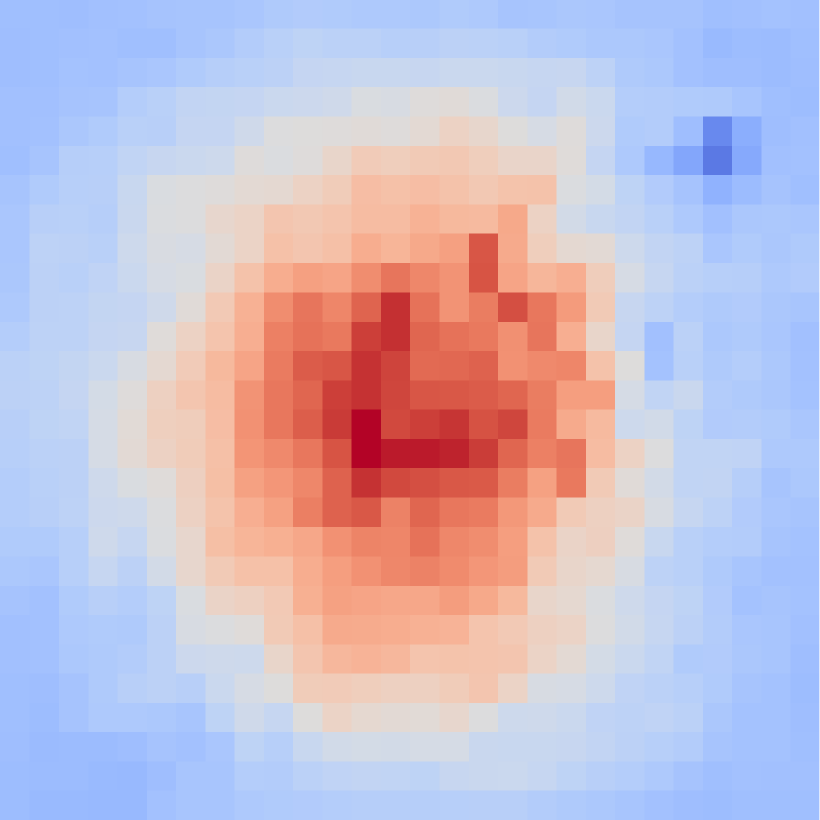} &
        \includegraphics[width=0.115\textwidth]{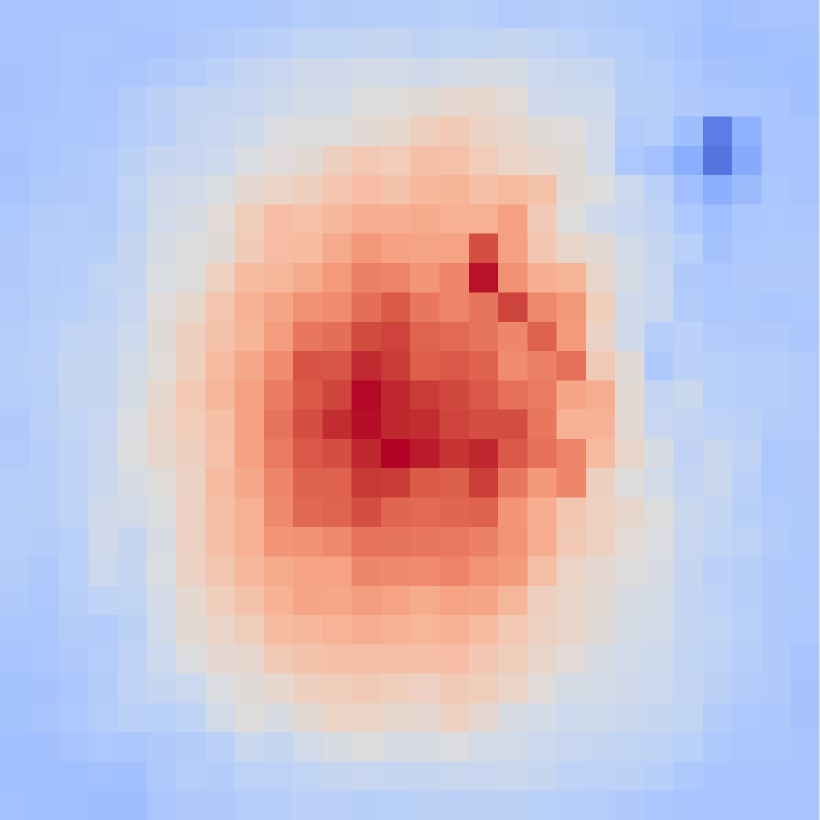} &
        \includegraphics[width=0.115\textwidth]{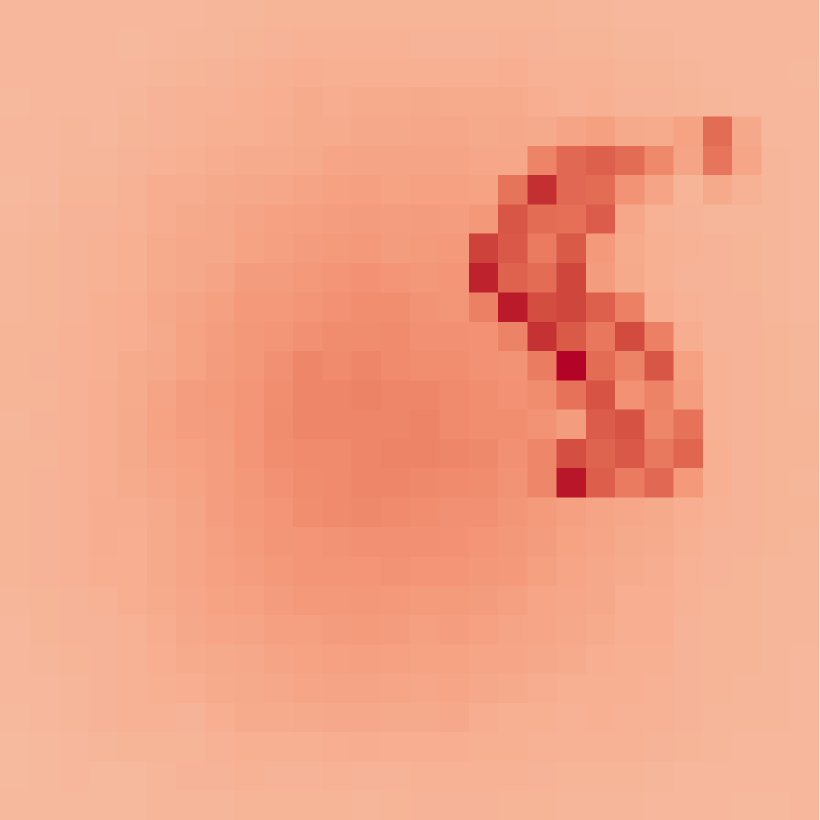} \\[3pt]

        \includegraphics[width=0.115\textwidth]{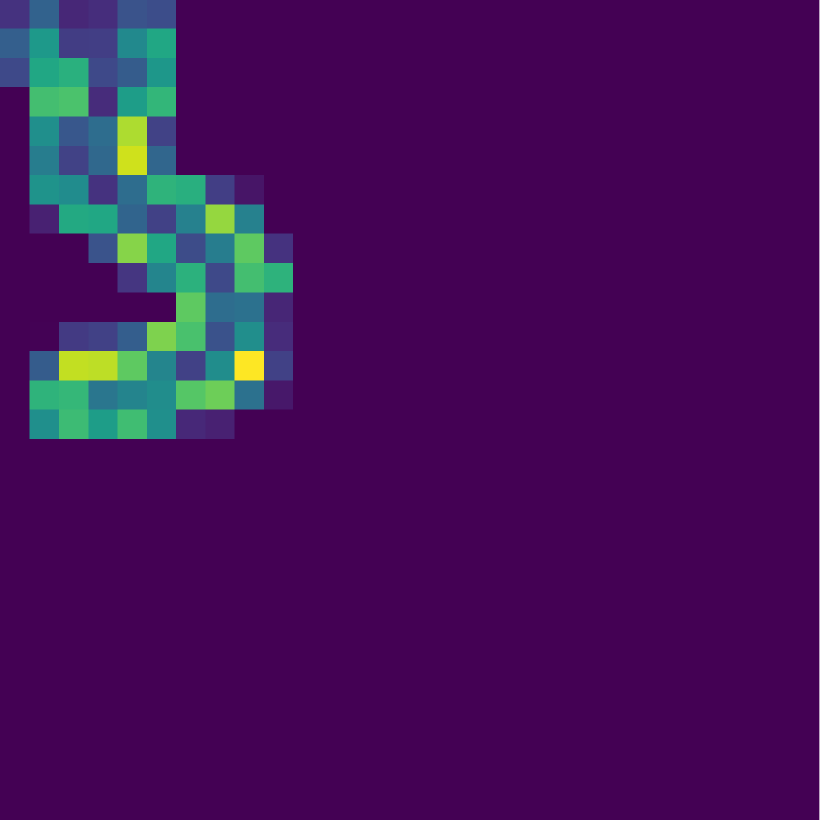} &
        \includegraphics[width=0.115\textwidth]{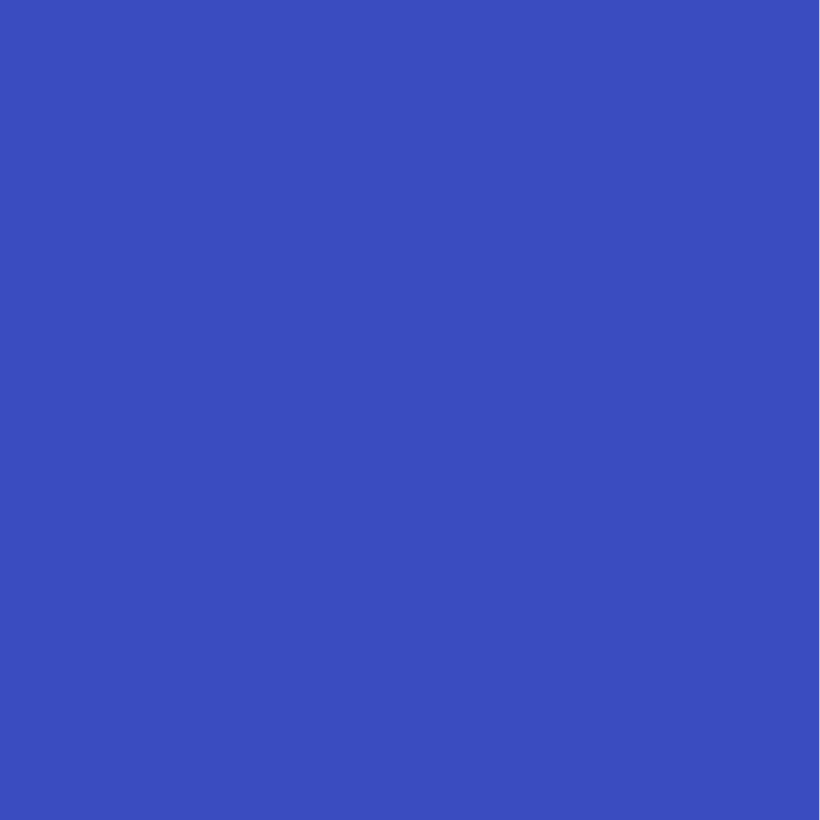} &
        \includegraphics[width=0.115\textwidth]{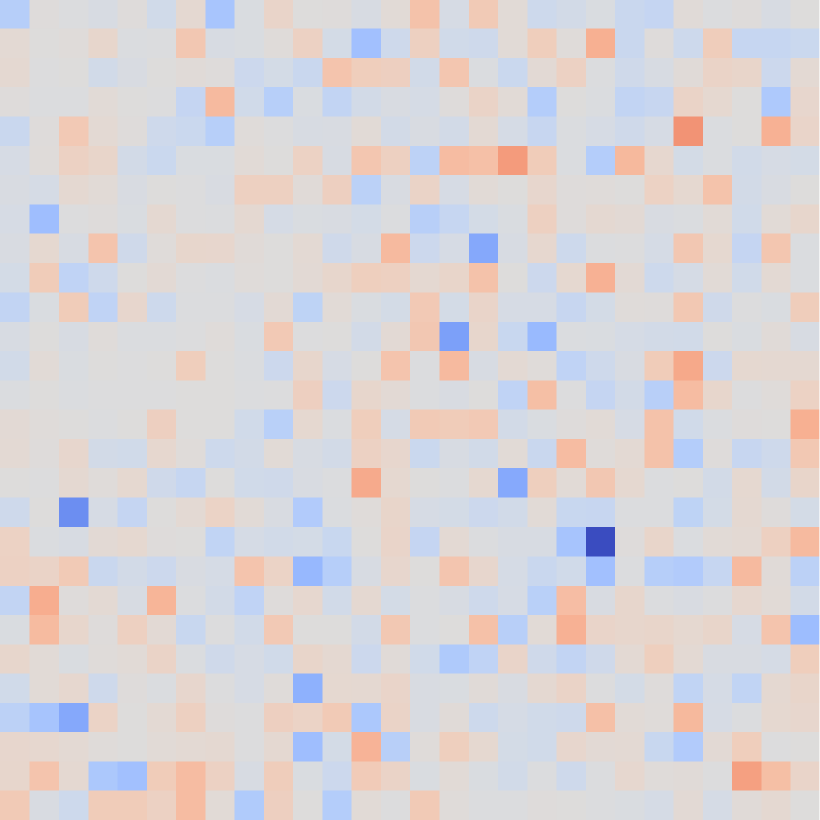} &
        \includegraphics[width=0.115\textwidth]{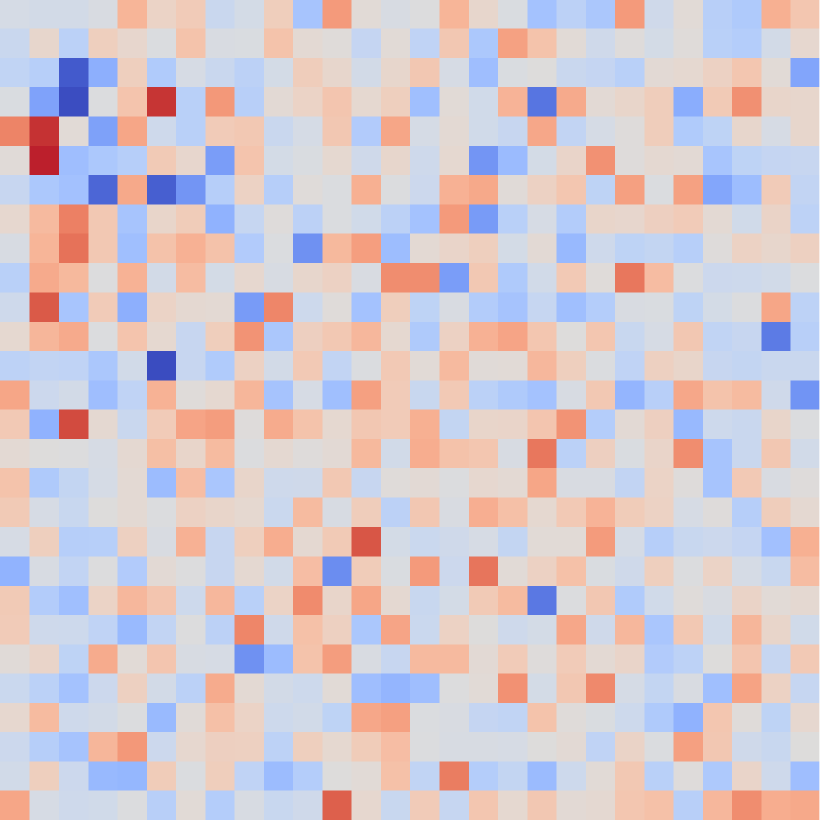} &
        \includegraphics[width=0.115\textwidth]{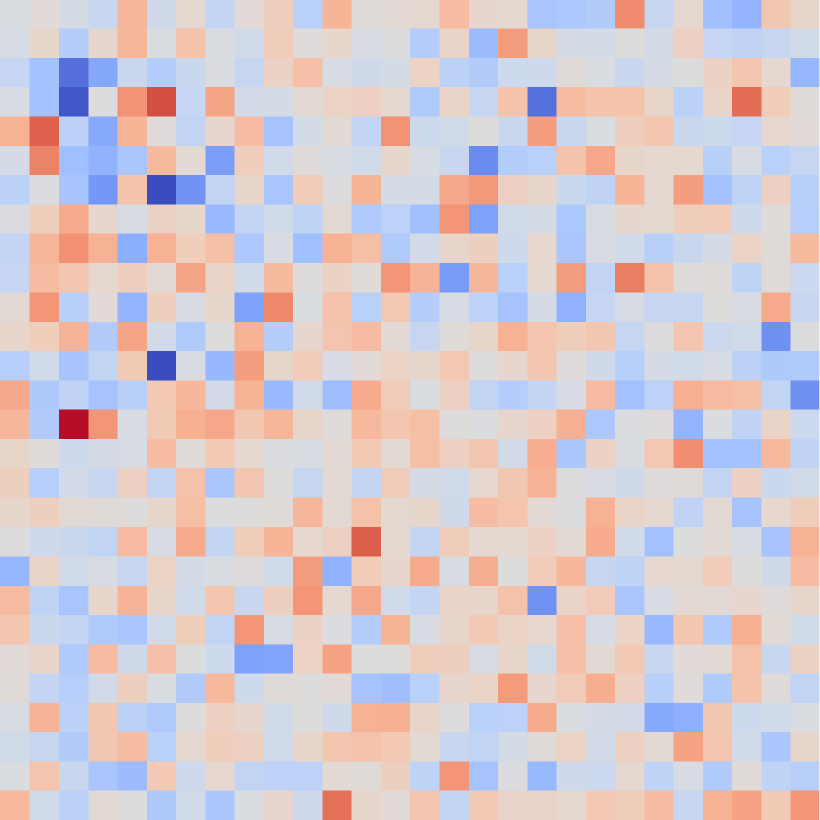} &
        \includegraphics[width=0.115\textwidth]{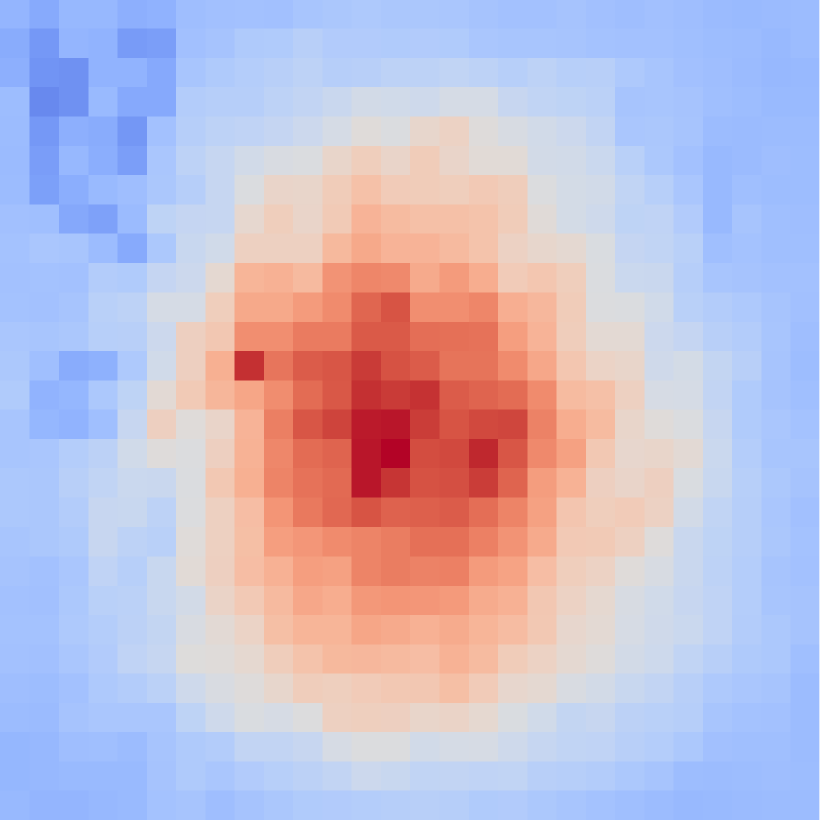} &
        \includegraphics[width=0.115\textwidth]{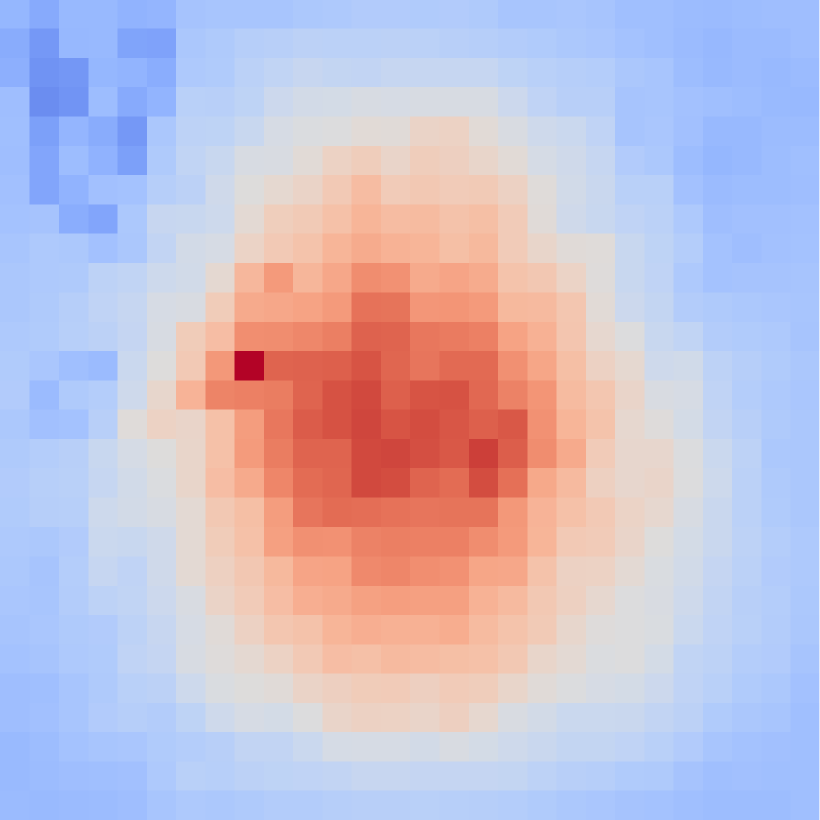} &
        \includegraphics[width=0.115\textwidth]{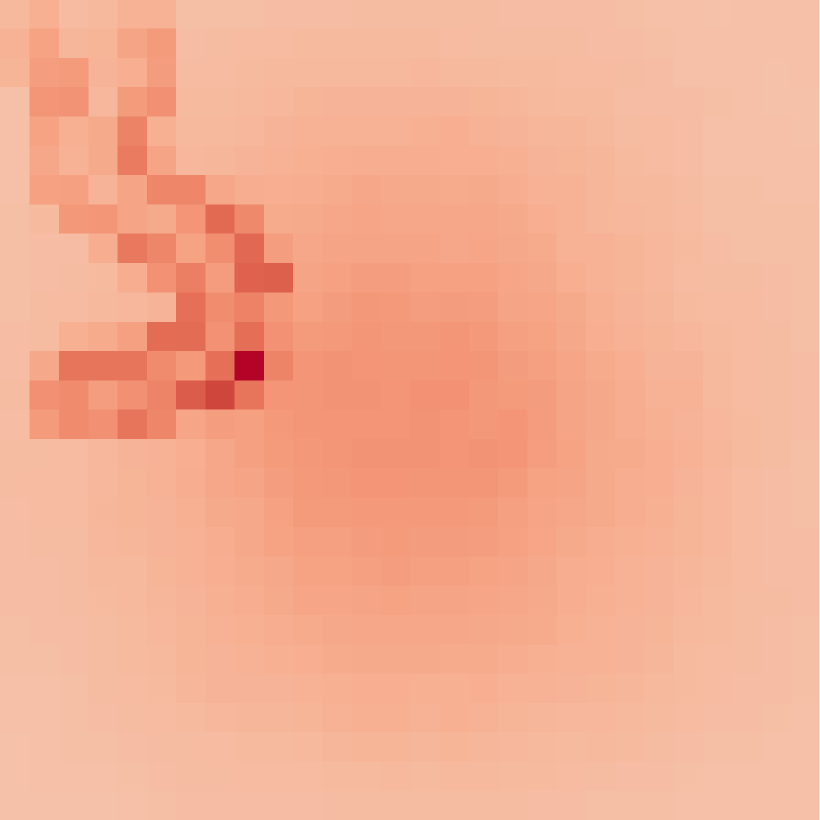} \\[3pt]

        \includegraphics[width=0.115\textwidth]{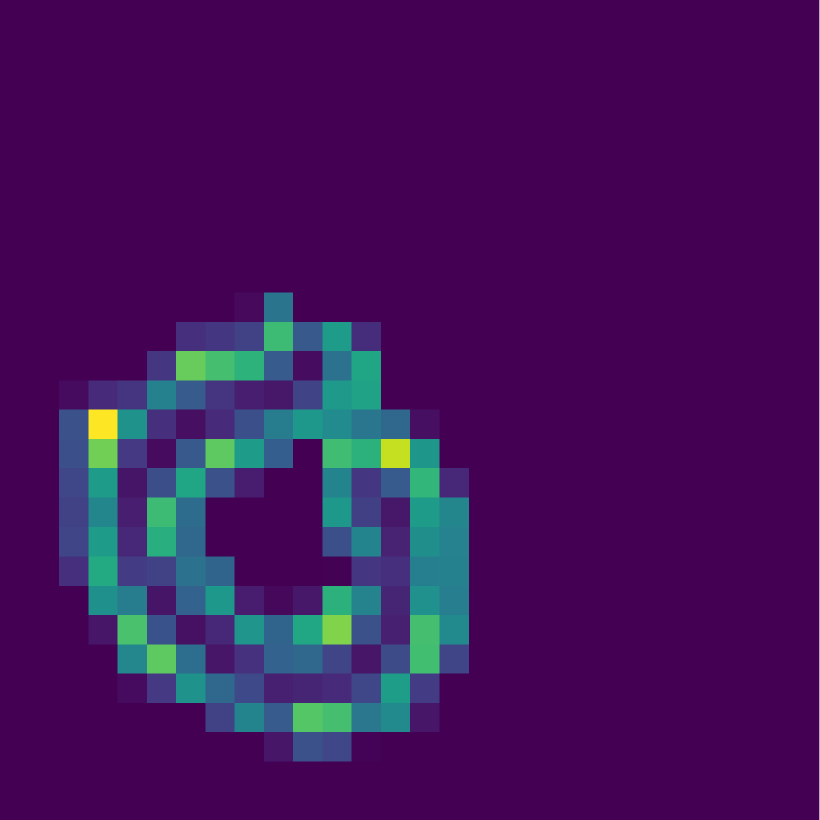} &
        \includegraphics[width=0.115\textwidth]{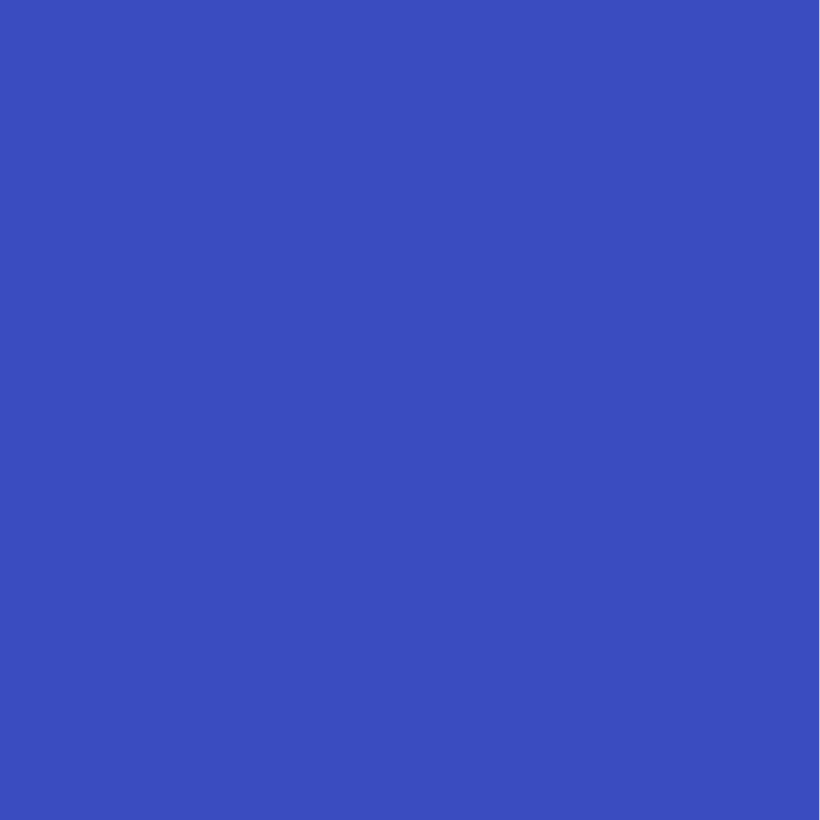} &
        \includegraphics[width=0.115\textwidth]{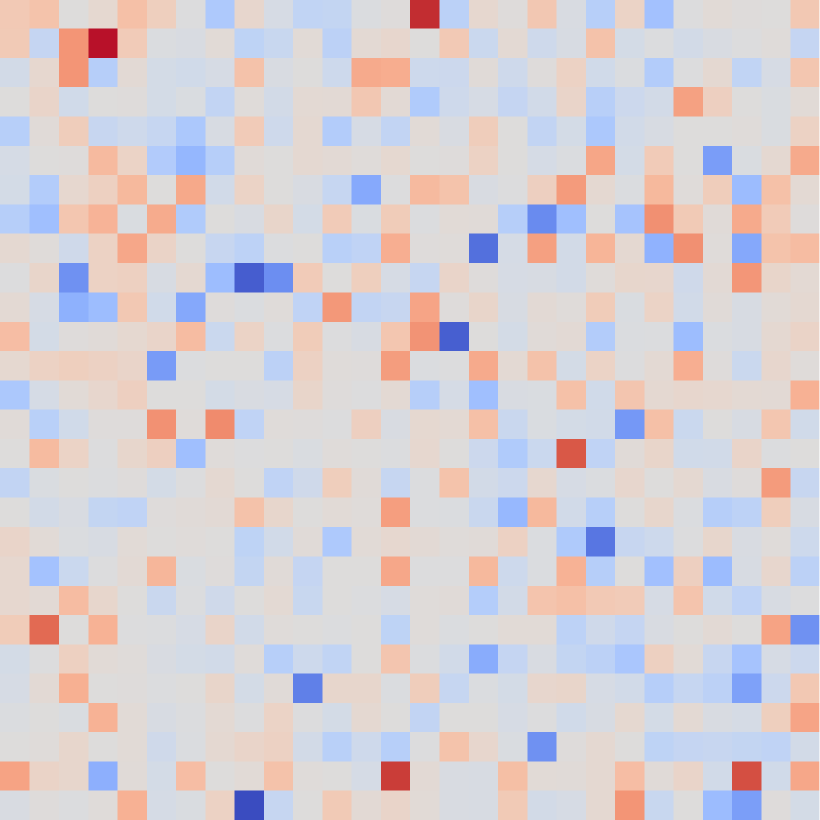} &
        \includegraphics[width=0.115\textwidth]{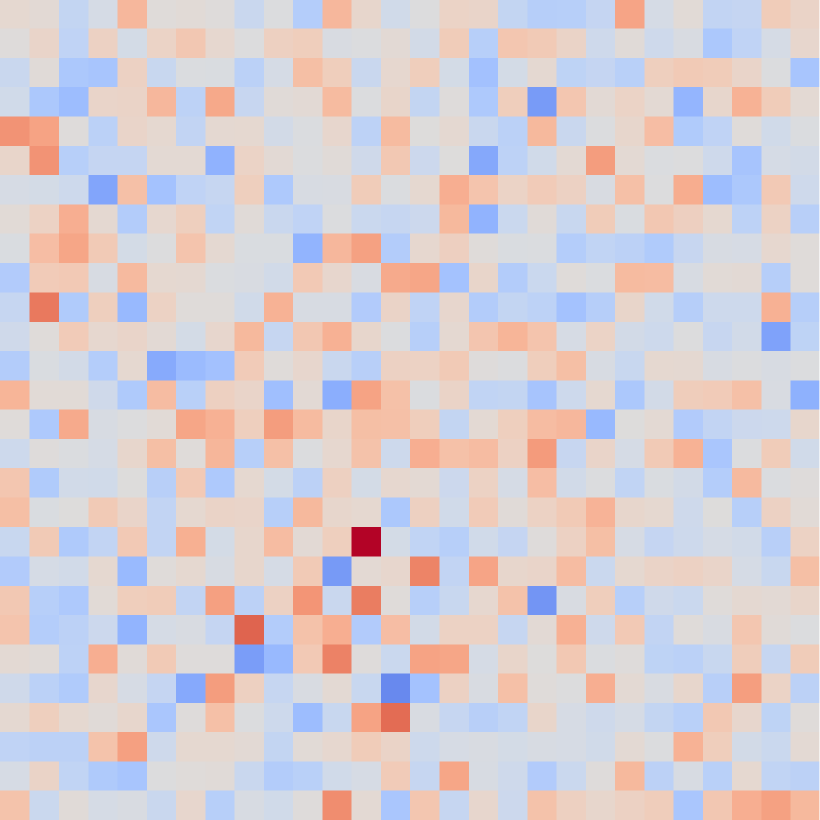} &
        \includegraphics[width=0.115\textwidth]{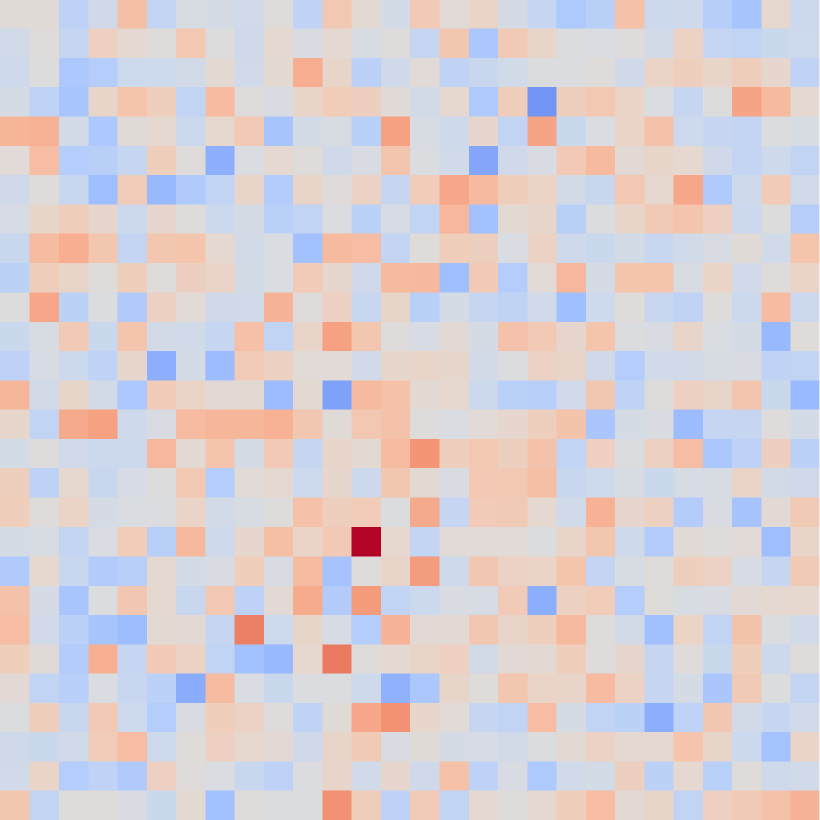} &
        \includegraphics[width=0.115\textwidth]{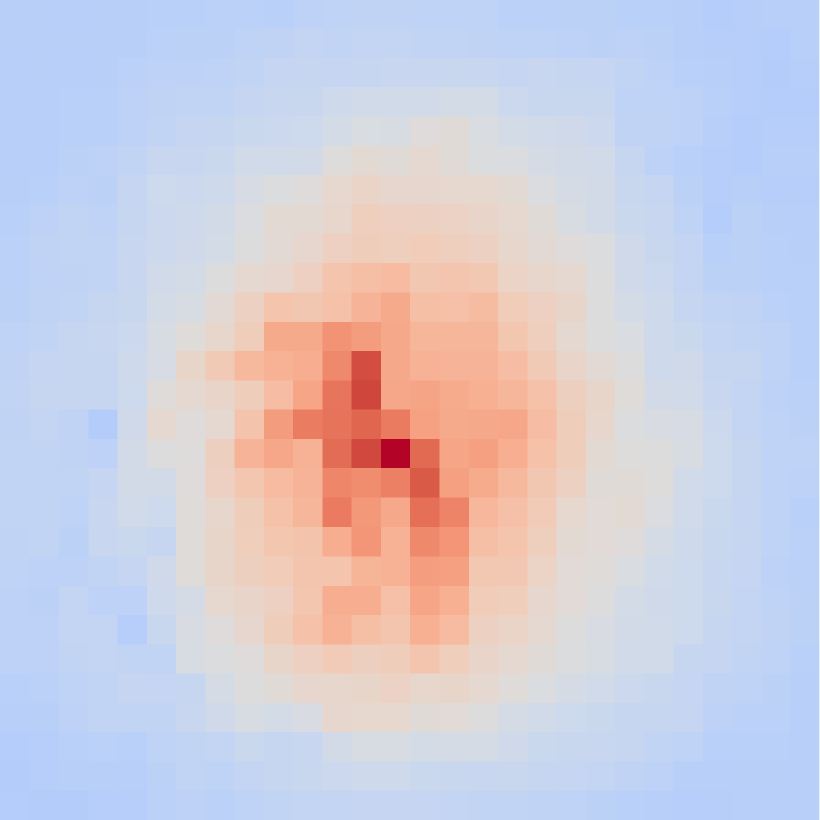} &
        \includegraphics[width=0.115\textwidth]{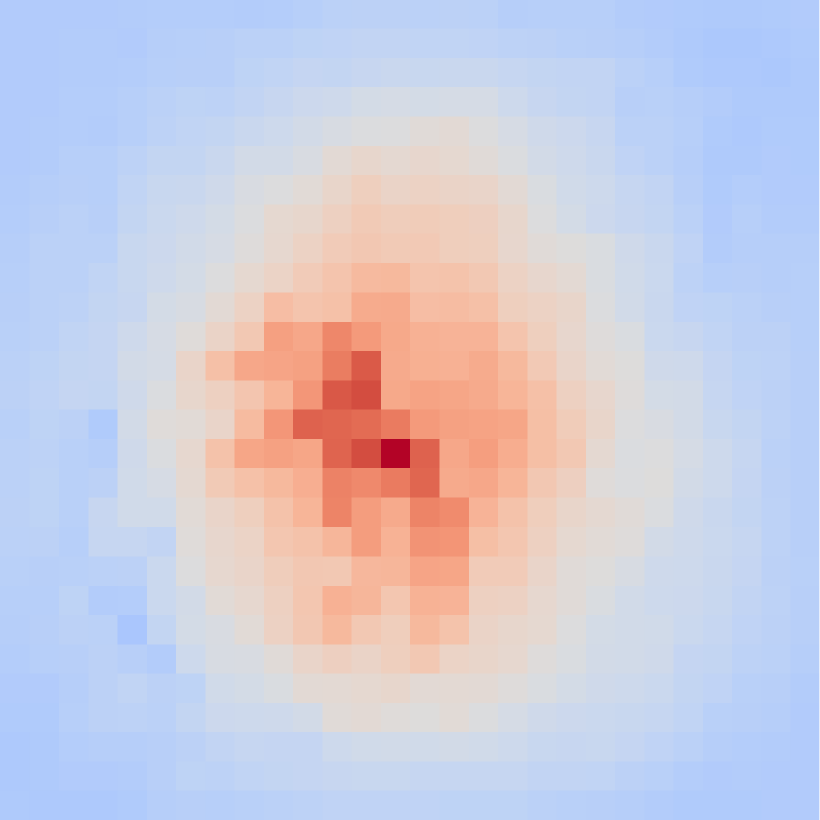} &
        \includegraphics[width=0.115\textwidth]{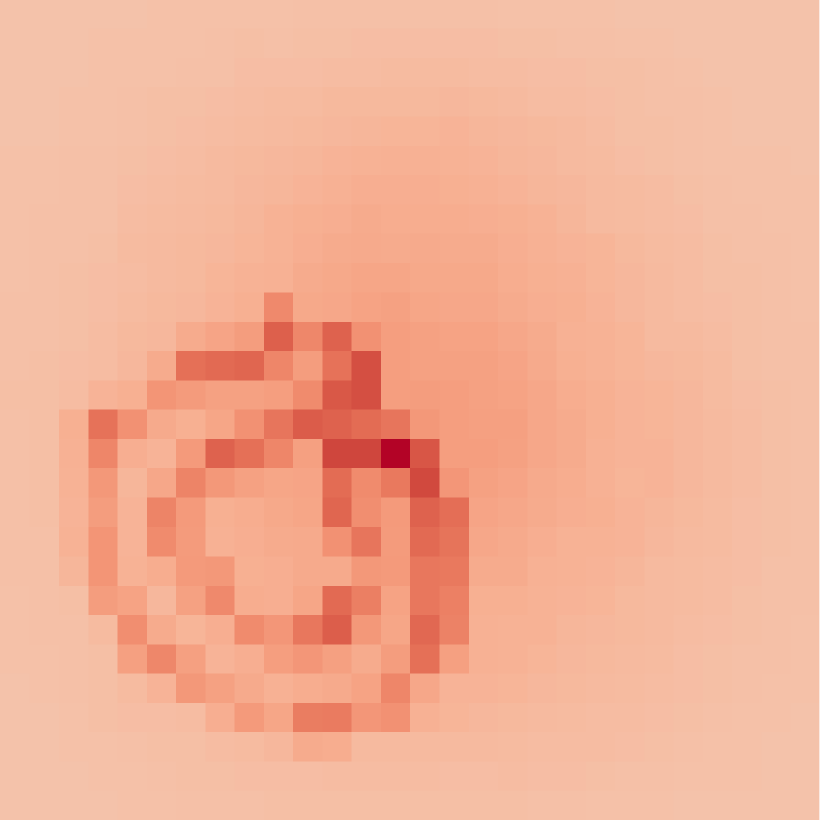} \\[3pt]
       
    \end{tabular}

    \caption{\emph{Attribution maps produced by different methods on TMNIST, TMNIST\_L, and AMNIST.} $XtrAIn$ (XT), $Xstep$ (X\_ST) and $XtrAIn^+$ (XT+) produce visually cleaner explanations in these examples, featuring a suppressed background and smooth foreground patterns.}
    \label{fig:both_res}

    \Description{A grid of 7x8 importance scores of different attribution methods for the Typeface MNIST, Fashion MNIST, Corrupted Mnist and Affine MNIST datasets. Starting from the leftmost column the image sample is presented, followed by different attribution methods: Deep Lift, Shapley Values, RISE, Integrated Gradients, $Xstep$, $XtrAIn$ and $XtrAIn^+$. $Xstep$ produces a sufficient approximation of $XtrAIn$, the heatmaps of which are more accurate, while $XtrAIn^+$ almost identifies the input signal. They jointly produce the cleanest heatmaps with respect to other techniques.}
\end{figure*}

\section{Discussion}
\label{sec:discussion}

A visual inspection of the attribution maps in \cref{fig:both_res} 
shows that $XtrAIn$ produces clean and interpretable explanations. Patterns within the central area correspond to the inputs at hand, while the predominantly non-activated background indicates the static, non-contributing nature of the corresponding neurons.

Different occlusion-based methods produce broadly similar patterns. However, $XtrAIn$ robustly generalizes these patterns: while other techniques produce artifacts across image regions $XtrAIn$ silences the background and generates clean patterns with gradual fill in the central region. These observations support our claims about feature dependence in standard occlusion-based methods and feature independence for $XtrAIn$. 

Furthermore, a comparison between $XtrAIn$ and its variant $Xstep$ reveals strong similarities in the results, providing evidence for a sufficient approximation for relatively simple models and datasets. Additionally, the consistency across samples of each class suggests these methods capture dominant, \textit{class-specific patterns} — which prevail over the more intricate patterns discovered during training (\cref{fig:upd_evol_patterns}). In contrast, $XtrAIn^+$ reveals explanations driven by \textit{input-specific patterns} that better fit the image signals. 

We provide quantitative validation in \cref{tab:main_results}, where $XtrAIn$ and $Xstep$ outperform baseline occlusion-based methods according to our proposed metric. Encouraged by these results, we conducted a deeper analysis by applying our methods to Affine MNIST, a dataset without spatial dependencies. While input-based occlusion methods only detect traces of the input signal, our variants confirm the model’s clean signal identification. The results suggest the model cannot identify dependencies between patterns and spatial location—a known limitation of FCNNs.

Applying our framework to the PAM50 breast cancer dataset reveals similarities with other attribution methods. Most methods, including $XtrAIn$ identify common features for classification to some extent. Despite this, they all provide evidence for features positively rewarding the Basal Type Cancer class for two epochs of training. On the other hand, $XtrAIn^+$ is only triggered by positive target class updates, thus resulting in a blank feature map. This application demonstrates its capacity in fostering safety in  
critical applications and system diagnostics.

Overall, $XtrAIn$ and its variants provide a powerful diagnostic lens for model explainability, revealing clean, interpretable patterns while uncovering intricate learning strategies within the model.

\section{Conclusion}
\label{sec:discussion}

In this work, we presented a novel methodology for estimating attribution scores, operating by chaining evolution, occlusion through weight 
perturbation and importance. We introduced $XtrAIn$ and its light variant $Xstep$, and demonstrated their complementary effectiveness in explaining 
predictions and revealing previously unknown insights into the model's capabilities and classification strategies, offering deeper insight into its decision-making
process.

This methodology establishes a new, unexplored path for transparent and granular model interpretation. By extending the usage of update mechanisms to more complex model architectures or functional aspects, its explanatory potential grows substantially. 

\clearpage
\bibliographystyle{ACM-Reference-Format}
\bibliography{sample-base}

@String{Computing = "Computing" }

@String{Computer = "{IEEE} Computer" }

@String{Academic = "Academic Press" }

@String{Springer = "Springer-Verlag" }

@article{10.1613/jair.1.12228,
      author = {Burkart, Nadia and Huber, Marco F.},
      title = {A Survey on the Explainability of Supervised Machine Learning},
      year = {2021},
      issue_date = {May 2021},
      publisher = {AI Access Foundation},
      address = {El Segundo, CA, USA},
      volume = {70},
      issn = {1076-9757},
      url = {https://doi.org/10.1613/jair.1.12228},
      doi = {10.1613/jair.1.12228},
      journal = {J. Artif. Int. Res.},
      month = may,
      pages = {245–317},
      numpages = {73}
}

@misc{vilone2020explainableartificialintelligencesystematic,
      title={Explainable Artificial Intelligence: a Systematic Review}, 
      author={Giulia Vilone and Luca Longo},
      year={2020},
      eprint={2006.00093},
      archivePrefix={arXiv},
      primaryClass={cs.AI},
      url={https://arxiv.org/abs/2006.00093}, 
}

@ARTICLE{9369420,
      author={Samek, Wojciech and Montavon, Grégoire and Lapuschkin, Sebastian and Anders, Christopher J. and Müller, Klaus-Robert},
      journal={Proceedings of the IEEE}, 
      title={Explaining Deep Neural Networks and Beyond: A Review of Methods and Applications}, 
      year={2021},
      volume={109},
      number={3},
      pages={247-278},
      doi={10.1109/JPROC.2021.3060483}
}

@article{Mumuni2025ExplainableAI,
      title={Explainable artificial intelligence (XAI): from inherent explainability to large language models},
      author={Fuseini Mumuni and Alhassan G. Mumuni},
      journal={ArXiv},
      year={2025},
      volume={abs/2501.09967},
      url={https://api.semanticscholar.org/CorpusID:275606857}
}

@misc{covert2023learningestimateshapleyvalues,
      title={Learning to Estimate Shapley Values with Vision Transformers}, 
      author={Ian Covert and Chanwoo Kim and Su-In Lee},
      year={2023},
      eprint={2206.05282},
      archivePrefix={arXiv},
      primaryClass={cs.CV},
      url={https://arxiv.org/abs/2206.05282}, 
}

@InProceedings{10.1007/978-3-319-10590-1_53,
      author="Zeiler, Matthew D.
      and Fergus, Rob",
      editor="Fleet, David
      and Pajdla, Tomas
      and Schiele, Bernt
      and Tuytelaars, Tinne",
      title="Visualizing and Understanding Convolutional Networks",
      booktitle="Computer Vision -- ECCV 2014",
      year="2014",
      publisher="Springer International Publishing",
      address="Cham",
      pages="818--833"
}

@InProceedings{10.5555/3295222.3295230,
      author = {Lundberg, Scott M. and Lee, Su-In},
      title = {A unified approach to interpreting model predictions},
      year = {2017},
      isbn = {9781510860964},
      publisher = {Curran Associates Inc.},
      address = {Red Hook, NY, USA},
      booktitle = {Proceedings of the 31st International Conference on Neural Information Processing Systems},
      pages = {4768–4777},
      numpages = {10},
      location = {Long Beach, California, USA},
      series = {NIPS'17}
}

@misc{petsiuk2018riserandomizedinputsampling,
      title={RISE: Randomized Input Sampling for Explanation of Black-box Models}, 
      author={Vitali Petsiuk and Abir Das and Kate Saenko},
      year={2018},
      eprint={1806.07421},
      archivePrefix={arXiv},
      primaryClass={cs.CV},
      url={https://arxiv.org/abs/1806.07421}, 
}

@INPROCEEDINGS{8237633,
      author={Fong, Ruth C. and Vedaldi, Andrea},
      booktitle={2017 IEEE International Conference on Computer Vision (ICCV)}, 
      publisher = {Association for Computing Machinery},
      title={Interpretable Explanations of Black Boxes by Meaningful Perturbation}, 
      year={2017},
      volume={},
      number={},
      pages={3449-3457},
      doi={10.1109/ICCV.2017.371}
}

@inproceedings{NEURIPS2021_4f5c422f,
      author = {Geiger, Atticus and Lu, Hanson and Icard, Thomas and Potts, Christopher},
      booktitle = {Advances in Neural Information Processing Systems},
      editor = {M. Ranzato and A. Beygelzimer and Y. Dauphin and P.S. Liang and J. Wortman Vaughan},
      pages = {9574--9586},
      publisher = {Curran Associates, Inc.},
      title = {Causal Abstractions of Neural Networks},
      url = {https://proceedings.neurips.cc/paper_files/paper/2021/file/4f5c422f4d49a5a807eda27434231040-Paper.pdf},
      volume = {34},
      year = {2021}
}

@InProceedings{nnsattr_causal,
      title = 	 {Neural Network Attributions: A Causal Perspective},
      author =       {Chattopadhyay, Aditya and Manupriya, Piyushi and Sarkar, Anirban and Balasubramanian, Vineeth N},
      booktitle = 	 {Proceedings of the 36th International Conference on Machine Learning},
      pages = 	 {981--990},
      year = 	 {2019},
      editor = 	 {Chaudhuri, Kamalika and Salakhutdinov, Ruslan},
      volume = 	 {97},
      series = 	 {Proceedings of Machine Learning Research},
      month = 	 {09--15 Jun},
      publisher =    {PMLR},
      url = 	 {https://proceedings.mlr.press/v97/chattopadhyay19a.html}
}

@InProceedings{10.1007/978-3-032-08333-3_8,
      author="Schmeisser, Fabian
      and Lucieri, Adriano
      and Dengel, Andreas
      and Ahmed, Sheraz",
      editor="Guidotti, Riccardo
      and Schmid, Ute
      and Longo, Luca",
      title="Spectral Occlusion - Attribution Beyond Spatial Relevance Heatmaps",
      booktitle="Explainable Artificial Intelligence",
      year="2026",
      publisher="Springer Nature Switzerland",
      address="Cham",
      pages="159--183"
}

@article{bluecher2024decoupling,
	title={Decoupling Pixel Flipping and Occlusion Strategy for Consistent {XAI} Benchmarks},
	author={Stefan Bluecher and Johanna Vielhaben and Nils Strodthoff},
	journal={Transactions on Machine Learning Research},
	issn={2835-8856},
	year={2024},
	url={https://openreview.net/forum?id=bIiLXdtUVM},
	note={}
}

@article{JMLR:v22:20-1316,
  author  = {Ian Covert and Scott Lundberg and Su-In Lee},
  title   = {Explaining by Removing: A Unified Framework for Model Explanation},
  journal = {Journal of Machine Learning Research},
  year    = {2021},
  volume  = {22},
  number  = {209},
  pages   = {1--90},
  url     = {http://jmlr.org/papers/v22/20-1316.html}
}

@Article{Gevaert2024,
    author={Gevaert, Arne
    and Rousseau, Axel-Jan
    and Becker, Thijs
    and Valkenborg, Dirk
    and De Bie, Tijl
    and Saeys, Yvan},
    title={Evaluating feature attribution methods in the image domain},
    journal={Machine Learning},
    year={2024},
    month={Sep},
    day={01},
    volume={113},
    number={9},
    pages={6019-6064},
    issn={1573-0565},
    doi={10.1007/s10994-024-06550-x},
    url={https://doi.org/10.1007/s10994-024-06550-x}
}

@InProceedings{pmlr-v162-rong22a,
  title = 	 {A Consistent and Efficient Evaluation Strategy for Attribution Methods},
  author =       {Rong, Yao and Leemann, Tobias and Borisov, Vadim and Kasneci, Gjergji and Kasneci, Enkelejda},
  booktitle = 	 {Proceedings of the 39th International Conference on Machine Learning},
  pages = 	 {18770--18795},
  year = 	 {2022},
  editor = 	 {Chaudhuri, Kamalika and Jegelka, Stefanie and Song, Le and Szepesvari, Csaba and Niu, Gang and Sabato, Sivan},
  volume = 	 {162},
  series = 	 {Proceedings of Machine Learning Research},
  month = 	 {17--23 Jul},
  publisher =    {PMLR},
  url = 	 {https://proceedings.mlr.press/v162/rong22a.html},
}

@article{lipton2017mythosmodelinterpretability,
      author = {Lipton, Zachary C.},
      title = {The mythos of model interpretability},
      year = {2018},
      issue_date = {October 2018},
      publisher = {Association for Computing Machinery},
      address = {New York, NY, USA},
      volume = {61},
      number = {10},
      issn = {0001-0782},
      url = {https://doi.org/10.1145/3233231},
      doi = {10.1145/3233231},
      journal = {Commun. ACM},
      month = sep,
      pages = {36-43},
      numpages = {8}
}

@article{Nauta_2023,
      title={From Anecdotal Evidence to Quantitative Evaluation Methods: A Systematic Review on Evaluating Explainable AI},
      volume={55},
      ISSN={1557-7341},
      url={http://dx.doi.org/10.1145/3583558},
      DOI={10.1145/3583558},
      number={13s},
      journal={ACM Computing Surveys},
      publisher={Association for Computing Machinery (ACM)},
      author={Nauta, Meike and Trienes, Jan and Pathak, Shreyasi and Nguyen, Elisa and Peters, Michelle and Schmitt, Yasmin and Schlötterer, Jörg and van Keulen, Maurice and Seifert, Christin},
      year={2023},
      month=jul, pages={1–42} 
}

@Article{sturmfels2020visualizing,
      author = {Sturmfels, Pascal and Lundberg, Scott and Lee, Su-In},
      title = {Visualizing the Impact of Feature Attribution Baselines},
      journal = {Distill},
      year = {2020},
      note = {https://distill.pub/2020/attribution-baselines},
      doi = {10.23915/distill.00022}
}

@misc{jain2022missingnessbiasmodeldebugging,
      title={Missingness Bias in Model Debugging}, 
      author={Saachi Jain and Hadi Salman and Eric Wong and Pengchuan Zhang and Vibhav Vineet and Sai Vemprala and Aleksander Madry},
      year={2022},
      eprint={2204.08945},
      archivePrefix={arXiv},
      primaryClass={cs.CV},
      url={https://arxiv.org/abs/2204.08945}, 
}

@article {Geoscience,
      author = "Antonios Mamalakis and Elizabeth A. Barnes and Imme Ebert-Uphoff",
      title = "Carefully Choose the Baseline: Lessons Learned from Applying XAI Attribution Methods for Regression Tasks in Geoscience",
      journal = "Artificial Intelligence for the Earth Systems",
      year = "2023",
      publisher = "American Meteorological Society",
      address = "Boston MA, USA",
      volume = "2",
      number = "1",
      doi = "10.1175/AIES-D-22-0058.1",
      pages=      "e220058",
      url = "https://journals.ametsoc.org/view/journals/aies/2/1/AIES-D-22-0058.1.xml"
}

@article{BROCKI2023131,
    title = {Feature perturbation augmentation for reliable evaluation of importance estimators in neural networks},
    journal = {Pattern Recognition Letters},
    volume = {176},
    pages = {131-139},
    year = {2023},
    issn = {0167-8655},
    doi = {https://doi.org/10.1016/j.patrec.2023.10.012},
    url = {https://www.sciencedirect.com/science/article/pii/S0167865523002842},
    author = {Lennart Brocki and Neo Christopher Chung}
}

@misc{augustin2024digindiffusionguidanceinvestigating,
      title={DiG-IN: Diffusion Guidance for Investigating Networks -- Uncovering Classifier Differences Neuron Visualisations and Visual Counterfactual Explanations}, 
      author={Maximilian Augustin and Yannic Neuhaus and Matthias Hein},
      year={2024},
      eprint={2311.17833},
      archivePrefix={arXiv},
      primaryClass={cs.CV},
      url={https://arxiv.org/abs/2311.17833}, 
}

@inproceedings{10.1007/978-3-030-69544-6_7,
author = {Agarwal, Chirag and Nguyen, Anh},
title = {Explaining Image Classifiers by Removing Input Features Using Generative Models},
year = {2020},
isbn = {978-3-030-69543-9},
publisher = {Springer-Verlag},
address = {Berlin, Heidelberg},
url = {https://doi.org/10.1007/978-3-030-69544-6_7},
doi = {10.1007/978-3-030-69544-6_7},
booktitle = {Computer Vision – ACCV 2020: 15th Asian Conference on Computer Vision, Kyoto, Japan, November 30 – December 4, 2020, Revised Selected Papers, Part VI},
pages = {101–118},
numpages = {18},
location = {Kyoto, Japan}
}

@misc{chang2019explainingimageclassifierscounterfactual,
      title={Explaining Image Classifiers by Counterfactual Generation}, 
      author={Chun-Hao Chang and Elliot Creager and Anna Goldenberg and David Duvenaud},
      year={2019},
      eprint={1807.08024},
      archivePrefix={arXiv},
      primaryClass={cs.CV},
      url={https://arxiv.org/abs/1807.08024}, 
}

@article{10.1016/j.imavis.2022.104516,
    author = {Shi, Rui and Li, Tianxing and Yamaguchi, Yasushi},
    title = {Output-targeted baseline for neuron attribution calculation},
    year = {2022},
    issue_date = {Aug 2022},
    publisher = {Butterworth-Heinemann},
    address = {USA},
    volume = {124},
    number = {C},
    issn = {0262-8856},
    url = {https://doi.org/10.1016/j.imavis.2022.104516},
    doi = {10.1016/j.imavis.2022.104516},
    journal = {Image Vision Comput.},
    month = aug,
    numpages = {13}
}

@inproceedings{Ren2023CanWF,
  title={Can We Faithfully Represent Absence States to Compute Shapley Values on a DNN?},
  author={J. Ren and Zhanpeng Zhou and Qirui Chen and Quanshi Zhang},
  booktitle={International Conference on Learning Representations},
  year={2023},
  url={https://api.semanticscholar.org/CorpusID:259298245}
}

@inproceedings{Izzo_2021, series={ESANN 2021},
   title={A Baseline for Shapley Values in MLPs: from Missingness to Neutrality},
   url={http://dx.doi.org/10.14428/esann/2021.ES2021-18},
   DOI={10.14428/esann/2021.es2021-18},
   booktitle={ESANN 2021 proceedings},
   publisher={Ciaco - i6doc.com},
   author={Izzo, Cosimo and Lipani, Aldo and Okhrati, Ramin and Medda, Francesca},
   year={2021},
   pages={605–610},
   collection={ESANN 2021} }

@inproceedings{10.5555/3540261.3540540,
author = {Hase, Peter and Xie, Harry and Bansal, Mohit},
title = {The out-of-distribution problem in explainability and search methods for feature importance explanations},
year = {2021},
isbn = {9781713845393},
publisher = {Curran Associates Inc.},
address = {Red Hook, NY, USA},
booktitle = {Proceedings of the 35th International Conference on Neural Information Processing Systems},
articleno = {279},
numpages = {17},
series = {NIPS '21}
}

@inproceedings{10.1007/978-3-031-09037-0_8,
    author = {Gomez, Tristan and Fr\'{e}our, Thomas and Mouch\`{e}re, Harold},
    title = {Metrics for saliency map evaluation of deep learning explanation methods},
    year = {2022},
    isbn = {978-3-031-09036-3},
    publisher = {Springer-Verlag},
    address = {Berlin, Heidelberg},
    url = {https://doi.org/10.1007/978-3-031-09037-0_8},
    doi = {10.1007/978-3-031-09037-0_8},
    booktitle = {Pattern Recognition and Artificial Intelligence: Third International Conference, ICPRAI 2022, Paris, France, June 1–3, 2022, Proceedings, Part I},
    pages = {84–95},
    numpages = {12},
    location = {Paris, France}
}

@article{salehi2022a,
    title={A Unified Survey on Anomaly, Novelty, Open-Set, and Out of-Distribution Detection: Solutions and Future Challenges},
    author={Mohammadreza Salehi and Hossein Mirzaei and Dan Hendrycks and Yixuan Li and Mohammad Hossein Rohban and Mohammad Sabokrou},
    journal={Transactions on Machine Learning Research},
    issn={2835-8856},
    year={2022},
    url={https://openreview.net/forum?id=aRtjVZvbpK},
    note={}
}

@inbook{10.5555/3454287.3455160,
      author = {Hooker, Sara and Erhan, Dumitru and Kindermans, Pieter-Jan and Kim, Been},
      title = {A benchmark for interpretability methods in deep neural networks},
      year = {2019},
      publisher = {Curran Associates Inc.},
      address = {Red Hook, NY, USA},
      booktitle = {Proceedings of the 33rd International Conference on Neural Information Processing Systems},
      articleno = {873},
      numpages = {12}
}

@article{10.1007/s11222-021-10057-z,
    author = {Hooker, Giles and Mentch, Lucas and Zhou, Siyu},
    title = {Unrestricted permutation forces extrapolation: variable importance requires at least one more model, or there is no free variable importance},
    year = {2021},
    issue_date = {Nov 2021},
    publisher = {Kluwer Academic Publishers},
    address = {USA},
    volume = {31},
    number = {6},
    issn = {0960-3174},
    url = {https://doi.org/10.1007/s11222-021-10057-z},
    doi = {10.1007/s11222-021-10057-z},
    journal = {Statistics and Computing},
    month = nov,
    numpages = {16}
}

@InProceedings{pmlr-v119-kumar20e,
  title = 	 {Problems with Shapley-value-based explanations as feature importance measures},
  author =       {Kumar, I. Elizabeth and Venkatasubramanian, Suresh and Scheidegger, Carlos and Friedler, Sorelle},
  booktitle = 	 {Proceedings of the 37th International Conference on Machine Learning},
  pages = 	 {5491--5500},
  year = 	 {2020},
  editor = 	 {III, Hal Daumé and Singh, Aarti},
  volume = 	 {119},
  series = 	 {Proceedings of Machine Learning Research},
  month = 	 {13--18 Jul},
  publisher =    {PMLR},
  url = 	 {https://proceedings.mlr.press/v119/kumar20e.html},
}

@misc{geiger2025causalabstractiontheoreticalfoundation,
      title={Causal Abstraction: A Theoretical Foundation for Mechanistic Interpretability}, 
      author={Atticus Geiger and Duligur Ibeling and Amir Zur and Maheep Chaudhary and Sonakshi Chauhan and Jing Huang and Aryaman Arora and Zhengxuan Wu and Noah Goodman and Christopher Potts and Thomas Icard},
      year={2025},
      eprint={2301.04709},
      archivePrefix={arXiv},
      primaryClass={cs.AI},
      url={https://arxiv.org/abs/2301.04709}, 
}

@inproceedings{NEURIPS2020_443dec30,
      author = {Tsang, Michael and Rambhatla, Sirisha and Liu, Yan},
      booktitle = {Advances in Neural Information Processing Systems},
      editor = {H. Larochelle and M. Ranzato and R. Hadsell and M.F. Balcan and H. Lin},
      pages = {6147--6159},
      publisher = {Curran Associates, Inc.},
      title = {How does This Interaction Affect Me? Interpretable Attribution for Feature Interactions},
      url = {https://proceedings.neurips.cc/paper_files/paper/2020/file/443dec3062d0286986e21dc0631734c9-Paper.pdf},
      volume = {33},
      year = {2020}
}

@article{10.1214/aos/1013203451,
      author = {Jerome H. Friedman},
      title = {{Greedy function approximation: A gradient boosting machine.}},
      volume = {29},
      journal = {The Annals of Statistics},
      number = {5},
      publisher = {Institute of Mathematical Statistics},
      pages = {1189 -- 1232},
      year = {2001},
      doi = {10.1214/aos/1013203451},
      URL = {https://doi.org/10.1214/aos/1013203451}
}

@article{JML-v20-18-760,
  author  = {Aaron Fisher and Cynthia Rudin and Francesca Dominici},
  title   = {All Models are Wrong, but Many are Useful: Learning a Variable's Importance by Studying an Entire Class of Prediction Models Simultaneously},
  journal = {Journal of Machine Learning Research},
  year    = {2019},
  volume  = {20},
  number  = {177},
  pages   = {1--81},
  url     = {http://jmlr.org/papers/v20/18-760.html}
}

@misc{apley2019visualizingeffectspredictorvariables,
      title={Visualizing the Effects of Predictor Variables in Black Box Supervised Learning Models}, 
      author={Daniel W. Apley and Jingyu Zhu},
      year={2019},
      eprint={1612.08468},
      archivePrefix={arXiv},
      primaryClass={stat.ME},
      url={https://arxiv.org/abs/1612.08468}, 
}

@misc{power2022grokkinggeneralizationoverfittingsmall,
      title={Grokking: Generalization Beyond Overfitting on Small Algorithmic Datasets}, 
      author={Alethea Power and Yuri Burda and Harri Edwards and Igor Babuschkin and Vedant Misra},
      year={2022},
      eprint={2201.02177},
      archivePrefix={arXiv},
      primaryClass={cs.LG},
      url={https://arxiv.org/abs/2201.02177}, 
}

@Article{10.5555/944919.944968,
      author = {Guyon, Isabelle and Elisseeff, Andr\'{e}},
      title = {An introduction to variable and feature selection},
      year = {2003},
      issue_date = {3/1/2003},
      publisher = {JMLR.org},
      volume = {3},
      number = {null},
      issn = {1532-4435},
      journal = {J. Mach. Learn. Res.},
      month = mar,
      pages = {1157–1182},
      numpages = {26}
}

@inproceedings{10.1145/3205651.3208227,
      author = {Saito, Shota and Shirakawa, Shinichi and Akimoto, Youhei},
      title = {Embedded feature selection using probabilistic model-based optimization},
      year = {2018},
      isbn = {9781450357647},
      publisher = {Association for Computing Machinery},
      address = {New York, NY, USA},
      url = {https://doi.org/10.1145/3205651.3208227},
      doi = {10.1145/3205651.3208227},
      booktitle = {Proceedings of the Genetic and Evolutionary Computation Conference Companion},
      pages = {1922–1925},
      numpages = {4},
      location = {Kyoto, Japan},
      series = {GECCO '18}
}

@inproceedings{inproceedingsOcclusionSensitiv,
      author = {Valois, Pedro and Niinuma, Koichiro and Fukui, Kazuhiro},
      year = {2024},
      month = {01},
      pages = {4817-4826},
      title = {Occlusion Sensitivity Analysis with Augmentation Subspace Perturbation in Deep Feature Space},
      doi = {10.1109/WACV57701.2024.00476}
}

@Article{Importance2024Veas,
      author = {Chan, Ho and Veas, Eduardo},
      year = {2024},
      month = {04},
      pages = {},
      title = {Importance Estimate of Features via analysis of their Weight and Gradient profile},
      doi = {10.21203/rs.3.rs-4217886/v1}
}

@misc{lymperopoulos2026weightperturbationfeatureattribution,
      title={From Weight Perturbation to Feature Attribution for Explaining Fully Connected Neural Networks}, 
      author={Thodoris Lymperopoulos and Denia Kanellopoulou},
      year={2026},
      eprint={2605.15328},
      archivePrefix={arXiv},
      primaryClass={cs.LG},
      url={https://arxiv.org/abs/2605.15328}, 
}

@article{haufe_explainable_2026,
	title = {Explainable {AI} needs formalization},
	volume = {2},
	copyright = {2026 The Author(s)},
	issn = {3005-1460},
	url = {https://www.nature.com/articles/s44387-026-00095-1},
	doi = {10.1038/s44387-026-00095-1},
	number = {1},
	urldate = {2026-04-26},
	journal = {npj Artificial Intelligence},
	author = {Haufe, Stefan and Wilming, Rick and Clark, Benedict and Zhumagambetov, Rustam and Boubekki, Ahcène and Martin, Jörg and Panknin, Danny},
	month = apr,
	year = {2026},
	pages = {42},
}

@misc{adebayo2020sanitycheckssaliencymaps,
      title={Sanity Checks for Saliency Maps}, 
      author={Julius Adebayo and Justin Gilmer and Michael Muelly and Ian Goodfellow and Moritz Hardt and Been Kim},
      year={2020},
      eprint={1810.03292},
      archivePrefix={arXiv},
      primaryClass={cs.CV},
      url={https://arxiv.org/abs/1810.03292}, 
}

@misc{fu2020axiombasedgradcamaccuratevisualization,
      title={Axiom-based Grad-CAM: Towards Accurate Visualization and Explanation of CNNs}, 
      author={Ruigang Fu and Qingyong Hu and Xiaohu Dong and Yulan Guo and Yinghui Gao and Biao Li},
      year={2020},
      eprint={2008.02312},
      archivePrefix={arXiv},
      primaryClass={cs.CV},
      url={https://arxiv.org/abs/2008.02312}, 
}

@Article{lrpOriginal,
      doi = {10.1371/journal.pone.0130140},
      author = {Bach, Sebastian AND Binder, Alexander AND Montavon, Grégoire AND Klauschen, Frederick AND Müller, Klaus-Robert AND Samek, Wojciech},
      journal = {PLOS ONE},
      publisher = {Public Library of Science},
      title = {On Pixel-Wise Explanations for Non-Linear Classifier Decisions by Layer-Wise Relevance Propagation},
      year = {2015},
      month = {07},
      volume = {10},
      url = {https://doi.org/10.1371/journal.pone.0130140},
      pages = {1-46}
}

@misc{sundararajan2017axiomaticattributiondeepnetworks,
      title={Axiomatic Attribution for Deep Networks}, 
      author={Mukund Sundararajan and Ankur Taly and Qiqi Yan},
      year={2017},
      eprint={1703.01365},
      archivePrefix={arXiv},
      primaryClass={cs.LG},
      url={https://arxiv.org/abs/1703.01365}, 
}

@misc{srinivas2019fullgradientrepresentationneuralnetwork,
      title={Full-Gradient Representation for Neural Network Visualization}, 
      author={Suraj Srinivas and Francois Fleuret},
      year={2019},
      eprint={1905.00780},
      archivePrefix={arXiv},
      primaryClass={cs.LG},
      url={https://arxiv.org/abs/1905.00780}, 
}

@misc{gdc_portal,
      author       = {{National Cancer Institute}},
      title        = {Genomic Data Commons Data Portal},
      year         = {n.d.},
      howpublished = {\url{https://portal.gdc.cancer.gov/}},
      note         = {Accessed: 2025-11-10}
}

@Article{doi:10.1200/JCO.2008.18.1370,
      author = {Parker, Joel S. and Mullins, Michael and Cheang, Maggie C.U. and Leung, Samuel and Voduc, David and Vickery, Tammi and Davies, Sherri and Fauron, Christiane and He, Xiaping and Hu, Zhiyuan and Quackenbush, John F. and Stijleman, Inge J. and Palazzo, Juan and Marron, J.S. and Nobel, Andrew B. and Mardis, Elaine and Nielsen, Torsten O. and Ellis, Matthew J. and Perou, Charles M. and Bernard, Philip S. },
      title = {Supervised Risk Predictor of Breast Cancer Based on Intrinsic Subtypes},
      journal = {Journal of Clinical Oncology},
      volume = {27},
      number = {8},
      pages = {1160-1167},
      year = {2009},
      doi = {10.1200/JCO.2008.18.1370},
      note ={PMID: 19204204},
      URL = {https://ascopubs.org/doi/abs/10.1200/JCO.2008.18.1370},
      eprint = {https://ascopubs.org/doi/pdf/10.1200/JCO.2008.18.1370}
}

@Article{METSCH2025110124,
      title = {BenchXAI: Comprehensive benchmarking of post-hoc explainable AI methods on multi-modal biomedical data},
      journal = {Computers in Biology and Medicine},
      volume = {191},
      pages = {110124},
      year = {2025},
      issn = {0010-4825},
      doi = {https://doi.org/10.1016/j.compbiomed.2025.110124},
      url = {https://www.sciencedirect.com/science/Article/pii/S0010482525004755},
      author = {Jacqueline Michelle Metsch and Anne-Christin Hauschild}
}

@misc{tmnist_vyawahare,
      author       = {Saurabh Vyawahare},
      title        = {TMNIST (Typeface MNIST)},
      year         = {2024},
      howpublished = {Kaggle},
      url          = {https://www.kaggle.com/datasets/saurabhvyawahare/tmnist-typeface-mnist}
}

@misc{affmnist_Mader,
      author       = {K Scott Mader},
      title        = {affNIST (Affine MNIST)},
      year         = {2013},
      howpublished = {Department of Computer Science, University of toronto},
      url          = {https://www.cs.toronto.edu/~tijmen/affNIST/}
}

@ARTICLE{11039632,
      author={Bhati, Deepshikha and Amiruzzaman, MD and Zhao, Ye and Guercio, Angela and Le, Tram},
      journal={IEEE Access}, 
      title={A Survey of Post-Hoc XAI Methods From a Visualization Perspective: Challenges and Opportunities}, 
      year={2025},
      volume={13},
      number={},
      pages={120785-120806},
      doi={10.1109/ACCESS.2025.3581136}
}

@misc{gu2019understandingindividualdecisionscnns,
      title={Understanding Individual Decisions of CNNs via Contrastive Backpropagation}, 
      author={Jindong Gu and Yinchong Yang and Volker Tresp},
      year={2019},
      eprint={1812.02100},
      archivePrefix={arXiv},
      primaryClass={cs.CV},
      url={https://arxiv.org/abs/1812.02100}, 
}

@misc{iwana2019explainingconvolutionalneuralnetworks,
      title={Explaining Convolutional Neural Networks using Softmax Gradient Layer-wise Relevance Propagation}, 
      author={Brian Kenji Iwana and Ryohei Kuroki and Seiichi Uchida},
      year={2019},
      eprint={1908.04351},
      archivePrefix={arXiv},
      primaryClass={cs.CV},
      url={https://arxiv.org/abs/1908.04351}, 
}

@misc{schaeffer2023emergentabilitieslargelanguage,
      title={Are Emergent Abilities of Large Language Models a Mirage?}, 
      author={Rylan Schaeffer and Brando Miranda and Sanmi Koyejo},
      year={2023},
      eprint={2304.15004},
      archivePrefix={arXiv},
      primaryClass={cs.AI},
      url={https://arxiv.org/abs/2304.15004}, 
}

@misc{Steinhardt2023EmergentAbilities,
      title={Emergent Deception and Emergent Optimization}, 
      author={Jacob Steinhardt},
      year={2023},
      archivePrefix={Bounded Regret},
      primaryClass={cs.AI},
      url={https://bounded-regret.ghost.io/emergent-deception-optimization/}, 
}

@misc{chen2025personavectorsmonitoringcontrolling,
      title={Persona Vectors: Monitoring and Controlling Character Traits in Language Models}, 
      author={Runjin Chen and Andy Arditi and Henry Sleight and Owain Evans and Jack Lindsey},
      year={2025},
      eprint={2507.21509},
      archivePrefix={arXiv},
      primaryClass={cs.CL},
      url={https://arxiv.org/abs/2507.21509}, 
}

@misc{shrikumar2019learningimportantfeaturespropagating,
      title={Learning Important Features Through Propagating Activation Differences}, 
      author={Avanti Shrikumar and Peyton Greenside and Anshul Kundaje},
      year={2019},
      eprint={1704.02685},
      archivePrefix={arXiv},
      primaryClass={cs.CV},
      url={https://arxiv.org/abs/1704.02685}, 
}

@Article{Haug2021OnBF,
      title={On Baselines for Local Feature Attributions},
      author={Johannes Haug and Stefan Zurn and Peter El-Jiz and Gjergji Kasneci},
      journal={ArXiv},
      year={2021},
      volume={abs/2101.00905},
      url={https://api.semanticscholar.org/CorpusID:230435957}
}

\clearpage
\appendix
\section{Proofs}

\subsection{Inverse Property of Attribution}
\label{app:inv_prop}

\begin{proof}
First, we express $\mathcal{R}^{t+2}$ in terms of $\mathcal{R}^{t}, d\mathcal{R}^{t, t+1}$ and $d\mathcal{R}^{t+1, t+2}$. Given an 
input $x\in \mathbb{R}^{n_0}$ and feature $i\in \{1, \dots, n_0\}$ it holds that:
\begin{equation}
    \mathcal{R}^{t+2}_i(x) = \mathcal{R}^t_i(x) +
    \underbrace{(d\mathcal{R}^{t, t+1}_i(x) + d\mathcal{R}^{t+1, t+2}_i(x))}_{d\mathcal{R}^{t, t+2}_i}. 
\end{equation}
The proof of the theorem reduces to showing that $d\mathcal{R}^{t, t+2}_i = 0$. We proceed by decomposing 
$d\mathcal{R}^{t, t+2}$ into its constituent terms:  
\begin{align}
    {d\mathcal{R}^{t, t+2}} = \mathcal{I}^T \cdot (\underbrace{\underbrace{(f^t_{W^t_i \rightarrow W^{t+1}_i}(x) - f^t(x))}_{\textcolor{blue}{A}} + \underbrace{(f^{t_1}(x) - f^{t+1}_{W^{t+1}_i \rightarrow  W^t_i}(x))}_{\textcolor{red}{B}}}_{d\mathcal{R}^t} \notag  \\[12pt] 
    + \underbrace{\underbrace{(f^{t+1}_{W^{t+1}_i\rightarrow W^{t+2}_i}(x) - f^{t+1}(x))}_{\textcolor{red}{C}} + \underbrace{(f^{t+2}(x) - f^{t+2}_{W^{t+2}_i \rightarrow W^{t+1}_i}(x))}_{\textcolor{blue}{D}}}_{d\mathcal{R}^{t+1}})
\end{align}
In this expression, $\mathcal{I}^T$ can be excluded. The terms $\textcolor{red}{B}$ and $\textcolor{red}{C}$ cancel, since $W^t_i=W^{t+2}_i$. Similarly, terms $\textcolor{blue}{A}$ 
and $\textcolor{blue}{D}$, cancel due to the identity $f^{t+2} = f^t$ and the condition $f^{t+2}_{W^{t+2}_i \rightarrow W^{t+1}_i} = f^t_{W^t_i \rightarrow W^{t+1}_i}$. 
Consequently, 
\begin{equation}
    d\mathcal{R}^{t, t+2}_i = 0,
\end{equation}
as required.
\end{proof}

\subsection{Loss Disentanglement}
\label{app:loss_dis_sup}

Cross-Entropy loss is widely regarded as the loss function of choice for classification tasks, owing to its robustness. For a 
multi-class classification problem, the loss is defined as:
\begin{equation}
    \mathcal{L}_{CE} = - \sum_{m=1}^{N} \sum_{c=1}^{[n_L]} y_{m,c} \log(\hat{y}_{m,c}).
\end{equation}
Here $m$ denotes a sample drawn from $N$ total samples, while $y_{m,c}$ and $\hat{y}_{m,c}$ represent the one-hot encoded label 
and predicted probability, respectively, for sample $m$ belonging to class to class $c$. The predicted probabilities are obtained 
by applying the softmax function to the logits of the output neurons:
\begin{equation}
    \hat{y}_{m,c} = \frac{e^{z_{m, c}}}{\sum_{c=1}^{[n_L]}e^{z_{m, c}}}.
\end{equation}

For this selection of loss function (simplified as $\mathcal{L}$), given a sample $m$ it holds that:
\begin{align}
    \label{eq:l_disent_1}
    \mathcal{L}_m &= \log(\frac{e^{z_{m, c}}}{\sum_{c=1}^{[n_L]}e^{z_{m, c}}}) \\
    \label{eq:l_disent_2}
    &= z_{m, t} - \log(\sum_{c=1}^{[n_L]} e^{z_{m, c}})
\end{align}

To this point, we observe that the computation in \cref{eq:l_disent_2} is skewed by the presence of the term $e^{z_{m, t}}$,
which corresponds to the target neuron. This bias is inherent, as no algebraic identity exists to separate the logarithm of a 
sum of terms. Nevertheless, the induced bias is negligible, enabling the derivation of an effective approximation, leading to:
\begin{equation}
    \label{eq:loss_break}
    \mathcal{L} \simeq \mathcal{L}_{target} + \mathcal{L}_{non-target}.
\end{equation}

Deeper into this analysis, the update step requires the calculation of the loss gradient with respect to weight parameters. By 
exploiting \cref{eq:loss_break} and the linearity in gradients, this $\frac{\partial \mathcal{L}}{\partial w_{ij}}$ can be 
expressed as the sum of disentangled gradients:
\begin{equation}
    \frac{\partial \mathcal{L}}{\partial w_{ij}} = \frac{\partial \mathcal{L}_{target}}{\partial w_{ij}} + \frac{\mathcal{L}_{non-target}}{\partial w_{ij}} \\
\end{equation}
\par \noindent
Therefore, by considering the standard procedure of parameter update of a FCNN, the total weight update is:
\begin{align}
    \Delta w_{ij} & = -\eta * \frac{\partial \mathcal{L}}{\partial w_{ij}} \\
    & = -\eta * [ \frac{\partial(\mathcal{L}_{target})}{\partial w_{ij}} + \frac{\partial(\mathcal{L}_{non-target})}{\partial w_{ij}} ] \\
    & = \Delta w_{target} + \Delta w_{non-target} \label{eq:delta_disent}
\end{align}

Thus, the change $\Delta w_{ij}$ can indeed be approximately disentangled into a component from the target neuron's loss 
($\Delta w_{target}$) and a component from the aggregation of non-target neurons' loss ($\Delta w_{non-target}$). The transition 
from a single sample to an entire minibatch only requires the summation of \cref{eq:l_disent_2,eq:delta_disent} over all $m$ samples.

\end{document}